\documentclass{article} 

\usepackage{microtype}
\usepackage{graphicx}
\usepackage{booktabs} 

\usepackage{empheq}
\usepackage{bm} 
\usepackage{bbm} 
\usepackage{commath} 

\DeclareMathOperator{\sgn}{sgn}
\usepackage{relsize} 
\DeclareMathOperator\arctanh{arctanh}
\usepackage{caption} 
\usepackage{wrapfig} 
\usepackage{subfig}
\usepackage{graphicx}

\usepackage{algorithm} 
\usepackage{algorithmic} 

\usepackage[preprint]{neurips_2022}
\title{Circular Belief Propagation for Approximate Probabilistic Inference}
\author{%
  Vincent Bouttier\textsuperscript{1,2}\thanks{Now at Iktos SA.} , Renaud Jardri\textsuperscript{1,2}, Sophie Denève\textsuperscript{1}\thanks{Now at Institut de Neurosciences de la Timone.}
  \\
  \textsuperscript{1}Group for Neural Theory, Ecole Normale Supérieure, Paris, France \\
  \textsuperscript{2}Université de Lille, Lille, France\\
  \texttt{renaud.jardri@univ-lille.fr}
}

\usepackage{amsmath}
\usepackage{amssymb}
\usepackage{mathtools}
\usepackage{amsthm}

\usepackage[capitalize,noabbrev]{cleveref}

\usepackage{minitoc} 
\usepackage[toc,page,header]{appendix} 

\theoremstyle{plain}
\newtheorem{theorem}{Theorem}[section]

\newtheorem{corollary}[theorem]{Corollary}
\theoremstyle{definition}

\theoremstyle{remark}

\usepackage[textsize=tiny]{todonotes}

\begin{document}

\doparttoc 
\faketableofcontents 

\maketitle

\begin{abstract}
Belief Propagation (BP) is a simple probabilistic inference algorithm, consisting of passing messages between nodes of a graph representing a probability distribution. 
Its analogy with a neural network suggests that it could have far-ranging applications for neuroscience and artificial intelligence. Unfortunately, it is only exact when applied to cycle-free graphs, which restricts the potential of the algorithm.
In this paper, we propose Circular Belief Propagation (CBP), an extension of BP which limits the detrimental effects of message reverberation caused by cycles by learning to detect and cancel spurious correlations and belief amplifications. We show in numerical experiments involving binary 
probabilistic graphs that CBP far outperforms BP and reaches good performance compared to that of previously proposed algorithms. 
\end{abstract}

\section{Introduction}\label{sec:intro}

Probabilistic inference arises in many applications, from medical diagnosis to Wi-Fi protocols.
In all methods performing inference, the Belief Propagation algorithm \citep{Pearl1988} is among the most successful ones. This procedure consists of propagating local information 
by passing messages between nodes of a graph representing the probability distribution. Belief Propagation is known to achieve exact inferences when the graph is acyclic. 
The key reason for this lies in the fact that a message $m_{i \to j}$ is computed based on all messages $m_{k \to i}$ received by the sender node except the opposite message $m_{j \to i}$, therefore preventing this message from being reverberated and thus counted twice. 
One drawback of BP is that it often performs poorly in cyclic graphs, for which correcting for these loops of length $2$ is no longer sufficient \citep{Murphy1999, Weiss2000}. 
Messages pass in cycles 
from node to node and return to the original node, causing the same piece of information to be counted several times. This is known as the ``double-counting" problem.
Several variants or generalization of BP have been proposed to tackle this problem ever since the interest in BP started to grow \citep{Minka2001a, Sudderth2003, Ihler2009}. 


We present in this work the Circular Belief Propagation algorithm. This algorithm proposes to counter the effect of cycles (which introduce spurious correlations between messages) by actively decorrelating the messages. 
We show that Circular BP significantly outperforms BP in cyclic probabilistic graphs, including in cases where BP does not converges, and performs approximate inference with a very impressive quality even for fully dense probabilistic graphs.

This paper is organized as follows. In Sections \ref{sec:background}-\ref{sec:conv-results}, 
we provide background information, define the Circular BP algorithm, relate it to existing algorithms, and provide sufficient conditions for its convergence. In sections \ref{sec:approach-learning-CBP} and \ref{sec:learning-CBP}, we show the performance of CBP on both synthetic and real-world problems involving binary variables, using supervised or unsupervised procedures. The Appendices provide, among others, formulations of the algorithm in the general case (non-pairwise graphs and continuous variables). 

\section{Background}\label{sec:background}

\subsection{Inference in probabilistic graphical models}\label{subsec:inference-in-graphical-models}
The object of study is a probability distribution $p(\mathbf{x})$, where $\mathbf{x} = (x_1, x_2, \dots, x_n)$. The distribution can be decomposed into a product of conditionally independent factors that each describes interactions between distinct subsets of variables \citep{Koller2009, WainwrightJordan2008}:
\begin{equation}\label{eq:factorization-p_x-pairwise}
    p(\mathbf{x}) = \frac{1}{Z} \prod\limits_{(i,j)} \psi_{ij}(x_i,x_j) \prod\limits_i \psi_i(x_i)
\end{equation}
where we consider only unitary and pairwise interactions 
(Appendices 
show an extension of this special case to any factorization of $p(\mathbf{x})$, that is, any Markov random field 
involving higher-order potentials, e.g., $\psi_{ijk}(x_i,x_j,x_k)$).
$\psi_{i}$ represents prior knowledge about variable $x_i$, 
while $\psi_{ij}$ describes the interaction between $x_i$ and $x_j$. $Z$ is a normalization constant. 
As Figure \ref{fig:BP-unexact-cylic-graph}A shows, the probability distribution can be represented graphically as a graphical model called \emph{factor graph}, composed of variable nodes $x_i$ and factor nodes $\psi_{ij}$ and $\psi_{i}$. 

We focus in this work on the particular inference problem of finding the marginals of $p(\mathbf{x})$, $p_i(x_i) \equiv \sum_{\mathbf{x} \setminus x_i} p(\mathbf{x})$,
given partial and noisy information (incorporated into $\{\psi_i(x_i)\}$) about $\mathbf{x}$. 
Unfortunately, direct calculation of the marginals $p_i(x_i)$ takes exponential time in the number of variables $n$. 
That is why algorithms, like the \emph{Belief Propagation algorithm} and sampling 
methods, have been developed to solve this inference problem more efficiently. 
However, their drawback is that they produce only approximate marginal probabilities $b_i(x_i) \approx p_i(x_i)$.

\subsection{Belief Propagation}\label{subsec:BP}

Belief Propagation \citep{Pearl1988}, or sum-product algorithm, is a message-passing algorithm 
which performs approximate inference on a probabilistic graph.
It approximates the marginals of the distribution by making variable nodes $x_i$ share all the probabilistic information available with the rest of the network by sending messages to other variable nodes, eventually propagating information to the whole network. 
The algorithm consists of running iteratively the following update message equation on the graph (see Figure \ref{fig:BP-unexact-cylic-graph}B): 
\begin{equation}\label{eq:BP-message}
    m_{i \to j}^{\text{new}}(x_j) \propto \sum_{x_i} \psi_{ij}(x_i,x_j) \psi_i(x_i) \prod\limits_{k \in \mathcal{N}(i) \setminus j} m^{\text{old}}_{k \to i}(x_i)
\end{equation}
where $\mathcal{N}(i)$ is the set of neighbors of node $x_i$ in the graph. Messages are for instance initialized uniformly over the nodes ($m_{i \to j}(x_j) = 1/\mathcal{N}(j)$).
Once messages have converged, approximate marginal probabilities (or \emph{beliefs}) are computed as: 
\begin{equation}\label{eq:BP-belief}
    b_i(x_i) \propto \psi_i(x_i) \prod\limits_{k \in \mathcal{N}(i)} m_{k \to i}(x_i)
\end{equation}

\begin{figure}[h]
  \centering
  \includegraphics[width=\linewidth]{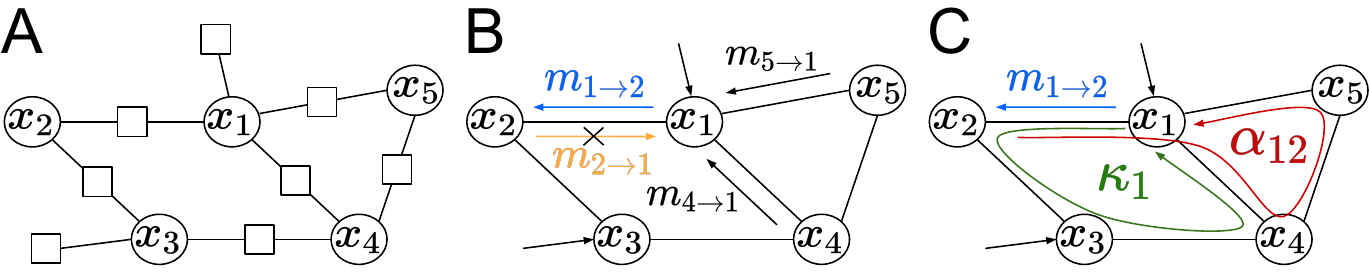}
  \vspace*{0.05mm}
  \caption{\textbf{Belief Propagation and Circular Belief Propagation algorithms applied to a probabilistic graph}. 
  \textbf{(A)} The probability distribution $p(\mathbf{x})$ is represented by a factor graph with pairwise potentials $\psi_{ij}$ and unitary potentials $\psi_i$.
  \textbf{(B)} BP aims at estimating marginals $p_i(x_i)$ by exchanging messages in the graph. The message $m_{1 \to 2}$ 
  depends on three components: the messages received by node $x_1$ from its neighbors except $x_2$, including the external message (see full black lines), 
  and the interaction
  $\psi_{12}$. 
  Estimated marginal $b_i(x_i)$ of a node $x_i$ 
  is formed based on all messages received by the node. 
  BP is not exact when applied to cyclic graphs, for two reasons. 
  First, messages get counted multiple times: $m_{1 \to 2}$ naturally travels back to 
  $x_1$ because of the cycle $x_1-x_2-x_3-x_4$. Second, opposite messages are correlated: $m_{1 \to 2}$ depends on $m_{5 \to 1}$ which depends on $m_{4 \to 5}$ which depends on $m_{1 \to 4}$ which depends on $m_{2 \to 1}$. 
  \textbf{(C)} Contrary to BP, Circular BP (partially) takes $m_{2 \to 1}$ into account to compute $m_{1 \to 2}$. 
  Parameter $\bm{\kappa}$ fights 
  the belief amplifications caused by messages being reverberated.
  Parameters $\bm{\alpha}$ decorrelates opposite messages. 
  }
  \label{fig:BP-unexact-cylic-graph} 
\end{figure}

A crucial feature of the BP algorithm is the message exclusion principle ($k \in \mathcal{N}(i) \setminus j$ in Equation \eqref{eq:BP-message}): 
to compute $m_{i \to j}$, all messages $m_{k \to i}$ coming to node $x_i$ are taken into account and combined, except the message in the opposite direction $m_{j \to i}$. In other words, node $x_i$ sends to $x_j$ everything it knows ($b_i$) except what $x_j$ communicated to $x_i$. 
This feature intuitively explains why BP produces exact marginals in the case of acyclic graphs: 
information propagates into the network without ``coming back". 
However, in the presence of cycles, the same evidence travels multiple times through loops of the graph, and is mistaken for new evidence \citep{Pearl1988}; see legend of Figure \ref{fig:BP-unexact-cylic-graph}. 
This intuitively explains why BP (also called \emph{Loopy} BP in this case) is incorrect in the general case. More precisely, the algorithm does not necessarily converge on cyclic graphs, and when it converges it produces incorrect marginals. 
A number of empirical studies \citep{Murphy1999, Weiss2000, Mooij2004, Litvak2009a} have led 
to a better understanding of the convergence and performance of BP depending on the probabilistic graph. 

BP has initially been used because of its empirical performance, but lacked a deep theoretical comprehension as to why the algorithm worked so well in practice. However, in the early 2000s, a theoretical foundation of BP emerged: stable fixed points of BP are minima of the Bethe free energy from statistical physics \citep{Yedidia2001, Heskes2002}. 
The \emph{Bethe free energy} approximates the \emph{Gibbs free energy} $G(b) := D_{KL}(b \Vert p) - \log(Z) = \sum_\mathbf{x} b(\mathbf{x}) E(\mathbf{x}) + \sum_\mathbf{x} b(\mathbf{x}) \log(b(\mathbf{x}))$ where the energy $E(\mathbf{x}) := - \log(p(\mathbf{x}) Z)$, by considering that the variational distribution $b(\mathbf{x})$ can be represented by an acyclic probabilistic graph 
(\emph{Bethe approximation}), as further explained in the Appendices. Minimizing the Gibbs free energy also minimizes the distance (KL divergence) between $b$ and $p$, and the idea (or hope) 
is that marginals of $b$ and $p$ are therefore similar: $b_i(x_i) \approx p_i(x_i)$.
This tight link between BP and the Bethe free energy 
opened the way to investigate 
the goodness of BP through the estimation of the BP error \citep{Wainwright2003b, Taga2006, Ihler2007, Mooij2009, Shi2010}, and the properties of the Bethe free energy and its fixed points \citep{Heskes2003, Heskes2004, Watanabe2009, Watanabe2011, Weller2014, Knoll2017}. 


\section{Circular Belief Propagation}\label{sec:Circular-BP}

\subsection{Definition}\label{subsubsec:definition-CBP} 

The Circular Belief Propagation (CBP) algorithm is an extension of the Belief Propagation (BP) algorithm and performs message-passing on probabilistic graphs to carry out approximate inference. 
It consists of initializing messages $m_{i \to j}$ to some distribution, and then running the following updates:
\begin{equation}\label{eq:eCBP-message}
    m_{i \to j}^{\text{new}}(x_j) \propto \sum_{x_i} \psi_{ij}(x_i,x_j)^{\beta_{ij}} \Big(\psi_i(x_i)^{\gamma_i} \times \prod\limits_{k \in \mathcal{N}(i) \setminus j} m^{\text{old}}_{k \to i}(x_i) \times m^{\text{old}}_{j \to i}(x_i)^{1 - \alpha_{ij} / \kappa_i}\Big)^{\kappa_i}
\end{equation} 
After convergence (or a given number of iterations), approximate marginals are computed using:
\begin{equation}\label{eq:eCBP-belief}
    b_i(x_i) \propto \bigg(\psi_i(x_i)^{\gamma_i} \prod\limits_{k \in \mathcal{N}(i)} m_{k \to i}(x_i)\bigg)^{\kappa_i}
\end{equation}
BP corresponds to the particular case $(\bm{\alpha}, \bm{\kappa}, \bm{\beta}, \bm{\gamma}) = (\bm{1}, \bm{1}, \bm{1}, \bm{1})$. The main difference between CBP and BP is the occurrence of the message in the opposite direction $m_{j \to i}$ to compute $m_{i \to j}$; see Figure \ref{fig:BP-unexact-cylic-graph}B. 
Note that $\alpha_{ij}$ and $\beta_{ij}$ are assigned to the \emph{undirected} edge $(i,j)$, while $\kappa_i$ and $\gamma_i$ are assigned to the variable node $x_i$.

\subsection{Circular BP for binary distributions}\label{subsec:BP-CBP-Ising-model} 

For probability distributions $p(\bm{x})$ over binary variables ($x_i \in \{-1,+1\}$)
, CBP takes a very simple form in the log-domain.
For simplicity, we consider here a particular class of such distributions 
called \emph{Ising models} (also known as \emph{Boltzmann machines}) from statistical physics \citep{Ising1925, Baxter1982} in which 
pairwise factors take a specific form: $\psi_{ij}(x_i, x_j) \propto \exp(J_{ij} x_i x_j)$. 
CBP writes as follows:
\begin{empheq}[left=\empheqlbrace]{align}
  &M_{i \to j}^{\text{new}} = f(B_i - \alpha_{ij} M_{j \to i},\, \beta_{ij} J_{ij})\label{eq:eCBP-message-log}\\
  &B_i = \kappa_i \Big(\sum\limits_{j \in \mathcal{N}(i)} M_{j \to i} + \gamma_i M_{\text{ext} \to i}\Big)\label{eq:eCBP-belief-log}
\end{empheq}
(see Appendix \ref{subsec:eCBP} for the demonstration). $M_{i \to j} \equiv \frac{1}{2} \log\big(\frac{m_{i \to j}(x_j=+1)}{m_{i \to j}(x_j=-1)}\big)$ represents the information about variable $x_j$ brought by variable $x_i$, and $B_i \equiv \frac{1}{2} \log\big(\frac{b_i(x_i=+1)}{b_i(x_i=-1)}\big)$ is by definition half of the log-likelihood ratio. 
The approximate marginal probabilities are given by $b_i(x_i=\pm 1) = \sigma(\pm 2 B_i)$, i.e., $b_i(x_i) \propto \exp(B_i x_i)$.
Lastly, $M_{\text{ext} \to i} \equiv \frac{1}{2} \log\big(\frac{\psi_i(x_i = +1)}{\psi_i(x_i = -1)}\big)$ represents alternatively prior information over a variable $x_i$, or here a noisy sensory observation providing information about the variable (\textit{ext} stands for ``external"). 

Function $f$ takes a simple form (see also \citet{Mooij2007}): 
\begin{equation}\label{eq:function-fij-Ising-model}
    f(x,\, J) = \arctanh\big(\tanh(J) \tanh(x)\big)
\end{equation}
$f$ is a sigmoidal function of $x$ controlled by parameter $J$ representing the bounded ``trust" $J$ between nodes; see Figure \ref{fig:BP-and-CBP}B. 
Note that $f$ resembles the usual non-linearity $f(x, w) = w \tanh(x)$.

\subsection{Interpretation of Circular BP}
Parameters $\bm{\alpha}$, $\bm{\kappa}$, $\bm{\beta}$ and $\bm{\gamma}$ can be interpreted easily.
$\bm{\beta}$ and $\bm{\gamma}$ act as rescaling factors (respectively for the pairwise and unitary factors). 
$\bm{\kappa}$ rescales the beliefs in order to counter the effect of message amplifications in loops (double-counting); see Figure \ref{fig:BP-unexact-cylic-graph}C.
But the specificity of BP is the parameter $\bm{\alpha}$ which controls how the opposite message $m_{j \to i}$ is taken into account for the computation of $m_{i \to j}$: positively ($\alpha_{ij} < \kappa_i$), negatively ($\alpha_{ij} > \kappa_i$) or not at all ($\alpha_{ij} = \kappa_i$) as in BP. $\bm{\alpha}$ therefore acts as a \emph{loop correction factor}.
Equation \eqref{eq:eCBP-message-log} means that node $i$, encoding for variable $x_i$, sends to node $j$ everything it knows ($B_i$) except a \emph{rescaled} version of what $j$ communicated to $i$ ($M_{j \to i}$).

\section{Related work}\label{sec:related-work}


\subsection{Circular BP from \citet{Jardri2013}}

A particular case of Circular BP was defined in \citet{Jardri2013}, which corresponds to $(\bm{\kappa}, \bm{\beta}, \bm{\gamma}) = (\bm{1}, \bm{1}, \bm{1})$.
The important conceptual difference is that the authors propose 
a way of impairing BP, not of improving it; see more in the discussion.

\subsection{Link to reweighted Bethe free energy}

Circular BP tightly relates to the family of reweighted BP algorithms for approximate inference, including Fractional BP \citep{Wiegerinck2002}, Reweighted BP \citep{Loh2014}, Tree-reweighted BP \citep{Wainwright2002, Wainwright2003,  Wainwright2005}, Convex BP \citep{Hazan2008, Hazan2010}, Power Expectation Propagation \citep{Minka2002, Minka2004}, and $\alpha$-BP \citep{Liu2019, Liu2020} 
(see the Appendices 
for more detail). 
All these reweighted BP algorithms are variational inference techniques which aim at minimizing a particular reweighted Bethe free energy approximating the Gibbs free energy.
First, the variational entropy component $-\sum_\mathbf{x} b(\mathbf{x}) \log(b(\mathbf{x}))$ of the Gibbs free energy is approximated as if the variational distribution $b(x)$ was: 
\begin{equation}\label{eq:eFBP-approx-b-wrt-marginals}
    b(\mathbf{x}) \approx \prod\limits_{i,j}\Big(\frac{b_{ij}(x_i,x_j)}{b_i(x_i) b_j(x_j)}\Big)^{1/\alpha_{ij}} \prod\limits_i \big(b_i(x_i)\big)^{1/\kappa_i}
\end{equation}
(which corresponds to BP's Bethe approximation for $\bm{\alpha} = \bm{1}$ and $\bm{\kappa} = \bm{1}$). Second, the variational average energy component $\sum_\mathbf{x} b(\mathbf{x}) E(\mathbf{x})$ of the Gibbs free energy is approximated as if $p(x)$ was: 
\begin{equation}\label{eq:eFBP-approx-b-wrt-potentials}
    p(\mathbf{x}) \approx \prod\limits_{i,j}\big(\psi_{ij}(x_i,x_j)\big)^{\beta_{ij}} \prod\limits_i \big(\psi_{i}(x_i)\big)^{\gamma_i}
\end{equation}
Altogether, these hypotheses result in a parametric approximation of the Gibbs free energy, which generalizes the Bethe free energy from BP. 
The algorithm aiming at minimizing this resulting approximate Gibbs free energy writes for Ising models:
\begin{empheq}[left=\empheqlbrace]{align}
  &M_{i \to j}^{\text{new}} = \frac{1}{\alpha_{ij}} f\Big(B_i - \alpha_{ij} M_{j \to i},\, \alpha_{ij} \beta_{ij} J_{ij}\Big) \label{eq:FBP-message-log}\\ 
  &B_i = \kappa_i \Big(\sum_{j \in \mathcal{N}(i)} M_{j \to i} + \gamma_i M_{\text{ext} \to i}\Big) \label{eq:FBP-belief-log}
\end{empheq}
where function $f$ is the same as above. Message update equations are similar between Reweighted BP algorithms and Circular BP (see Equations \eqref{eq:eCBP-message-log} and \eqref{eq:eCBP-belief-log}) and therefore so are the results of the algorithms (see Figure \ref{fig:FBP-vs-CBP}). The algorithms differ in the specific form of the message update for parameter $\bm{\alpha}$. In Reweighted BP algorithms, $\bm{\alpha}$ corresponds to a weight in the reweighted Bethe free energy, the quantity that these algorithms try to minimize. Instead, in Circular BP, $\bm{\alpha}$ relates to 
the idea of cancelling the dependencies between opposite messages arising because of loops (see Figure \ref{fig:BP-unexact-cylic-graph}C). 


To find parameters for which these Reweighted BP algorithms outperform BP, several techniques are used.
Tree-Reweighted BP \citep{Wainwright2005} uses a gradient ascent method to optimize $\bm{\alpha}$ 
, but involves solving a max-weight spanning tree problem at each iteration. 
Fractional BP \citep{Wiegerinck2002} 
fits $\bm{\alpha}$ by using an unsupervised learning rule based on linear response theory. 



\section{Convergence properties of Circular BP}\label{sec:conv-results}

Convergence of Circular BP is crucial to carry out approximate inference. In general, if message-passing does not converge, then beliefs have very little to do with the correct marginals \citep{Murphy1999}. 
It is therefore important to control the convergence properties, and if possible, to ensure the convergence of our proposed algorithm.
We state here sufficient conditions for the convergence of Circular BP in an Ising model, and provide ways to find such parameters. 
Notably, Theorem \ref{theorem:convergence-is-possible} shows that choosing $\bm{\alpha}$ and $\bm{\kappa}$ uniformly, equal and small enough guarantees the convergence of CBP, for any set of parameters $(\bm{\gamma}, \bm{\beta})$. More precisely, in this case, the algorithm has only one fixed point and converges to it with at least a linear rate.



We start by defining matrix $\bm{A}$ whose coefficients are:
\begin{equation}\label{eq:def-matrix-A}
    A_{i \to j, k \to l}  = \abs{\kappa_i} \tanh\abs{\beta_{ij} J_{ij}} \delta_{il} \mathbbm{1}_{\mathcal{N}(i)}(k) \abs{1 - \frac{\alpha_{ij}}{\kappa_i}}_{j = k}
\end{equation}
where $\mathbbm{1}_{\mathcal{N}(i)}(k) = 1$ if $k \in \mathcal{N}(i)$, otherwise $= 0$, $\delta_{il} = 1$ if $i = l$, otherwise $= 0$, and $\abs{1 - \frac{\alpha_{ij}}{\kappa_i}}_{j = k} = 1 - \frac{\alpha_{ij}}{\kappa_i}$ if $j = k$, otherwise $1$. 

\begin{theorem}\label{theorem:conv-norm}
    If for any induced operator norm $\lVert \cdot \rVert$ (sometimes called natural matrix norm), $\lVert A \rVert < 1$, then CBP has a unique fixed point and CBP converges to it with at least a linear rate.
\end{theorem}


\begin{theorem}\label{theorem:conv-spectral-radius}
    If $ \forall (i,j)$, $\alpha_{ij}/\kappa_i \leq 1$ and the spectral radius of A, $\rho(A) < 1$, then Circular BP converges to the unique fixed point.
\end{theorem}
The proofs of Theorems \ref{theorem:conv-norm} and \ref{theorem:conv-spectral-radius} are provided in the Appendices and follow closely the ones provided in \citet{Mooij2007} for BP, that is, the special case $(\bm{\alpha}, \bm{\kappa}, \bm{\beta}, \bm{\gamma}) = (\bm{1}, \bm{1}, \bm{1}, \bm{1})$. 

A consequence of Theorem \ref{theorem:conv-spectral-radius} is the following fundamental result, which distinguishes Circular BP from related approaches like Power EP, Fractional BP and $\alpha$-BP: 
\begin{theorem}\label{theorem:convergence-is-possible}
    For a given weighted graph (with weights $J_{ij}$), it is possible to find parameters $\bm{\alpha}$ and $\bm{\kappa}$ such that Circular BP converges for any external input $\bm{M_{\text{ext}}}$ and any choice of parameters $(\bm{\gamma}, \bm{\beta})$.
\end{theorem}
\begin{proof}
    Let us take $\alpha_{ij} = \kappa_i \equiv v \in \mathbb{R}_{+}$. In this case, $A_{i \to j, k \to l}  = v \tanh\abs{\beta_{ij} J_{ij}} \delta_{il} \mathbbm{1}_{\mathcal{N}(i) \setminus j}(k)$. When $v \to 0$, all coefficients of A go to zero. 
    The spectral norm (induced by the $l_2$-norm) is a continuous application, and $\rho(\bm{0}) = 0$ 
    , thus $\rho(A) \to 0$ when $v \to 0$.
    We conclude by using Theorem \ref{theorem:conv-spectral-radius} as $\alpha_{ij} / \kappa_i = 1 \leq 1$: there exists $v^\star$ such that for all $p < p^\star$, CBP converges. 
\end{proof}

\section{Approach for optimizing the parameters}\label{sec:approach-learning-CBP}

The goal is, given the interactions $\bm{J} = \{J_{ij}\}$ between variables of the probability distribution, to perform approximate inference for any external input $\bm{M_{\text{ext}}} = \{M_{\text{ext} \to i}\}$.\footnote{Note that the training does not aim at generalizing to different graph structures or different graph weights, as the fitted parameters $(\bm{\alpha}, \bm{\kappa}, \bm{\beta}, \bm{\gamma})$ depend on the interactions $\bm{J}$.}
We consider two ways for finding the optimal parameters of a given weighted graph. The supervised learning method minimizes the MSE loss between the approximate and true marginals, but can only be applied to small or sparsely connected graph where exact inference can be performed on the training examples. Additionally, we propose an unsupervised learning method where parameters are estimated using local learning rules, and that can generalize to arbitrarily large and complex graphs.

\subsection{Supervised learning procedure}


Parameters of the model $(\bm{\alpha}, \bm{\kappa}, \bm{\beta}, \bm{\gamma})$ 
are learnt so that the approximate marginals or beliefs $\{b_i(x_i)\}$ are as close as possible to the true marginals $\{p_i(x_i)\}$. 
We fit these $2 n_{\text{nodes}} + 2 n_{\text{edges}}$ parameters 
by minimizing 
the MSE loss or squared L2 norm $L(b,p) = \frac{1}{n_{\text{nodes}}} \sum\limits_{i=1}^{n_{\text{nodes}}} [b_i(x_i = +1) - p_i(x_i = +1)]^2$
between beliefs $\{b_i(x_i)\}$ obtained with CBP and the true marginals $\{p_i(x_i)\}$. 
True marginals ${p_i(x_i)}$ are computed using the Junction Tree algorithm, an exact inference algorithm implemented on the pgmpy library \citep{pgmpy2015}. Approximate marginals $b_i(x_i)$ are obtained by running CBP for $T = 100$ time steps.
The model is trained on PyTorch \citep{PyTorch2019} using backpropagation through time with a gradient-descent based method, resilient backpropagation \citep{Rprop1993}. Pytorch does automatic differentiation to update parameters.
We use a learning rate of $0.001$ and stop the learning once the validation loss saturates. 
Parameters 
are initialized such that CBP converges with this choice of parameters (see section \ref{sec:conv-results}). 
One initialization tested, among others, is to initialize $\bm{\beta}$ and $\bm{\gamma}$ at the value $\bm{1}$, while $\bm{\alpha}$ and $\bm{\kappa}$ are initialized at the same small enough value $v$: 
we decrease $v$ from value 1 until $\rho(\bm{A}) < 1$
, which ensures that CBP converges (see Theorem \ref{theorem:convergence-is-possible}). 
The model is also trained using function \emph{least squares} of SciPy 
\citep{SciPy2020}. We select the parameters of the method (PyTorch or SciPy) performing better on the validation set. 

\subsection{Unsupervised learning procedure}\label{sec:unsupervised-learning-eCBP} 

The inconvenient of the supervised learning procedure is the need to generate training examples, i.e., to compute the true marginals. This has an exponential complexity in the size of the graph and therefore does not scale to big dense graphs. 
We propose a way of learning parameters $\bm{\alpha}$ and $\bm{\kappa}$ with: 
\begin{empheq}[left=\empheqlbrace]{align}
  &\Delta \alpha_{ij} = \eta_1 \big[M_{j \to i} (B_i - \alpha_{ij} M_{j \to i})+ M_{i \to j} (B_j - \alpha_{ij} M_{i \to j})\big] \label{eq:unsupervised-learning-rule-beta}\\
  &\Delta \kappa_i = - \eta_2 M_{\text{ext} \to i} (B_i - M_{\text{ext} \to i})\label{eq:unsupervised-learning-rule-gamma} 
\end{empheq}
where we keep $(\bm{\beta}, \bm{\gamma}) = (\bm{1}, \bm{1})$ fixed. Messages and beliefs are taken after $T=100$ iterations, and $\eta_1$ and $\eta_2$ are learning rates. The second part of Equation \eqref{eq:unsupervised-learning-rule-beta} ensures that $\bm{\alpha}$ is symmetric.

The proposed learning rules try to minimize redundancies in the information sent to other nodes. 
We place ourselves in a scenario where there should not be correlations between messages sent in opposite directions: training examples 
are pure noise, that is, are uninformative. 
The learning rule on $\bm{\alpha}$ makes the information that node $i$ sends to $j$  $(B_i-\alpha_{ij} M_{j \to i})$ orthogonal to the information that node $j$ received from $j$ ($M_{j \to i}$). 
The learning rule on $\bm{\kappa}$ aims at countering the amplification of inputs messages, so that overall the information received by node $i$ (including the effects of cycles) is $M_{\text{ext} \to i}$.
In the resulting message-passing algorithm, redundancies caused by cycles are minimized.
See Appendix \ref{sec:unsupervised-learning-appendix} for more details.

\section{Numerical experiments}\label{sec:learning-CBP} 

\subsection{Synthetic problems}\label{subsec:experimental-setting} 

\paragraph{Experimental setting}

Simulations involve Erdős-Renyi graphs with 9 nodes (see Figure \ref{fig:CBP-learning-results}) 
generated with connection probabilities $p$ ranging from $0.2$ to $1$. 
For each graph topology or connection probability, 30 Ising models (weighted graphs) are generated. 
An absence of edge means $J_{ij} = 0$, while existing edges had their weights sampled randomly according to $J_{ij} \sim \mathcal{N}(0, 1)$ (spin-glass). As a reminder, $J_{ij}$ is associated to the \emph{unoriented} edge $(i,j)$: $J_{i \to j} = J_{j \to i} \equiv J_{ij}$. 
For each weighted graph, external evidences $\bm{M_{\text{ext}}}$ are generated according to $M_{\text{ext} \to i} \sim \mathcal{N}(0, 1)$. 200 training examples, 100 validation examples and 100 test examples are generated, where an example is a vector $\bm{M_{\text{ext}}}$. 


The Appendices generalize this restrictive setup  
to structured rather than random graphs (see Figure \ref{fig:marginals-CBP-models-structures}) as grids and bipartite graphs, graphs with a different distribution of weights (uniform positive couplings; see Figure \ref{fig:marginals-BCP-models-positive-couplings}), and graphs with a higher number of nodes (100 nodes; see Figure \ref{fig:marginals-BCP-models-bigger-graphs}).


\paragraph{Results of supervised learning}


Marginals obtained after fitting the CBP algorithm can be seen in Figure \ref{fig:CBP-learning-results}. Circular BP performs well qualitatively: 
it generalizes very well to test data for all connection probabilities. On the contrary, for BP, there is a sharp transition from carrying out rather good approximate inference for $p=0.2$ to performing really poorly for $p \geq 0.3$ (at least on some of the 30 weighted graphs) because the system does not converge (frustratation phenomenon).

\begin{figure}
  \centering
  \includegraphics[width=\linewidth]{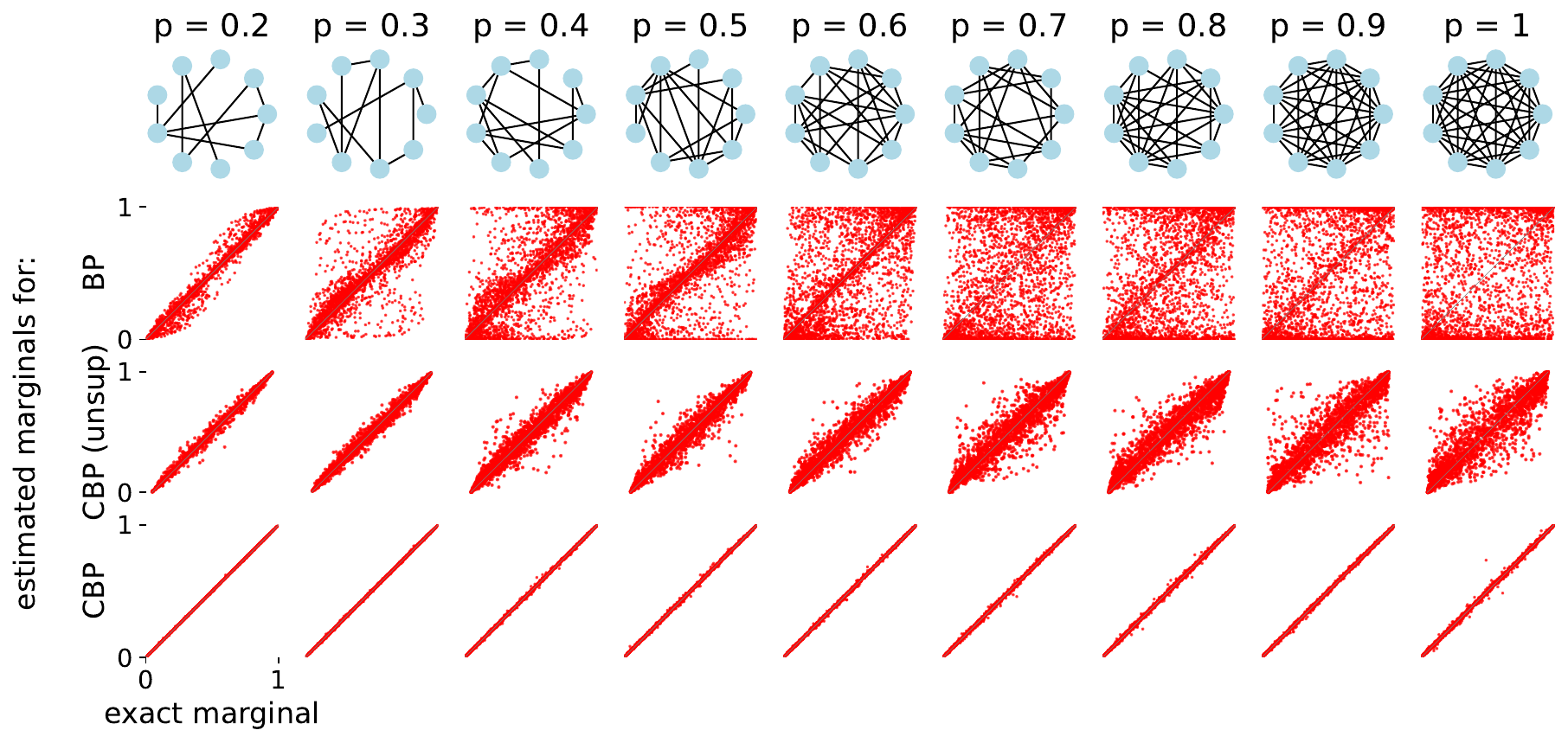}
  \vspace*{-5mm}
  \caption{\textbf{Results of Circular BP on Erdos-Renyi graphs}. 
  Estimated marginals on the test set for BP and Circular BP, for both unsupervised and supervised learning procedures. One point represents the belief of a node, on one of the test examples, and one of the 30 randomly generated graphs.
  }
  \label{fig:CBP-learning-results}
\end{figure}

\begin{figure}[h]
  \centering
  \includegraphics[width=0.7\linewidth]{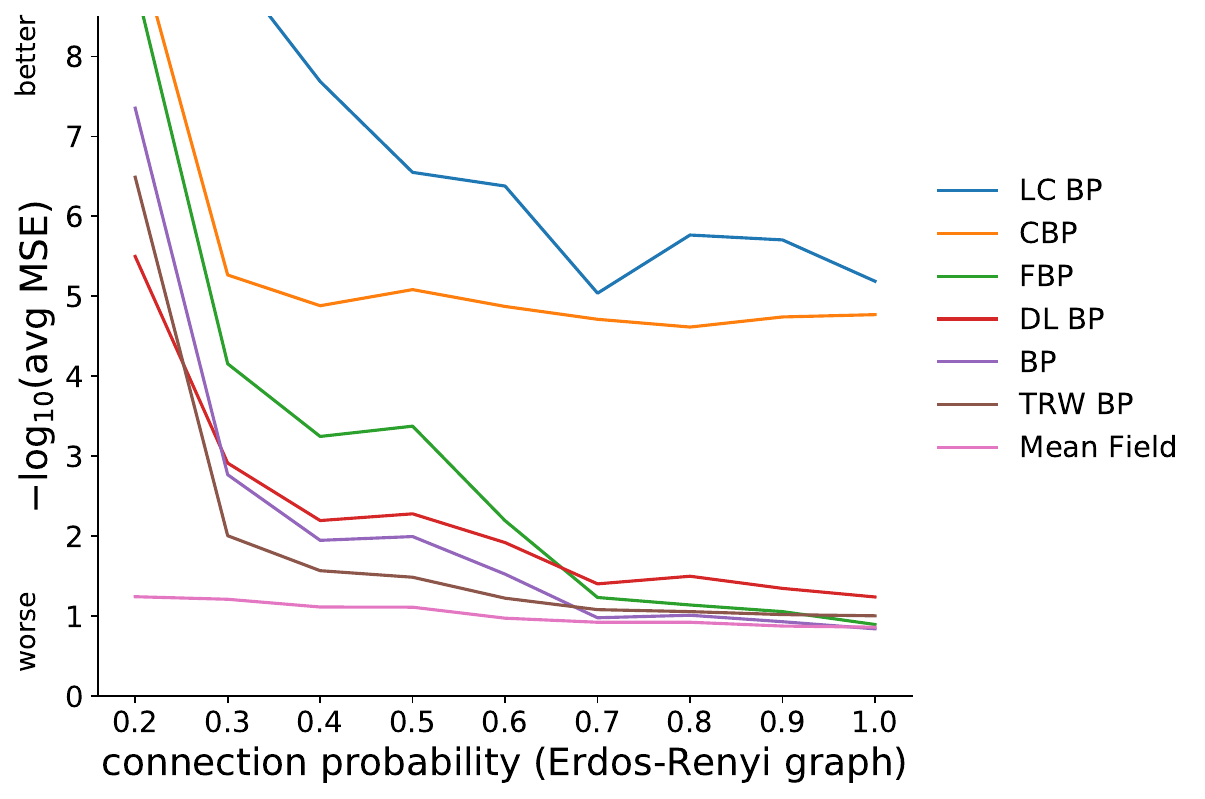}
  \caption{\textbf{Comparison between various algorithms}. 
  Circular BP strongly outperforms the algorithms from the same family (Fractional BP, 
  BP, Tree-Reweighted BP and Mean-Field) and the more complex and Double-Loop BP. CBP has comparable performance for dense graphs with another more complex algorithm, Loop-Corrected BP. 
  The score measure is given in Equation \eqref{eq:def-score-measure}.
  }
  \label{fig:errors-BP-CBP-FBP-eBP-eCBP-eFBP}
\end{figure}


We compare quantitatively Circular BP to various approximate inference algorithms in Figure \ref{fig:errors-BP-CBP-FBP-eBP-eCBP-eFBP}.
Circular BP outperforms Fractional BP, itself outperforming BP. Logically, all algorithms show a decreased performance for increased graph density, but at a higher rate for BP.  
Additionally, we compared Circular BP to two more complex approximate inference algorithms: Loop-corrected BP \citep{Mooij2007b}, which approximates the cavity distributions for each variable in a two-step way, 
and Double-Loop BP \citep{Heskes2003}
, which uses constrained minimization of the Bethe free energy itself in a double-loop procedure guaranteeing the convergence of the algorithm
. We used the libDAI implementation of these algorithms \citep{libDAI2010} (option ``LCBP\_FULLCAVin\_SEQRND" for LC BP and ``HAK\_BETHE" for DL BP). 
Circular BP performs worse than Loop-corrected BP although reaches comparable performance for dense graphs, and strongly outperforms Double-Loop BP.

\paragraph{Results of unsupervised learning} 
The unsupervised learning rule achieves good performance in all tested cases; see Figure \ref{fig:CBP-learning-results}. 
We used 5000 training examples, and learning rates starting from $\eta_1 = 0.03$ and $\eta_2 = 0.0003$ (which were both decreased by half after one third of the optimization, and after two thirds of the optimization).
Note that we add some damping to the algorithm (see Appendices), $\bm{\epsilon} = \bm{0.7}$
, contrary to the supervised learning case.

\subsection{Application to computer vision}

We tackle a denoising problem on the MNIST database of handwritten digits \citep{Deng2012}.
Initial images $\{\bm{\xi}^{k}\}$ are vectors of size $n=784$ ($28$ by $28$ pixels) composed of gray levels initially from $0$ to $255$, rescaled between $-1$ and $+1$. Noisy images $\{\tilde{\bm{\xi}}^k\}$ have half of their pixels set at value zero. 

To reconstruct images, we first use the Hopfield network \citep{Hopfield1982} 
with weights defined by the pseudoinverse rule \citep{Personnaz1986} based on $N=60000$ training images: $\bm{\theta} = \frac{1}{n} \bm{\xi} \bm{\xi}^I$ where $\bm{\xi}^I$ is the pseudo-inverse of matrix $\bm{\xi}$ (whose columns are the training examples $\{\bm{\xi}^k\}$). 
The Hopfield network consists of initializing $\bm{x}_0 = \tilde{\bm{\xi}}^k$ and running $\bm{x}_{t+1} = \sgn(\bm{\theta} \bm{x}_t)$ until convergence; see Figure \ref{fig:applications-CBP}B for 5 examples of reconstruction of noisy test data. 

Interestingly, the energy function of a Hopfield network has the same form as the one of an Ising model. Therefore, probabilistic inference methods can be used for this task. 
We apply BP to an all-to-all connected graph of 784 nodes with interactions $\bm{J} = 5 \bm{\theta}$ (rescaling of the Hopfield weights) and external inputs $\bm{M}_{\text{ext}}^k = 3 \tilde{\bm{\xi}}^k$ (rescaling of the noisy images); Figure \ref{fig:applications-CBP}B shows the marginals rescaled between $-1$ and $+1$.
Further, we apply CBP, trained using the unsupervised learning procedure described above (with damping $\epsilon = 0.8$ on CBP) with 1000 training examples consisting of pure noise: $\tilde{\bm{\xi}}^k \sim \mathcal{N}(\bm{0}, \textbf{Id})$. CBP performs better 
than BP and the Hopfield network; see Figure \ref{fig:applications-CBP}B.
This shows the applicability of the unsupervised learning method on a large dense graph (784 nodes).

\begin{figure}[h]
    \centering
    \includegraphics[width=6cm]{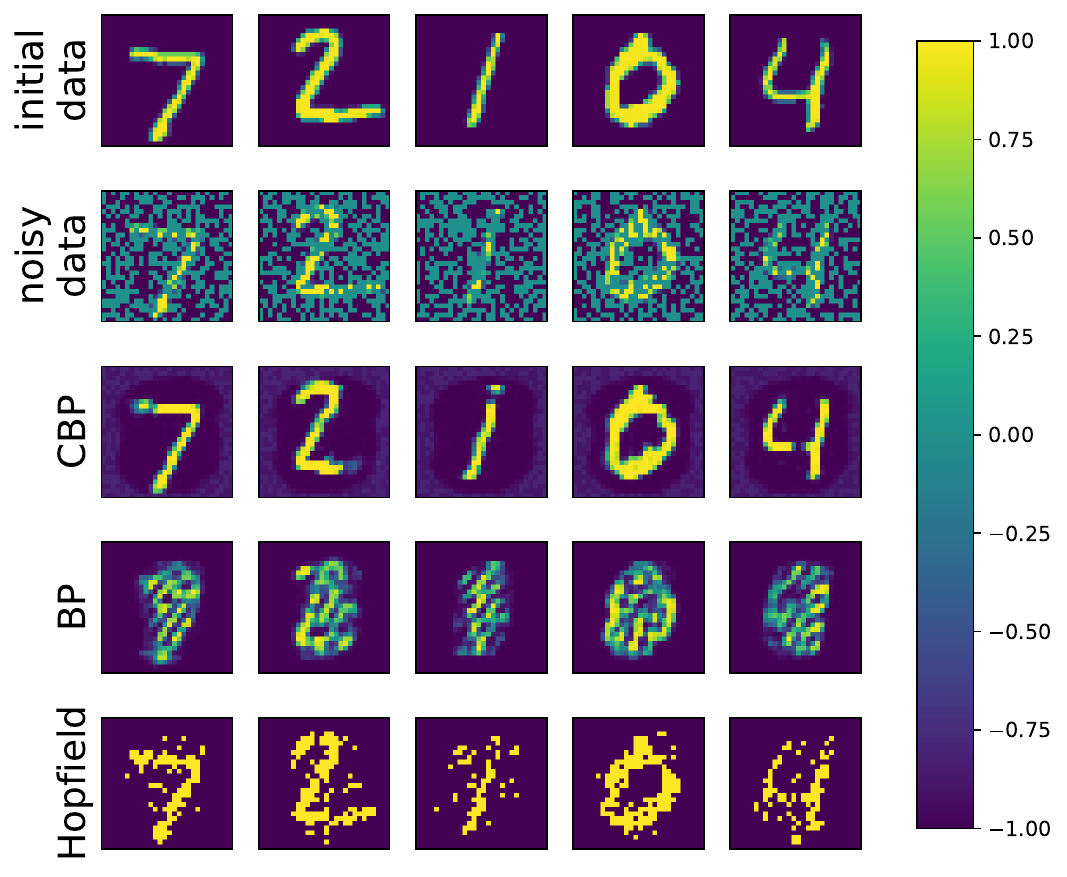}%
    \caption{Potential application of Circular BP: Denoising hand-written digits for computer vision. Noisy data is presented to the Hopfield network / BP / CBP, which reconstruct a denoised version of the signal. 
    }
    \label{fig:applications-CBP}
\end{figure}

\section{Discussions and conclusion}\label{sec:discussion}


In this paper, we present an approximate inference method on undirected graphical models, the \emph{Circular Belief Propagation} algorithm, defined for general probability distributions on any type of variables (discrete or continuous) and possibly with higher-order interactions. 

The algorithm is easily interpretable: CBP consists of disturbing the exchange of information in cycles of length two in order to partly compensate for the detrimental effect of cycles of length 3 and more, which create spurious correlations between messages in BP. 
One can also talk of ``anti-Circular BP".
We motivate the Circular BP algorithm theoretically, by relating it to Fractional BP and other Reweighted BP algorithms based on a reweighting of terms in BP's Bethe free energy, although Circular BP itself does not come from such an reweighting, nor minimizes a free energy. 

We propose a way of learning the parameters of Circular BP using a supervised learning procedure (see section \ref{sec:learning-CBP}). However, one limitation of this procedure is the need to generate exact marginals, which is infeasible in large graphs. The unsupervised learning method to learn the corrective multiplicative factors brings a solution to this issue. 
It is less accurate than the supervised learning rule, but it is not limited by the network size and complexity.
The unsupervised learning rules are intuitive: they 
relate to the principles of efficient information transmission, by trying to remove redundant probabilistic information created by cycles.
They maintain the gain of activity and remove the prediction of the incoming information, similar to predictive coding \citep{Rao1999}. Indeed, Circular BP can be written $M_{i \to j} = f(B_i - \hat{B}_i^j, \beta_{ij} J_{ij})$ where $\hat{B}_i^j$ is the prediction of the state of $x_i$ by variable $x_j$.

We show the applications of the model to a variety of problems involving binary variables, both synthetic and from the real world. 
Overall, CBP improves the quality of probabilistic inference in cyclic graphs compared to BP. It also converges on problems on which BP fails.
CBP achieves near-optimal performance, even if BP does not converge 
(see Figure \ref{fig:CBP-learning-results}) thanks to the existence of parameters ensuring the convergence of the algorithm. 
In all cases, Circular BP hugely outperforms existing message-passing algorithms operating on the probabilistic graph itself, like BP and Fractional BP. 
We show in simulations, by considering particular probability distributions (Ising models) that the algorithm outperforms not only BP but also other message-passing algorithms such as Fractional BP \citep{Wiegerinck2002} and Power EP \citep{Minka2004}. 
Results of inference using Circular BP are impressively good when using supervised learning.  


\paragraph{Applications of the method}
The paper presents results in the restricted binary case. 
However, BP is used in this case in many real-world applications. 
This includes the MIMO (multiple input, multiple output) detection problem for wireless communication where variants of BP are in use in Wi-Fi and 5G protocols \citep{Som2010, Yang2015, Liu2018}
, error-correcting codes \citep{McEliece1998}, 
and computer vision \citep{George2017}, as well as many applications of Restricted Boltzmann Machines like sampling and learning \citep{LazaroGredilla2021}.
The improvement of Circular BP over BP and its guarantee of existence of convergence conditions may enable approximate probabilistic inference in a wide range of new applications. Note that some applications could be harmful to society: better inference systems could be used either for the best or for the worst.

\paragraph{Inference in the brain}

It has been proposed that the BP algorithm is implemented in the brain to carry out inferences based on partial sensory (and/or cognitive) information \citep{Shon2005, Deneve2008, Steimer2009}. Circular BP, as an extension of BP, could also be implemented in the brain.
\citet{Jardri2013} use a particular case of CBP to
account for aberrant beliefs in general, and psychosis in particular \citep{Jardri2013, DeneveJardri2016, Jardri2017, Leptourgos2017, Leptourgos2021, Bouttier2021}, based on the fact that $\alpha < 1$ brings overconfidence to the estimated marginals compared to BP. 
Importantly, the algorithm from \citet{Jardri2013} is not motivated normatively but only at the intuitive level 
\citep{Deneve2005, Deneve2008} as a way of allowing for information reverberations with the underlying idea of excitation-inhibition imbalance (by introducing positive feedback between excitatory populations).


\paragraph{Machine learning research}

This work paves the way for future research on approximate inference. 
Message-passing algorithms have recently been receiving a lot of attention from the machine learning community with the success of Graph Neural Networks (GNN), 
Message-Passing Neural Networks, and variants inspired from BP \citep{Gilmer2017, Kuck2020, Satorras2021}. 
We believe that the idea behind Circular BP (removing only part of the opposite message) can improve further these message-passing algorithms. 
In fact, Equation \eqref{eq:eCBP-message-log} 
provides a general recipe to use correction in message-passing algorithms, including mean-field methods, which typically write $x_i = \sum_j f(x_j,\, J_{ji})$ and could be translated into $x_i = \sum_j m_{ji}$ together with $m_{ji} = f(x_j - \alpha_{ji} m_{ij},\, J_{ji})$.
This ``corrective multiplicative factor trick" can be used on top of 
algorithms which already improve BP while conserving the same form,
like \citet{Kuck2020} or the message-GNN of \citet{Yoon2018},
potentially leading to further improvement of these methods.





\bibliography{biblio} 

\begin{thebibliography}{78}
\providecommand{\natexlab}[1]{#1}
\providecommand{\url}[1]{\texttt{#1}}
\expandafter\ifx\csname urlstyle\endcsname\relax
  \providecommand{\doi}[1]{doi: #1}\else
  \providecommand{\doi}{doi: \begingroup \urlstyle{rm}\Url}\fi

\bibitem[Ankan and Panda(2015)]{pgmpy2015}
Ankur Ankan and Abinash Panda.
\newblock pgmpy: Probabilistic graphical models using python.
\newblock In \emph{Proceedings of the 14th Python in Science Conference (SCIPY 2015)}. Citeseer, 2015.

\bibitem[Baxter(1982)]{Baxter1982}
R.~J. Baxter.
\newblock \emph{{Exactly solved models in statistical mechanics}}.
\newblock 1982.
\newblock ISBN 978-0-486-46271-4.

\bibitem[Bouttier et~al.(2021)Bouttier, Duttagupta, Den{\`{e}}ve, and Jardri]{Bouttier2021}
Vincent Bouttier, Suhrit Duttagupta, Sophie Den{\`{e}}ve, and Renaud Jardri.
\newblock Circular inference predicts nonuniform overactivation and dysconnectivity in brain-wide connectomes.
\newblock \emph{Schizophrenia Research}, 2021.
\newblock ISSN 0920-9964.
\newblock \doi{https://doi.org/10.1016/j.schres.2020.12.045}.

\bibitem[Dempster et~al.(1977)Dempster, Laird, and Rubin]{Dempster1977}
A.~P. Dempster, N.~M. Laird, and D.~B. Rubin.
\newblock Maximum likelihood from incomplete data via the em algorithm.
\newblock \emph{Journal of the Royal Statistical Society, Series B}, 39\penalty0 (1):\penalty0 1--38, 1977.

\bibitem[Den{\`{e}}ve(2005)]{Deneve2005}
Sophie Den{\`{e}}ve.
\newblock {Bayesian inference in spiking neurons}.
\newblock In \emph{Advances in neural information processing systems}, 2005.

\bibitem[Den{\`{e}}ve(2008)]{Deneve2008}
Sophie Den{\`{e}}ve.
\newblock {Bayesian Spiking Neurons I: Inference}.
\newblock \emph{Neural Computation}, 20\penalty0 (1):\penalty0 91--117, 2008.
\newblock ISSN 0899-7667.
\newblock \doi{10.1162/neco.2008.20.1.91}.

\bibitem[Den{\`{e}}ve and Jardri(2016)]{DeneveJardri2016}
Sophie Den{\`{e}}ve and Renaud Jardri.
\newblock {Circular inference: mistaken belief, misplaced trust}.
\newblock \emph{Current Opinion in Behavioral Sciences}, 11:\penalty0 40--48, oct 2016.
\newblock ISSN 2352-1546.
\newblock \doi{10.1016/J.COBEHA.2016.04.001}.

\bibitem[Deng(2012)]{Deng2012}
Li~Deng.
\newblock The mnist database of handwritten digit images for machine learning research.
\newblock \emph{IEEE Signal Processing Magazine}, 29\penalty0 (6):\penalty0 141--142, 2012.

\bibitem[Garcia~Satorras and Welling(2021)]{Satorras2021}
V{\'i}ctor Garcia~Satorras and Max Welling.
\newblock Neural enhanced belief propagation on factor graphs.
\newblock In Arindam Banerjee and Kenji Fukumizu, editors, \emph{Proceedings of The 24th International Conference on Artificial Intelligence and Statistics}, volume 130 of \emph{Proceedings of Machine Learning Research}, pages 685--693. PMLR, 2021.

\bibitem[George et~al.(2017)George, Lehrach, Kansky, Lázaro-Gredilla, Laan, Marthi, Lou, Meng, Liu, Wang, Lavin, and Phoenix]{George2017}
Dileep George, Wolfgang Lehrach, Ken Kansky, Miguel Lázaro-Gredilla, Christopher Laan, Bhaskara Marthi, Xinghua Lou, Zhaoshi Meng, Yi~Liu, Huayan Wang, Alex Lavin, and D.~Scott Phoenix.
\newblock A generative vision model that trains with high data efficiency and breaks text-based captchas.
\newblock \emph{Science}, 358\penalty0 (6368):\penalty0 eaag2612, 2017.
\newblock \doi{10.1126/science.aag2612}.

\bibitem[Gilmer et~al.(2017)Gilmer, Schoenholz, Riley, Vinyals, and Dahl]{Gilmer2017}
Justin Gilmer, Samuel~S. Schoenholz, Patrick~F. Riley, Oriol Vinyals, and George~E. Dahl.
\newblock {Neural Message Passing for Quantum Chemistry}.
\newblock \emph{34th International Conference on Machine Learning, ICML 2017}, 3:\penalty0 2053--2070, apr 2017.

\bibitem[Hazan and Shashua(2008)]{Hazan2008}
Tamir Hazan and Amnon Shashua.
\newblock {Convergent Message-Passing Algorithms for Inference over General Graphs with Convex Free Energies}.
\newblock In \emph{Uncertainty in Artificial Intelligence}, 2008.

\bibitem[Hazan and Shashua(2010)]{Hazan2010}
Tamir Hazan and Amnon Shashua.
\newblock {Norm-product belief propagation: Primal-dual message-passing for approximate inference}.
\newblock \emph{IEEE Transactions on Information Theory}, 56\penalty0 (12):\penalty0 6294--6316, dec 2010.
\newblock ISSN 00189448.
\newblock \doi{10.1109/TIT.2010.2079014}.

\bibitem[Heskes(2002)]{Heskes2002}
Tom Heskes.
\newblock {Stable Fixed Points of Loopy Belief Propagation Are Minima of the Bethe Free Energy}.
\newblock In \emph{Advances in neural information processing systems}, 2002.

\bibitem[Heskes(2004)]{Heskes2004}
Tom Heskes.
\newblock {On the uniqueness of loopy belief propagation fixed points}.
\newblock \emph{Neural Computation}, 16\penalty0 (11):\penalty0 2379--2413, 2004.
\newblock ISSN 08997667.
\newblock \doi{10.1162/0899766041941943}.

\bibitem[Heskes et~al.(2002)Heskes, Albers, and Kappen]{Heskes2003}
Tom Heskes, Kees Albers, and Bert Kappen.
\newblock Approximate inference and constrained optimization.
\newblock In \emph{Proceedings of the Nineteenth Conference on Uncertainty in Artificial Intelligence}, UAI'03, page 313–320, San Francisco, CA, USA, 2002. Morgan Kaufmann Publishers Inc.
\newblock ISBN 0127056645.

\bibitem[Hopfield(1982)]{Hopfield1982}
J~J Hopfield.
\newblock Neural networks and physical systems with emergent collective computational abilities.
\newblock \emph{Proceedings of the National Academy of Sciences}, 79\penalty0 (8):\penalty0 2554--2558, 1982.
\newblock \doi{10.1073/pnas.79.8.2554}.

\bibitem[Ihler and McAllester(2009)]{Ihler2009}
Alexander Ihler and David McAllester.
\newblock Particle belief propagation.
\newblock In David van Dyk and Max Welling, editors, \emph{Proceedings of the Twelth International Conference on Artificial Intelligence and Statistics}, volume~5 of \emph{Proceedings of Machine Learning Research}, pages 256--263, Hilton Clearwater Beach Resort, Clearwater Beach, Florida USA, 16--18 Apr 2009. PMLR.

\bibitem[Ihler(2007)]{Ihler2007}
Alexander~T Ihler.
\newblock {Accuracy Bounds for Belief Propagation}.
\newblock In \emph{Proceedings of the Twenty-Third Conference on Uncertainty in Artificial Intelligence}, UAI'07, pages 183--190, Arlington, Virginia, USA, 2007. AUAI Press.
\newblock ISBN 0974903930.

\bibitem[{Ising}(1925)]{Ising1925}
Ernst {Ising}.
\newblock {Beitrag zur Theorie des Ferromagnetismus}.
\newblock \emph{Zeitschrift fur Physik}, 31\penalty0 (1):\penalty0 253--258, 1925.
\newblock \doi{10.1007/BF02980577}.

\bibitem[Jardri and Den{\`{e}}ve(2013)]{Jardri2013}
Renaud Jardri and Sophie Den{\`{e}}ve.
\newblock {Circular inferences in schizophrenia}.
\newblock \emph{Brain}, 136\penalty0 (11):\penalty0 3227--3241, 2013.
\newblock ISSN 14602156.
\newblock \doi{10.1093/brain/awt257}.

\bibitem[Jardri et~al.(2017)Jardri, Duverne, Litvinova, and Den{\`{e}}ve]{Jardri2017}
Renaud Jardri, Sandrine Duverne, Alexandra~S. Litvinova, and Sophie Den{\`{e}}ve.
\newblock {Experimental evidence for circular inference in schizophrenia}.
\newblock \emph{Nature Communications}, 8:\penalty0 14218, jan 2017.
\newblock ISSN 2041-1723.
\newblock \doi{10.1038/ncomms14218}.

\bibitem[Knoll and Pernkopf(2017)]{Knoll2017}
Christian Knoll and Franz Pernkopf.
\newblock {On Loopy Belief Propagation - Local Stability Analysis for Non-Vanishing Fields}.
\newblock In \emph{Proceedings of the conference on Uncertainty in Artificial Intelligence}, 2017.

\bibitem[Koller and Friedman(2009)]{Koller2009}
D.~Koller and N.~Friedman.
\newblock \emph{Probabilistic Graphical Models: Principles and Techniques}.
\newblock Adaptive computation and machine learning. MIT Press, 2009.
\newblock ISBN 9780262013192.

\bibitem[Kschischang et~al.(2001)Kschischang, Frey, and Loeliger]{Kschischang2001}
F.R. Kschischang, B.J. Frey, and H.-A. Loeliger.
\newblock Factor graphs and the sum-product algorithm.
\newblock \emph{IEEE Transactions on Information Theory}, 47\penalty0 (2):\penalty0 498--519, 2001.
\newblock \doi{10.1109/18.910572}.

\bibitem[Kuck et~al.(2020)Kuck, Chakraborty, Tang, Luo, Song, Sabharwal, and Ermon]{Kuck2020}
Jonathan Kuck, Shuvam Chakraborty, Hao Tang, Rachel Luo, Jiaming Song, Ashish Sabharwal, and Stefano Ermon.
\newblock Belief propagation neural networks.
\newblock In H.~Larochelle, M.~Ranzato, R.~Hadsell, M.~F. Balcan, and H.~Lin, editors, \emph{Advances in Neural Information Processing Systems}, volume~33, pages 667--678. Curran Associates, Inc., 2020.

\bibitem[Lazaro-Gredilla et~al.(2021)Lazaro-Gredilla, Dedieu, and George]{LazaroGredilla2021}
Miguel Lazaro-Gredilla, Antoine Dedieu, and Dileep George.
\newblock {Perturb-and-max-product: Sampling and learning in discrete energy-based models}.
\newblock In M~Ranzato, A~Beygelzimer, Y~Dauphin, P~S Liang, and J~Wortman Vaughan, editors, \emph{Advances in Neural Information Processing Systems}, volume~34, pages 928--940. Curran Associates, Inc., 2021.

\bibitem[Leptourgos et~al.(2017)Leptourgos, Den{\`{e}}ve, and Jardri]{Leptourgos2017}
Pantelis Leptourgos, Sophie Den{\`{e}}ve, and Renaud Jardri.
\newblock {Can circular inference relate the neuropathological and behavioral aspects of schizophrenia?}
\newblock \emph{Current Opinion in Neurobiology}, 46:\penalty0 154--161, oct 2017.
\newblock ISSN 0959-4388.
\newblock \doi{10.1016/J.CONB.2017.08.012}.

\bibitem[Leptourgos et~al.(2021)Leptourgos, Bouttier, Den{\`{e}}ve, and Jardri]{Leptourgos2021}
Pantelis Leptourgos, Vincent Bouttier, Sophie Den{\`{e}}ve, and Renaud Jardri.
\newblock {From hallucinations to synaesthesia: a circular inference account of unimodal and multimodal erroneous percepts in clinical and drug-induced psychosis}.
\newblock \emph{PsyArXiv}, 2021.

\bibitem[Litvak et~al.(2009)Litvak, Karlinsky, and Ullman]{Litvak2009a}
Shai Litvak, Leonid Karlinsky, and Shimon Ullman.
\newblock {Properties of Cortical Networks Improve Inference in Highly Interconnected Graphical Models}.
\newblock 2009.

\bibitem[Liu et~al.(2019)Liu, Moghadam, Rasmussen, Huang, and Chatterjee]{Liu2019}
Dong Liu, Nima~N. Moghadam, Lars~K. Rasmussen, Jinliang Huang, and Saikat Chatterjee.
\newblock {$\alpha$ Belief Propagation as Fully Factorized Approximation}.
\newblock In \emph{2019 IEEE Global Conference on Signal and Information Processing (GlobalSIP)}, pages 1--5, 2019.
\newblock \doi{10.1109/GlobalSIP45357.2019.8969545}.

\bibitem[Liu et~al.(2020)Liu, Vu, Li, and Rasmussen]{Liu2020}
Dong Liu, Minh~Th{\`{a}}nh Vu, Zuxing Li, and Lars~K. Rasmussen.
\newblock {$\alpha$ Belief propagation for approximate inference}.
\newblock \emph{arXiv}, 2020.
\newblock ISSN 23318422.

\bibitem[Liu and Li(2018)]{Liu2018}
Xiangfeng Liu and Ying Li.
\newblock Deep mimo detection based on belief propagation.
\newblock In \emph{2018 IEEE Information Theory Workshop (ITW)}, pages 1--5, 2018.
\newblock \doi{10.1109/ITW.2018.8613336}.

\bibitem[Loh and Wibisono(2014)]{Loh2014}
Po-Ling Loh and Andre Wibisono.
\newblock {Concavity of reweighted Kikuchi approximation}.
\newblock In Z~Ghahramani, M~Welling, C~Cortes, N~Lawrence, and K~Q Weinberger, editors, \emph{Advances in Neural Information Processing Systems}, volume~27. Curran Associates, Inc., 2014.

\bibitem[McEliece et~al.(1998)McEliece, MacKay, and Cheng]{McEliece1998}
R.J. McEliece, D.J.C. MacKay, and Jung-Fu Cheng.
\newblock Turbo decoding as an instance of pearl's "belief propagation" algorithm.
\newblock \emph{IEEE Journal on Selected Areas in Communications}, 16\penalty0 (2):\penalty0 140--152, 1998.
\newblock \doi{10.1109/49.661103}.

\bibitem[Minka(2004)]{Minka2004}
Thomas Minka.
\newblock {Power EP}.
\newblock Technical report, 2004.

\bibitem[Minka(2005)]{Minka2005}
Thomas Minka.
\newblock {Divergence measures and message passing}.
\newblock Technical report, 2005.

\bibitem[Minka and Lafferty(2002)]{Minka2002}
Thomas Minka and John Lafferty.
\newblock Expectation-propagation for the generative aspect model.
\newblock In \emph{Proceedings of the Eighteenth Conference on Uncertainty in Artificial Intelligence}, UAI'02, page 352–359, San Francisco, CA, USA, 2002. Morgan Kaufmann Publishers Inc.
\newblock ISBN 1558608974.

\bibitem[Minka(2001)]{Minka2001a}
Thomas~P. Minka.
\newblock {Expectation propagation for approximate Bayesian inference}.
\newblock In \emph{Uncertainty in Artificial Intelligence}, volume~17, pages 362--369, 2001.

\bibitem[Mongillo and Deneve(2008)]{Mongillo2008}
Gianluigi Mongillo and Sophie Deneve.
\newblock {Online Learning with Hidden Markov Models}.
\newblock \emph{Neural Computation}, 20\penalty0 (7):\penalty0 1706--1716, 07 2008.
\newblock ISSN 0899-7667.
\newblock \doi{10.1162/neco.2008.10-06-351}.

\bibitem[Mooij(2010)]{libDAI2010}
Joris~M. Mooij.
\newblock lib{DAI}: A free and open source {C++} library for discrete approximate inference in graphical models.
\newblock \emph{Journal of Machine Learning Research}, 11:\penalty0 2169--2173, August 2010.

\bibitem[Mooij and Kappen(2009)]{Mooij2009}
Joris~M Mooij and Hilbert Kappen.
\newblock {Bounds on marginal probability distributions}.
\newblock In D~Koller, D~Schuurmans, Y~Bengio, and L~Bottou, editors, \emph{Advances in Neural Information Processing Systems}, volume~21. Curran Associates, Inc., 2009.

\bibitem[Mooij and Kappen(2007{\natexlab{a}})]{Mooij2007}
Joris~M. Mooij and Hilbert~J. Kappen.
\newblock Sufficient conditions for convergence of the sum–product algorithm.
\newblock \emph{IEEE Transactions on Information Theory}, 53\penalty0 (12):\penalty0 4422--4437, 2007{\natexlab{a}}.
\newblock \doi{10.1109/TIT.2007.909166}.

\bibitem[Mooij and Kappen(2007{\natexlab{b}})]{Mooij2007b}
Joris~M. Mooij and Hilbert~J. Kappen.
\newblock Loop corrections for approximate inference on factor graphs.
\newblock \emph{J. Mach. Learn. Res.}, 8:\penalty0 1113–1143, December 2007{\natexlab{b}}.
\newblock ISSN 1532-4435.

\bibitem[Mooij and Kappen(2004)]{Mooij2004}
Joris~Marten Mooij and Hilbert~J Kappen.
\newblock {Validity estimates for loopy Belief Propagation on binary real-world networks}.
\newblock In \emph{Advances in neural information processing systems}, 2004.

\bibitem[Murphy et~al.(1999)Murphy, Weiss, and Jordan]{Murphy1999}
Kevin~P. Murphy, Yair Weiss, and Michael~I. Jordan.
\newblock Loopy belief propagation for approximate inference: An empirical study.
\newblock In \emph{Proceedings of the Fifteenth Conference on Uncertainty in Artificial Intelligence}, UAI'99, page 467–475, San Francisco, CA, USA, 1999. Morgan Kaufmann Publishers Inc.
\newblock ISBN 1558606149.

\bibitem[Paszke et~al.(2019)Paszke, Gross, Massa, Lerer, Bradbury, Chanan, Killeen, Lin, Gimelshein, Antiga, Desmaison, Kopf, Yang, DeVito, Raison, Tejani, Chilamkurthy, Steiner, Fang, Bai, and Chintala]{PyTorch2019}
Adam Paszke, Sam Gross, Francisco Massa, Adam Lerer, James Bradbury, Gregory Chanan, Trevor Killeen, Zeming Lin, Natalia Gimelshein, Luca Antiga, Alban Desmaison, Andreas Kopf, Edward Yang, Zachary DeVito, Martin Raison, Alykhan Tejani, Sasank Chilamkurthy, Benoit Steiner, Lu~Fang, Junjie Bai, and Soumith Chintala.
\newblock Pytorch: An imperative style, high-performance deep learning library.
\newblock In H.~Wallach, H.~Larochelle, A.~Beygelzimer, F.~d\textquotesingle Alch\'{e}-Buc, E.~Fox, and R.~Garnett, editors, \emph{Advances in Neural Information Processing Systems 32}, pages 8024--8035. Curran Associates, Inc., 2019.

\bibitem[Pearl(1988)]{Pearl1988}
Judea Pearl.
\newblock \emph{Probabilistic Reasoning in Intelligent Systems: Networks of Plausible Inference}.
\newblock Morgan Kaufmann Publishers Inc., San Francisco, CA, USA, 1988.
\newblock ISBN 0934613737.

\bibitem[Personnaz et~al.(1986)Personnaz, Guyon, and Dreyfus]{Personnaz1986}
L.~Personnaz, I.~Guyon, and G.~Dreyfus.
\newblock Collective computational properties of neural networks: New learning mechanisms.
\newblock \emph{Phys. Rev. A}, 34:\penalty0 4217--4228, 1986.
\newblock \doi{10.1103/PhysRevA.34.4217}.

\bibitem[Peterson and Anderson(1987)]{Peterson1987}
C.~Peterson and J.~R. Anderson.
\newblock A mean field theory learning algorithm for neural networks.
\newblock \emph{Complex Systems}, 1:\penalty0 995--1019, 1987.

\bibitem[Rao and Ballard(1999)]{Rao1999}
Rajesh~P.N. Rao and Dana~H. Ballard.
\newblock {Predictive coding in the visual cortex: a functional interpretation of some extra-classical receptive-field effects}.
\newblock \emph{Nature Neuroscience}, 2\penalty0 (1):\penalty0 79--87, 1999.
\newblock ISSN 1097-6256.
\newblock \doi{10.1038/4580}.

\bibitem[Riedmiller and Braun(1993)]{Rprop1993}
M.~Riedmiller and H.~Braun.
\newblock A direct adaptive method for faster backpropagation learning: the rprop algorithm.
\newblock In \emph{IEEE International Conference on Neural Networks}, pages 586--591 vol.1, 1993.
\newblock \doi{10.1109/ICNN.1993.298623}.

\bibitem[Roosta et~al.(2008)Roosta, Wainwright, and Sastry]{Roosta2008}
Tanya~G. Roosta, Martin~J. Wainwright, and Shankar~S. Sastry.
\newblock {Convergence analysis of reweighted sum-product algorithms}.
\newblock \emph{IEEE Transactions on Signal Processing}, 56\penalty0 (9):\penalty0 4293--4305, 2008.
\newblock \doi{10.1109/TSP.2008.924136}.

\bibitem[Shi et~al.(2010)Shi, Schonfeld, and Tuninetti]{Shi2010}
Xiangqiong Shi, Dan Schonfeld, and Daniela Tuninetti.
\newblock Message error analysis of loopy belief propagation for the sum-product algorithm.
\newblock \emph{CoRR}, abs/1009.2305, 2010.

\bibitem[Shon and Rao(2005)]{Shon2005}
Aaron~P. Shon and Rajesh~P.N. Rao.
\newblock {Implementing belief propagation in neural circuits}.
\newblock \emph{Neurocomputing}, 65-66:\penalty0 393--399, 2005.
\newblock ISSN 09252312.
\newblock \doi{10.1016/j.neucom.2004.10.035}.

\bibitem[Som et~al.(2010)Som, Datta, Chockalingam, and Rajan]{Som2010}
Pritam Som, Tanumay Datta, A.~Chockalingam, and B.~Sundar Rajan.
\newblock Improved large-mimo detection based on damped belief propagation.
\newblock In \emph{2010 IEEE Information Theory Workshop on Information Theory (ITW 2010, Cairo)}, pages 1--5, 2010.
\newblock \doi{10.1109/ITWKSPS.2010.5503188}.

\bibitem[Steimer et~al.(2009)Steimer, Maass, and Douglas]{Steimer2009}
Andreas Steimer, Wolfgang Maass, and Rodney Douglas.
\newblock {Belief propagation in networks of spiking neurons}.
\newblock \emph{Neural Computation}, 21\penalty0 (219), 2009.

\bibitem[Sudderth et~al.(2003)Sudderth, Ihler, Freeman, and Willsky]{Sudderth2003}
Erik~B. Sudderth, Alexander~T. Ihler, William~T. Freeman, and Alan~S. Willsky.
\newblock {Nonparametric belief propagation}.
\newblock In \emph{Proceedings of the IEEE Computer Society Conference on Computer Vision and Pattern Recognition}, volume~1, 2003.
\newblock \doi{10.1109/cvpr.2003.1211409}.

\bibitem[Taga and Mase(2006)]{Taga2006}
Nobuyuki Taga and Shigeru Mase.
\newblock Error bounds between marginal probabilities and beliefs of loopy belief propagation algorithm.
\newblock In \emph{Proceedings of the 5th Mexican International Conference on Artificial Intelligence}, MICAI'06, page 186–196, Berlin, Heidelberg, 2006. Springer-Verlag.
\newblock ISBN 3540490264.
\newblock \doi{10.1007/11925231\_18}.

\bibitem[Virtanen et~al.(2020)Virtanen, Gommers, Oliphant, Haberland, Reddy, Cournapeau, Burovski, Peterson, Weckesser, Bright, {van der Walt}, Brett, Wilson, Millman, Mayorov, Nelson, Jones, Kern, Larson, Carey, Polat, Feng, Moore, {VanderPlas}, Laxalde, Perktold, Cimrman, Henriksen, Quintero, Harris, Archibald, Ribeiro, Pedregosa, {van Mulbregt}, and {SciPy 1.0 Contributors}]{SciPy2020}
Pauli Virtanen, Ralf Gommers, Travis~E. Oliphant, Matt Haberland, Tyler Reddy, David Cournapeau, Evgeni Burovski, Pearu Peterson, Warren Weckesser, Jonathan Bright, St{\'e}fan~J. {van der Walt}, Matthew Brett, Joshua Wilson, K.~Jarrod Millman, Nikolay Mayorov, Andrew R.~J. Nelson, Eric Jones, Robert Kern, Eric Larson, C~J Carey, {\.I}lhan Polat, Yu~Feng, Eric~W. Moore, Jake {VanderPlas}, Denis Laxalde, Josef Perktold, Robert Cimrman, Ian Henriksen, E.~A. Quintero, Charles~R. Harris, Anne~M. Archibald, Ant{\^o}nio~H. Ribeiro, Fabian Pedregosa, Paul {van Mulbregt}, and {SciPy 1.0 Contributors}.
\newblock {{SciPy} 1.0: Fundamental Algorithms for Scientific Computing in Python}.
\newblock \emph{Nature Methods}, 17:\penalty0 261--272, 2020.
\newblock \doi{10.1038/s41592-019-0686-2}.

\bibitem[Wainwright et~al.(2005)Wainwright, Jaakkola, and Willsky]{Wainwright2005}
M.~J. Wainwright, T.~S. Jaakkola, and A.~S. Willsky.
\newblock A new class of upper bounds on the log partition function.
\newblock \emph{IEEE Trans. Inf. Theor.}, 51\penalty0 (7):\penalty0 2313–2335, 2005.
\newblock ISSN 0018-9448.
\newblock \doi{10.1109/TIT.2005.850091}.

\bibitem[Wainwright and Jordan(2008)]{WainwrightJordan2008}
Martin~J. Wainwright and Michael~I. Jordan.
\newblock Graphical models, exponential families, and variational inference.
\newblock \emph{Foundations and Trends in Machine Learning}, 1\penalty0 (1–2):\penalty0 1–305, 2008.
\newblock ISSN 1935-8237.
\newblock \doi{10.1561/2200000001}.

\bibitem[Wainwright. et~al.(2002)Wainwright., Jaakkola, and Willsky]{Wainwright2002}
Martin~J. Wainwright., Tommi Jaakkola, and Alan~S. Willsky.
\newblock {Tree-based reparameterization for approximate inference on loopy graphs}.
\newblock \emph{Advances in Neural Information Processing Systems}, 2002.
\newblock ISSN 10495258.

\bibitem[Wainwright. et~al.(2003)Wainwright., Jaakkola, and Willsky]{Wainwright2003}
Martin~J. Wainwright., Tommi Jaakkola, and Alan~S. Willsky.
\newblock {Tree-reweighted belief propagation algorithms and approximate ML estimation by pseudo-moment matching}.
\newblock \emph{Workshop on Artificial Intelligence and Statistics}, 21:\penalty0 97, 2003.

\bibitem[Wainwright et~al.(2003)Wainwright, Jaakkola, and Willsky]{Wainwright2003b}
Martin~J. Wainwright, Tommi Jaakkola, and Alan~S. Willsky.
\newblock Tree-based reparameterization framework for analysis of sum-product and related algorithms.
\newblock \emph{IEEE Transactions on Information Theory}, 49\penalty0 (5):\penalty0 1120--1146, 2003.

\bibitem[Watanabe(2011)]{Watanabe2011}
Yusuke Watanabe.
\newblock {Uniqueness of Belief Propagation on Signed Graphs}.
\newblock In \emph{Advances in neural information processing systems}, 2011.

\bibitem[Watanabe and Fukumizu(2009)]{Watanabe2009}
Yusuke Watanabe and Kenji Fukumizu.
\newblock Graph zeta function in the bethe free energy and loopy belief propagation.
\newblock In Y.~Bengio, D.~Schuurmans, J.~Lafferty, C.~Williams, and A.~Culotta, editors, \emph{Advances in Neural Information Processing Systems}, volume~22. Curran Associates, Inc., 2009.

\bibitem[Weiss(2000)]{Weiss2000}
Yair Weiss.
\newblock {Correctness of Local Probability Propagation in Graphical Models with Loops}.
\newblock \emph{Neural Computation}, 2000.

\bibitem[Weiss(2001)]{Weiss2001}
Yair Weiss.
\newblock Comparing the mean field method and belief propagation for approximate inference in mrfs.
\newblock In Manfred Opper and David Saad, editors, \emph{Advanced Mean Field Methods: Theory and Practice}, pages 229--239. MIT Press, 2001.

\bibitem[Weller et~al.(2014)Weller, Tang, Sontag, and Jebara]{Weller2014}
Adrian Weller, Kui Tang, David Sontag, and Tony Jebara.
\newblock Understanding the bethe approximation: When and how can it go wrong?
\newblock In \emph{Proceedings of the Thirtieth Conference on Uncertainty in Artificial Intelligence}, UAI'14, page 868–877, Arlington, Virginia, USA, 2014. AUAI Press.
\newblock ISBN 9780974903910.

\bibitem[Wiegerinck and Heskes(2002)]{Wiegerinck2002}
Wim Wiegerinck and Tom Heskes.
\newblock {Fractional Belief Propagation}.
\newblock In \emph{Advances in Neural Information Processing Systems}, volume~15, 2002.

\bibitem[Wilson and Cowan(1972)]{Wilson1972}
Hugh~R. Wilson and Jack~D. Cowan.
\newblock {Excitatory and Inhibitory Interactions in Localized Populations of Model Neurons}.
\newblock \emph{Biophysical Journal}, 12\penalty0 (1):\penalty0 1, 1972.
\newblock \doi{10.1016/S0006-3495(72)86068-5}.

\bibitem[Winn and Bishop(2005)]{Winn2005}
John Winn and Christopher~M. Bishop.
\newblock Variational message passing.
\newblock \emph{Journal of Machine Learning Research}, 6:\penalty0 661--694, 2005.

\bibitem[Yang et~al.(2015)Yang, Zhang, Liang, Xu, and You]{Yang2015}
Junmei Yang, Chuan Zhang, Xiao Liang, Shugong Xu, and Xiaohu You.
\newblock Improved symbol-based belief propagation detection for large-scale mimo.
\newblock In \emph{2015 IEEE Workshop on Signal Processing Systems (SiPS)}, pages 1--6, 2015.
\newblock \doi{10.1109/SiPS.2015.7345035}.

\bibitem[Yedidia et~al.(2001)Yedidia, Freeman, and Weiss]{Yedidia2001}
Jonathan~S Yedidia, William~T Freeman, and Yair Weiss.
\newblock {Generalized Belief Propagation}.
\newblock \emph{Advances in neural information processing systems}, 2001.

\bibitem[Yedidia et~al.(2003)Yedidia, Freeman, and Weiss]{Yedidia2003}
Jonathan~S. Yedidia, William~T. Freeman, and Yair Weiss.
\newblock \emph{Understanding Belief Propagation and Its Generalizations}, page 239–269.
\newblock Morgan Kaufmann Publishers Inc., San Francisco, CA, USA, 2003.
\newblock ISBN 1558608117.

\bibitem[Yedidia et~al.(2005)Yedidia, Freeman, and Weiss]{Yedidia2005}
Jonathan~S. Yedidia, William~T. Freeman, and Yair Weiss.
\newblock {Constructing Free Energy Approximations and Generalized Belief Propagation Algorithms}.
\newblock \emph{IEEE Transactions on Information Theory}, 51\penalty0 (7):\penalty0 2282--2312, 2005.

\bibitem[Yoon et~al.(2018)Yoon, Liao, Xiong, Zhang, Fetaya, Urtasun, Zemel, and Pitkow]{Yoon2018}
KiJung Yoon, Renjie Liao, Yuwen Xiong, Lisa Zhang, Ethan Fetaya, Raquel Urtasun, Richard Zemel, and Xaq Pitkow.
\newblock Inference in probabilistic graphical models by graph neural networks.
\newblock In \emph{ICLR Workshop}, 2018.

\end{thebibliography}
\bibliographystyle{plainnat} 
\newpage
\appendix
\onecolumn
\renewcommand\thefigure{S\arabic{figure}} 
\setcounter{figure}{0}
\addcontentsline{toc}{section}{Appendix} 
\part{Appendices}

\parttoc 

\section{BP and Circular BP: additional figures}\label{sec:bp-additional-figures}

We illustrate in Figure \ref{fig:BP-and-CBP} the Circular Belief Propagation algorithm applied on an Ising model, in the special case where $(\bm{\kappa}, \bm{\beta}, \bm{\gamma}) = (\bm{1}, \bm{1}, \bm{1})$, that is, with parameter $\bm{\alpha}$ only. 
In this case, the Circular BP algorithms writes (special cases of Equations \eqref{eq:eCBP-message-log}, 
\eqref{eq:eCBP-belief-log}, and \eqref{eq:function-fij-Ising-model}):
\begin{empheq}[left=\empheqlbrace]{align}
  &M_{i \to j}^{\text{new}} = f(B_i - \alpha_{ij} M_{j \to i},\, J_{ij})\\ 
  &B_i = \sum\limits_{j \in \mathcal{N}(i)} M_{j \to i} + M_{\text{ext} \to i}
\end{empheq}
where
\begin{equation}
    f(x,\, J) = \arctanh\big(\tanh(J) \tanh(x)\big)
\end{equation}

The Belief Propagation algorithm consists of setting the last parameter to the value 1: $\bm{\alpha} = \bm{1}$.


\begin{figure}[h]
  \centering
  \includegraphics[width=0.8\linewidth]{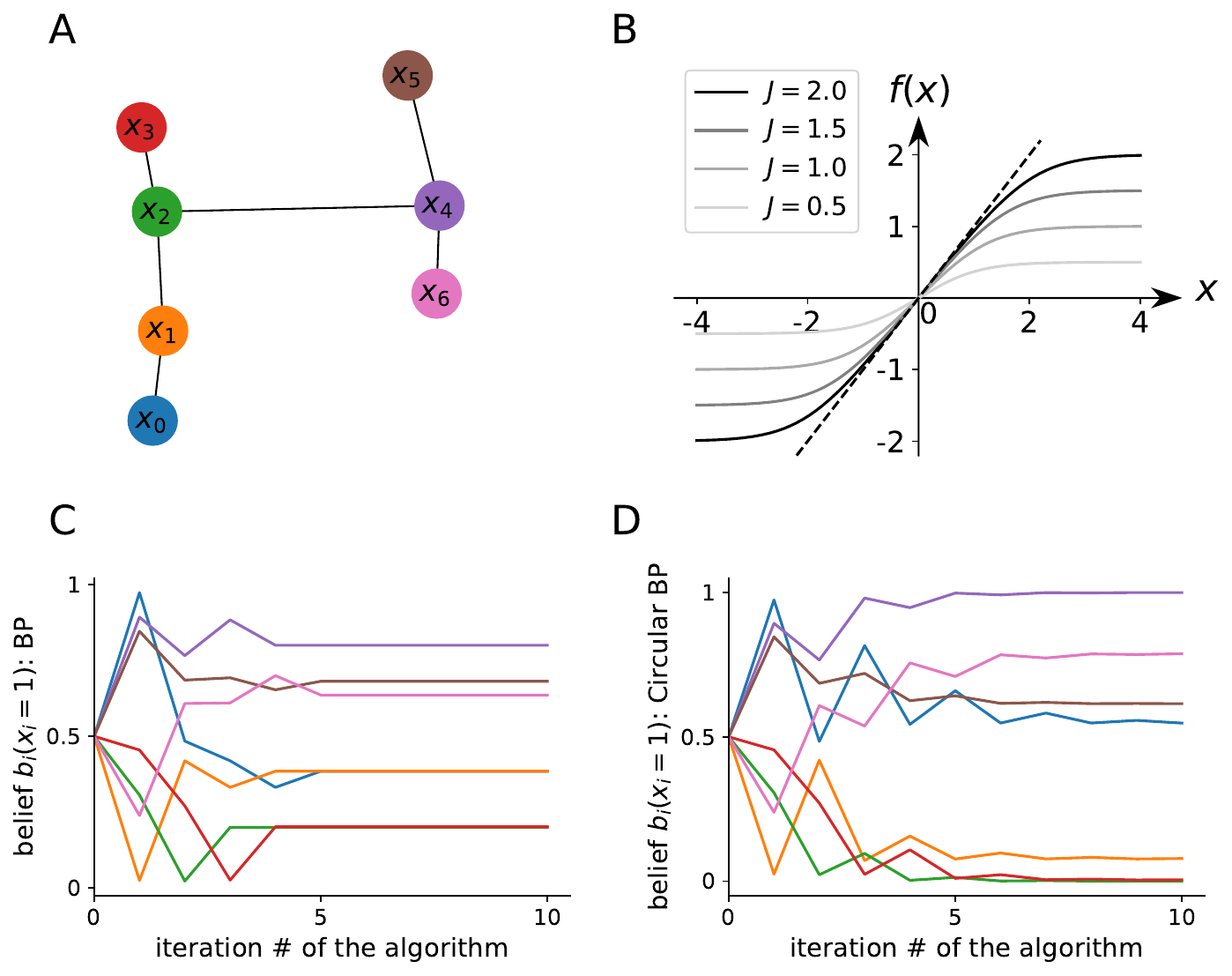} 
  \caption{\textbf{Running the Belief Propagation (BP) algorithm and its variant, the Circular Belief Propagation (CBP) algorithm}.
  \textbf{(A)} Example of acyclic graph taken for the simulation. 
  In the example, the probability distribution corresponds to an Ising model: pairwise potentials $\psi_{ij}(x_i, x_j) \propto \exp(J_{ij} x_i x_j)$ and unitary potentials $\psi_i(x_i) \propto \exp(M_{\text{ext} \to i} x_i)$ where $x_i \in \{-1;+1\}$ (binary case). $J_{ij}$ is generated randomly ($\sim \mathcal{N}(0,3)$), as well as $M_{\text{ext} \to i}$ ($\sim \mathcal{N}(0,2)$).
  \textbf{(B)} The update function $f$ for both BP and CBP 
  is a parametric sigmoidal function close to the hyperbolic tangent. The parameter $J_{ij}$ represents the level of trust between variables $x_i$ and $x_j$.
  \textbf{(C)} Belief Propagation is a message-passing algorithm which consists of running the update equation \eqref{eq:BP-message} 
  until convergence of the messages. The approximate marginals, or beliefs, are defined by Equation \eqref{eq:BP-belief}. 
  Here the beliefs found by BP are exact as the graph has no cycles.
  \textbf{(D)} The Circular Belief Propagation algorithm is a parametric generalization of the Belief Propagation algorithm with parameter $\alpha_{ij}$ assigned to each edge $(i,j)$ of the graph. It is identical to BP for $\bm{\alpha} = \bm{1}$. In the simulation, $\bm{\alpha}$ is taken uniformly over the edges, equal to $0.5$.
  }
  \label{fig:BP-and-CBP}
\end{figure}

\section{Supervised learning of Circular BP - Additional analyses}\label{subsec:additional-analyses}

In this section, we show additional analyses related to supervised learning of CBP parameters.

The supervised learning method for learning parameters of Circular BP $(\bm{\alpha}, \bm{\kappa}, \bm{\beta}, \bm{\gamma})$ is described in the main text. Additionally, many other models than CBP are fitted, using the same learning procedure. This includes special cases of CBP (with one or more of the parameters being fixed) but different classes of algorithms as well, like Fractional BP, the classical tanh model, etc. Each model is initialized when possible with parameters which guarantee convergence of the algorithm (otherwise the learning procedure sometimes does not find a region of convergence, most of the time because there is none but not always); see Appendix \ref{sec:conv-results-appendix} for results on the convergence of CBP depending on its parameters. Convergence can be ensured through the combination of low and positive $\bm{\alpha}$ and $\bm{\kappa}$ as for CBP, but also with low enough $\bm{\beta}$ and lastly, with low enough $\bm{\kappa}$ (even if there is no theoretical result about the spectral radius of matrix $A$ being lower than one if $\bm{\kappa}$ is sufficiently low, it is the case in practice with the generated graphs and thus guarantees convergence).

\subsection{Influence of the interaction strength}
In the numerical experiments presented in the main text, $J_{ij} \sim \mathcal{N}(0, 1)$.
Figure \ref{fig:sigma-Jij-eCBP-and-BP} shows the effect of increasing the strength of interactions weights $\bm{J} = \{J_{ij}\}$ from a given weighted graph. BP clearly shows a performance decrease with increased interaction strength, while CBP remains very good qualitatively.

\begin{figure}[h]
  \centering
  \includegraphics[width=\linewidth]{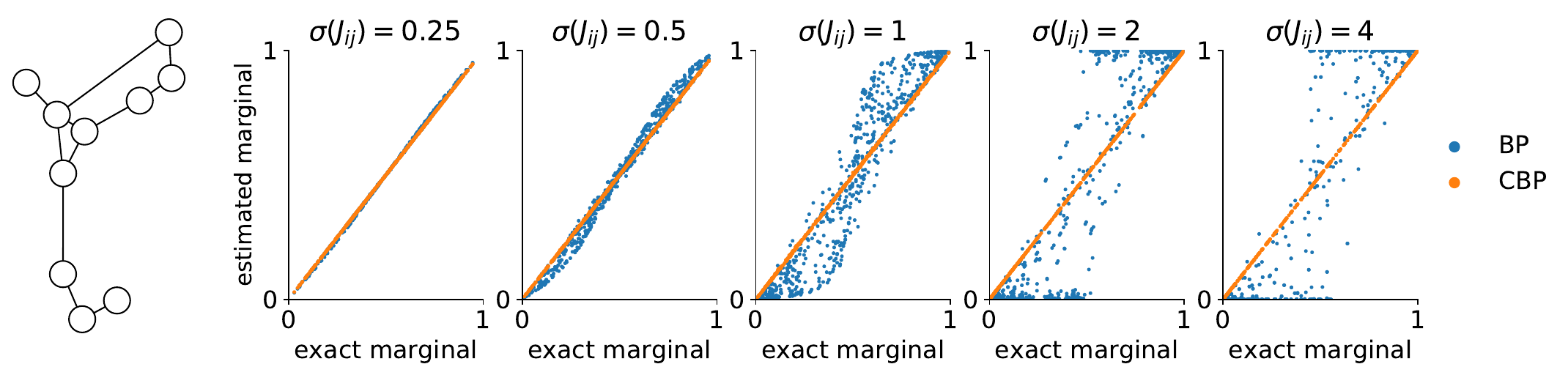}
  \vspace*{-5mm}
  \caption{\textbf{Approximate marginals versus true marginals on the test set, for Belief Propagation and CBP}. We consider a given graph randomly generated (shown on the left) with normally distributed weights $\bm{J} = \{J_{ij}\}$. 
  By increasing the strength of the graph weights, BP gets worse, but CBP still performs inference with very good quality.
  In all cases, CBP outperforms BP. BP shows here overconfidence over the graph, while CBP is on average not overconfident nor underconfident.}
  \label{fig:sigma-Jij-eCBP-and-BP}
\end{figure}

\subsection{Generalization ability}

Here we show that the Circular BP algorithm, using the learning procedure described in the main text, is able to generalize to new data. This shows the goodness of the learning procedure.

\paragraph{Within-set and out-of-set generalization}

A first necessary check for the supervised learning procedure is to make sure that the proposed model is able to learn properly the training data and to generalize to new inputs. 

The training indeed learns to represent the training data and allows the model to generalize well to unseen inputs with identical statistics as in the training set; see Figure \ref{fig:out-of-set-generalization}A. In other words, the Circular BP algorithm is able to generalize to unseen situations, where a ``situation" is a vector $\bm{M_{\text{ext}}}$ of size $n_{\text{nodes}}$. \emph{Unseen} refers to data not included in the training examples. Note that only the external inputs $\bm{M_{\text{ext}}}$ change in the test set: the graph weights $\bm{J}$ are fixed. In other words, parameters $\bm{\alpha}, \bm{\kappa}, \bm{\beta}$, and $\bm{\gamma}$ are learnt given the weights $\bm{J}$, although for any external inputs $\bm{M_{\text{ext}}}$. 

The model also generalizes well to out-of-set inputs, that is, to external inputs $\bm{M_{\text{ext}}}$ with different statistics from the ones used for training; see Figure \ref{fig:out-of-set-generalization}B.



\begin{figure}[h]
  \centering
  \includegraphics[width=\linewidth]{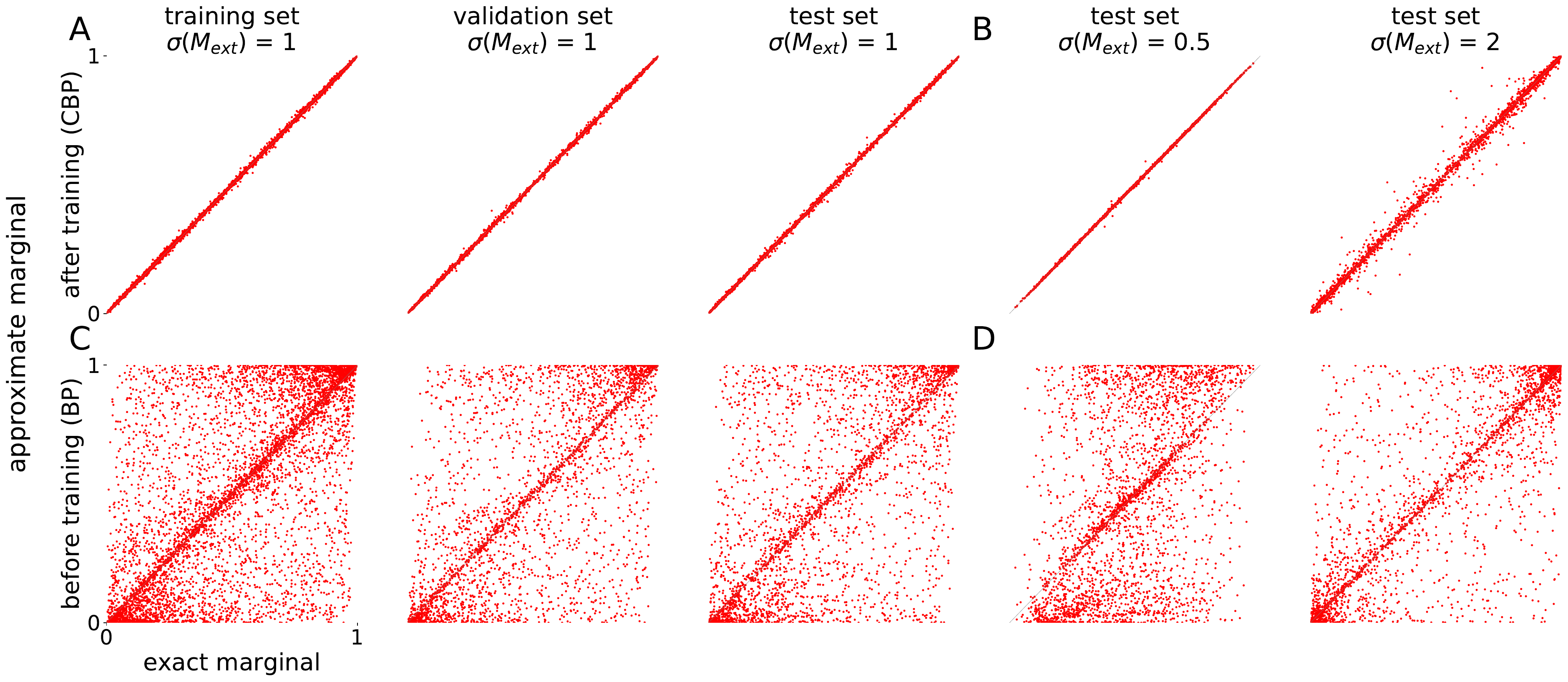}
  \vspace*{-3mm}
  \caption{\textbf{Circular BP generalizes to new data after training}. The topology considered here is the Erdos-Renyi model with $p=0.6$ and 9 nodes, with 30 random weighted graphs generated, each with 100 test examples. \textbf{(A)} The model learns and generalizes well to the test set (within-set generalization). \textbf{(B)} Generalization to examples with different statistics (out-of-set generalization). While $\sigma(\bm{M_{\text{ext}}}) = 1$ on the training set, the model still performs well for lower and higher standard deviations of the input. 
  The performance goes down for increased $\sigma(\bm{M_{\text{ext}}})$ as expected (the system becomes highly non-linear, and the correction brought by CBP is linear) but only slightly. Inferences remain relatively good on such examples for instance compared to BP. \textbf{(C, D)} Same as above but for BP, which shows frustration in (at least) some of the randomly weighted graphs, that is, absence of convergence of the algorithm, in which case the beliefs have very little to do with the correct marginals.}
  \label{fig:out-of-set-generalization}
\end{figure}

\paragraph{Ablation study: Circular BP outperforms its special cases} 

The Circular BP algorithm is expected to outperform all its special cases, including Circular BP from \citet{Jardri2013} (i.e., $(\bm{\kappa}, \bm{\beta}, \bm{\gamma}) = (\bm{1}, \bm{1}, \bm{1})$) and BP (i.e., $(\bm{\alpha}, \bm{\kappa}, \bm{\beta}, \bm{\gamma}) = (\bm{1}, \bm{1}, \bm{1}, \bm{1})$).
This is indeed the case, as shown in Figure \ref{fig:eCBP-vs-special-cases} where models are ordered according to the following score (average over graphs of the log of the MSE averaged over examples) written ``$-\log_{10}(\text{avg MSE})$" in figures:
\begin{equation}\label{eq:def-score-measure}
    \text{score} = - \frac{1}{N_{\text{graphs}}} \sum\limits_{\text{graphs}} \log_{10}\Bigg(\frac{1}{n_{\text{examples}}} \sum_{\text{examples}} \text{MSE}_{\text{graph, example}} \Bigg) 
\end{equation}
where the MSE loss is defined in the main text. 
The performance of all models decreases with the complexity of graphs. Several models like BP perform very poorly for dense graphs because the system becomes frustrated (see Figure \ref{fig:CBP-learning-results}
). It is the true for Circular BP as well, which shows that parameter $\bm{\alpha}$ alone is not enough to make the algorithm converge. On the contrary, other models (all models outperforming or equal to ``CBP nodal") 
show no sign of frustration thanks to additional parameters, as expected by Theorem \ref{theorem:convergence-is-possible}. This allows all these models, including CBP, to keep a rather good level of performance in the approximate inference task even for complete graphs.

\begin{figure}[h]
  \centering
  \includegraphics[width=\linewidth]{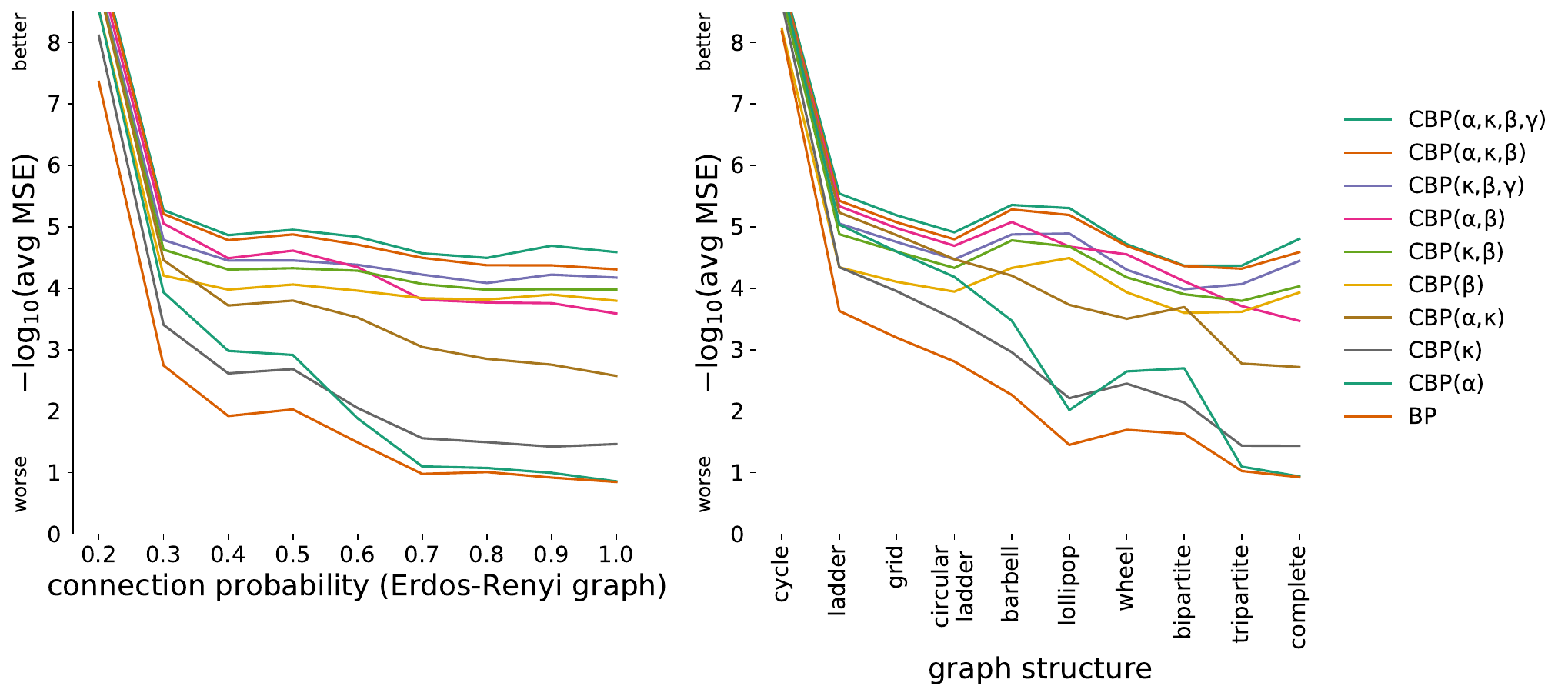}
  \vspace*{-5mm}
  \caption{\textbf{From BP to Circular BP: additional parameters help generalize better}. Circular BP outperforms its special cases: indeed, adding parameters to the algorithms from BP (no parameter) to Circular BP (parameters $\bm{\alpha}, \bm{\kappa}, \bm{\beta}, \bm{\gamma}$) increasingly helps generalize better to the test set. More generally, all models which are special cases of others perform worse comparatively on the test set. This indicates an absence of overfitting: parameters of CBP can be learnt despite their consequent number w.r.t. the amount of training data. 
  Models are ordered w.r.t. their performance on Erdos-Renyi graphs with $p=0.6$. Each point represents the log-MSE score on the test set, averaged over 30 weighted graphs (where the MSE is averaged over examples). Weighted graphs are randomly generated from a graph topology, with normally generated weights $J_{ij} \sim \mathcal{N}(0,1)$. Each weighted graph has 100 test examples.}
  \label{fig:eCBP-vs-special-cases}
\end{figure}



More generally, and quite logically, models which were special cases of others performed comparatively worse on the test set. 
We used the libDAI \citep{libDAI2010} implementation of Generalized BP (``GBP\_MIN" algorithm, i.e., using minimal clusters: one outer region for each maximal factor), Tree-Reweighted BP (with 10000 sampled trees) and Variational message-passing, to compute the associated approximate marginals. 


Among all tested algorithms, it is noteworthy that neither Circular BP from \citet{Jardri2013} nor Fractional BP, for which the only only parameter fitted is $\bm{\alpha}$, provides good results for moderately dense or highly dense graphs. By looking at the marginals produced by these algorithms for Erdos-Renyi graphs (see Figure \ref{fig:CBP-learning-results}
, the reason why they do not perform well is the frustration of the system caused by cycles (strong oscillations between the two extremes $b_i(x_i) \approx 0$ and $b_i(x_i) \approx 1$ without convergence of the algorithm). As observed by \citet{Murphy1999}, convergence of BP implies a good approximation of the correct marginals by the beliefs, and convergence can be forced by using a damped algorithm. Similarly, here, when Circular and Fractional BP converge (for graph with low density), they produce beliefs close to the correct marginals. However, for denser graphs, the algorithms do not converge and in this case beliefs have very little to do with the correct marginals, which Murphy pointed as well. As in \citet{Murphy1999}, introducing damping to CBP and FBP would help the algorithms to converge. However, the damped algorithm would become slower to converge in cases where the undamped algorithm converged. This does not seem like an optimal solution to the frustration problem, as the inference performed by the algorithm should be as fast as possible (including when the inputs vary with time). The question whether the algorithms can converge only by choosing the right $\bm{\alpha}$, without damping, has not been answered yet. It is probable that the $\bm{\alpha}$ parameter is not enough to prevent the frustration behavior from occurring. Alternately, it is possible that the learning procedure does not find the optimal $\bm{\alpha}$. 
This second possibility as yet not been ruled out.

\subsection{Learning CBP on structured graphs}

While the main text presents results on Erdos-Renyi graphs, Figures \ref{fig:marginals-CBP-models-structures} and \ref{fig:eCBP-vs-special-cases-structures} shows that the conclusions remain the same for structured graphs.

\begin{figure}[h]
  \centering
    \includegraphics[width=\linewidth]{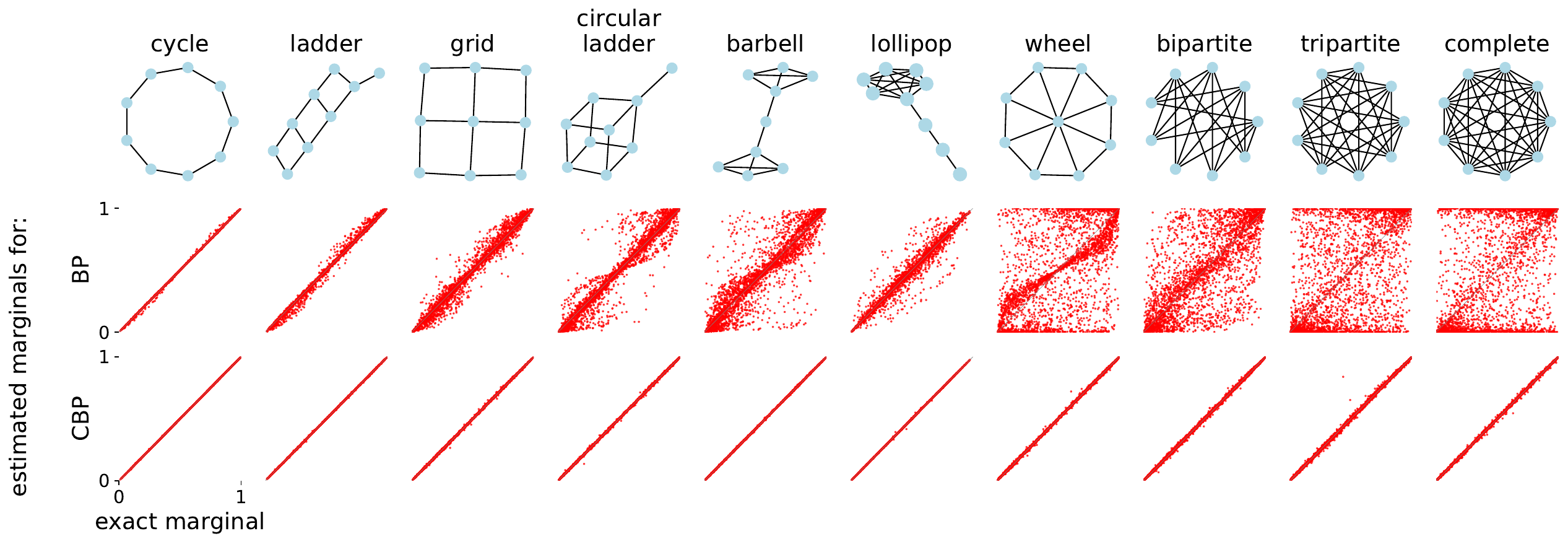}
  \caption{\textbf{Estimated marginals of Circular BP and BP for graphs of various topologies}, on the test set. For each graph topology, 30 randomly weighted graphs are considered with $J_{ij} \sim \mathcal{N}(0, 1)$, each with 200 external input examples in the training set and 100 in the test set. Results are qualitatively the same as in Erdos-Renyi graphs. 
  Graph structures are the same as in \citet{Yoon2018}. 
  }
  \label{fig:marginals-CBP-models-structures}
\end{figure}

\begin{figure}[h]
  \centering
  \includegraphics[width=0.6\linewidth]{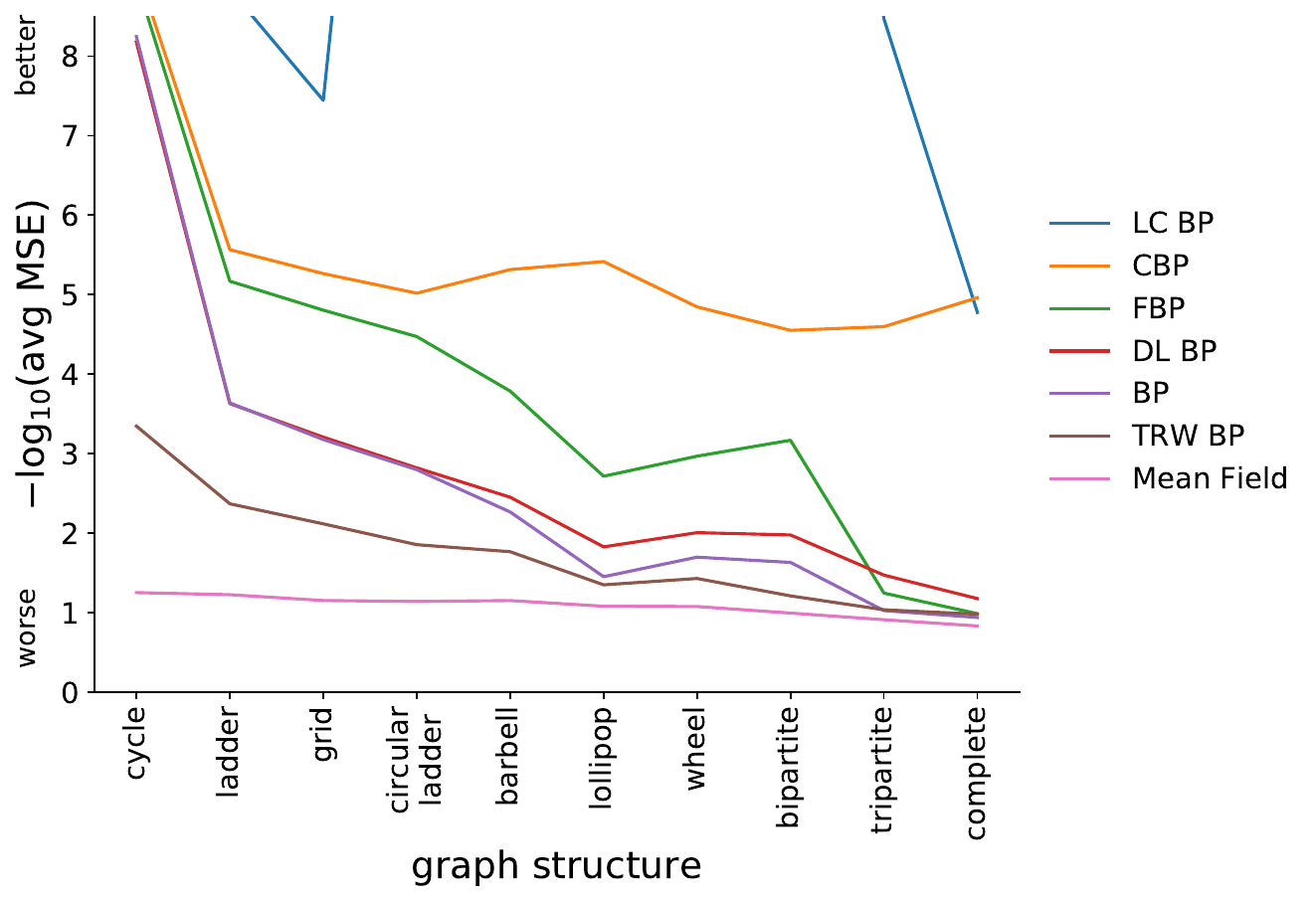}
  \caption{\textbf{Quantitative comparison between various algorithms on different graph structures}. The results are qualitatively the same as for randomly connected graphs.}
  \label{fig:eCBP-vs-special-cases-structures}
\end{figure}

\subsection{Learning CBP with positive couplings}

We take here homogeneous positive couplings $J_ij$ and external fields $M_{\text{ext} \to i}$, a model designed to show off the weakness of BP; see also \citep{Yedidia2001, Roosta2008}. For simplicity, we possibly in d-regular graphs like grid with toroidal boundary conditions). See Figure \ref{fig:marginals-BCP-models-positive-couplings} for the results.

\begin{figure}[h]
  \centering
  \includegraphics[width=0.4\linewidth]{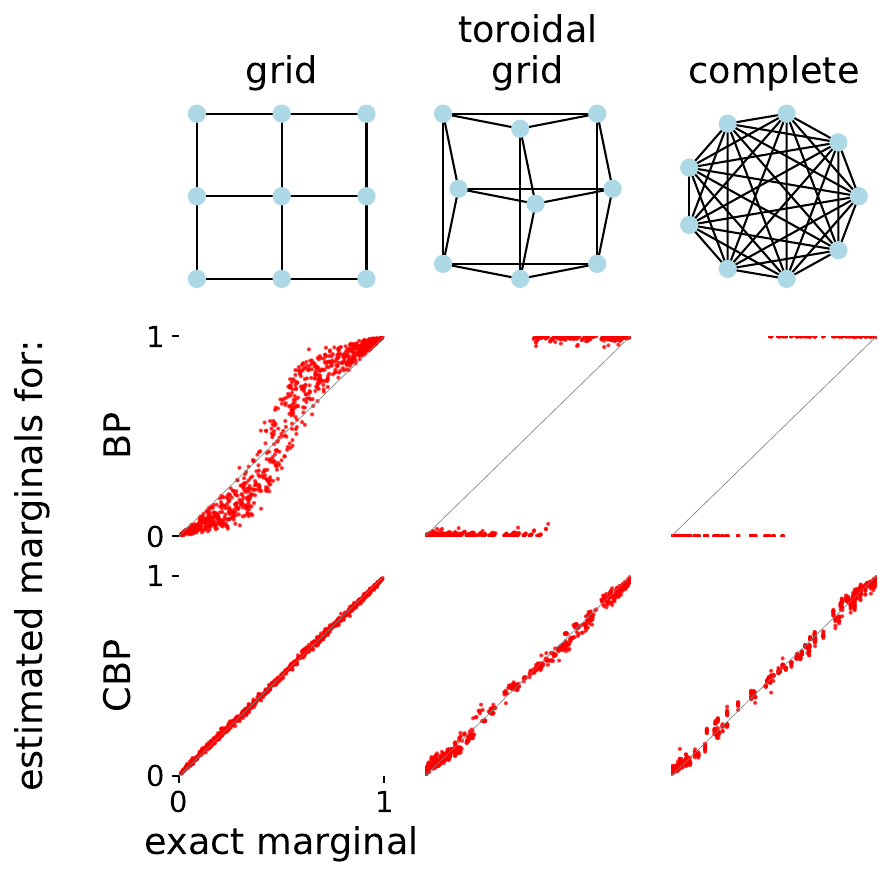}
  \caption{\textbf{Learning CBP with positive couplings}. Here couplings are taken homogeneous with $J_{ij} = \arctanh{(0.6)} \approx 0.7$. 
  }
  \label{fig:marginals-BCP-models-positive-couplings}
\end{figure}

\subsection{Learning CBP in larger graphs}

Figure \ref{fig:marginals-BCP-models-bigger-graphs} shows that CBP (and its supervised optimization procedure on CBP parameters) scales up on graphs of moderate size, on which it is still possible to compute the true marginals. We took the example of $10*10$ grid with normally distributed weights $J_{ij} \sim \mathcal{N}(0, 1)$ (spin-glass). This contrasts with the main text, which only considers graphs with 9 nodes. Note also that the Hopfield network application from the main text involves unsupervised learning in a all-to-all connected graph with 784 nodes.

\begin{figure}[h]
  \centering
  \includegraphics[width=0.22\linewidth]{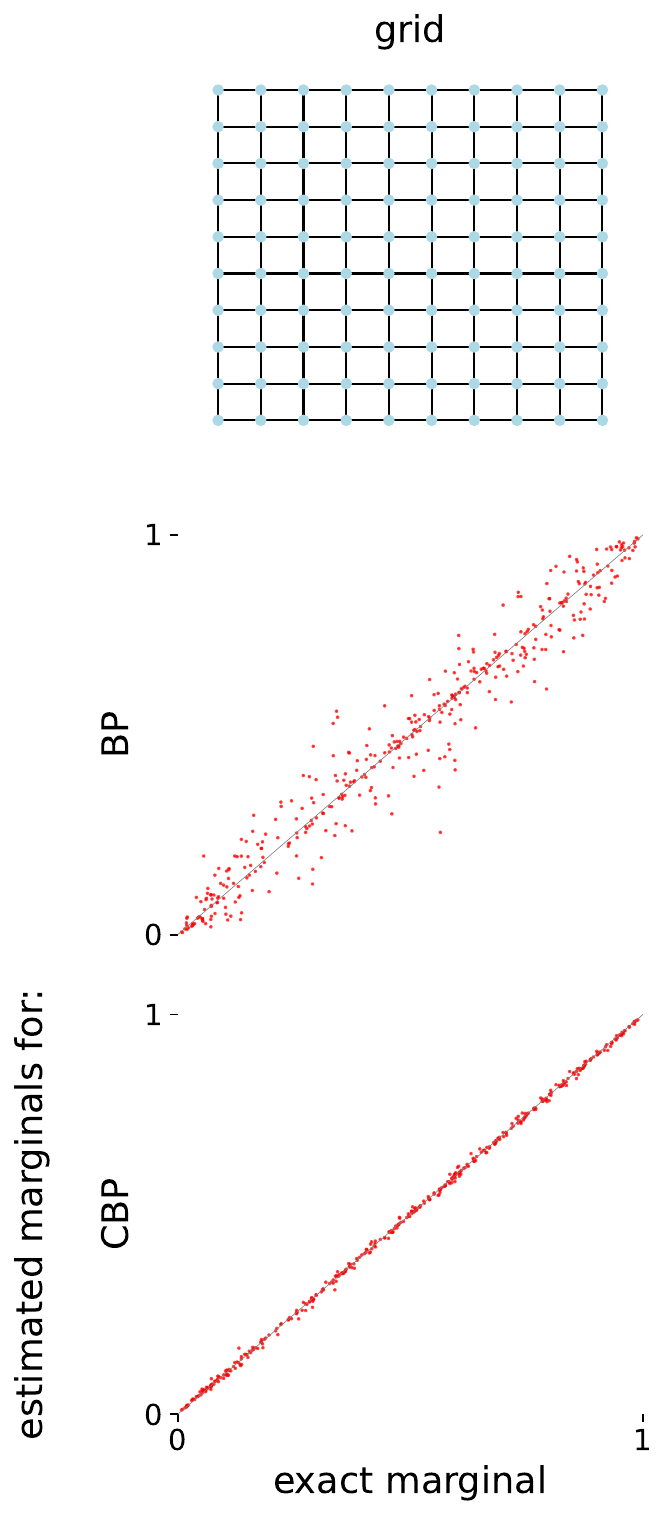}
  \caption{\textbf{Learning CBP in larger graphs}. Example of a grid with 100 nodes, on which supervised learning was applied. 4 different weighted graphs were generated.}
  \label{fig:marginals-BCP-models-bigger-graphs}
\end{figure}

\subsection{Breaking symmetry}

In Circular BP, parameters $\bm{\alpha} = \{\alpha_{ij}\}$ and $\bm{\beta} = \{\beta_{ij}\}$ are symmetric matrices as elements of the matrix depend on the \emph{unoriented} edge (for instance, for $\bm{\alpha}$, $\alpha_{i \to j} = \alpha_{j \to i}$). We tested whether removing this symmetry constraint would improve the inference algorithm. 
In this case, the message update equation of Circular BP becomes:
\begin{equation}\label{eq:eCBP-message-directed-alpha}
    m_{i \to j}^{\text{new}}(x_j) \propto \sum_{x_i} \psi_{ij}(x_i,x_j)^{\beta_{i \to j}} \Big(\psi_i(x_i)^{\gamma_i} \prod\limits_{k \in \mathcal{N}(i) \setminus j} m_{k \to i}(x_i) \: m_{j \to i}(x_i)^{1 - \alpha_{i \to j} / \kappa_i}\Big)^{\kappa_i}
\end{equation}
instead of Equation \eqref{eq:eCBP-message}. Crucially, $\alpha_{i \to j}$ and $\beta_{i \to j}$ are assigned to a \emph{directed} edge $(i, j)$. 
A consequence is that quantities $\alpha_{i \to j}$, $\alpha_{j \to i}$, $\beta_{i \to j}$ and $\beta_{j \to i}$ do not represent inverse overcounting numbers anymore, as these parameters weigh quantities which depend only on the pair $(i,j)$ but not the direction.


As Figure \ref{fig:errors-BP-CBP-FBP-BPnotsymmetric-CBPnotsymmetric-FBPnotsymmetric} shows, removing the symmetry constraint on $\bm{\alpha}$ and $\bm{\beta}$ does not significantly allows for better approximate inference, even though the newly created algorithm (``CBP + non-symmetric" in the figure) is a generalization of the previous one (CBP). The number of parameters to fit increases ($n_{\text{edges}}$ additional parameters for $\bm{\alpha}$ and $n_{\text{edges}}$ additional parameters for $\bm{\beta}$. The figure suggests a slight overfitting, together with the fact that the performance of ``CBP + non-symmetric" is systematically (slightly) better than the one of CBP on the training set (not shown). An explanation among others could be the low amount of training data.

\begin{figure}[h]
  \centering
  \includegraphics[width=0.7\linewidth]{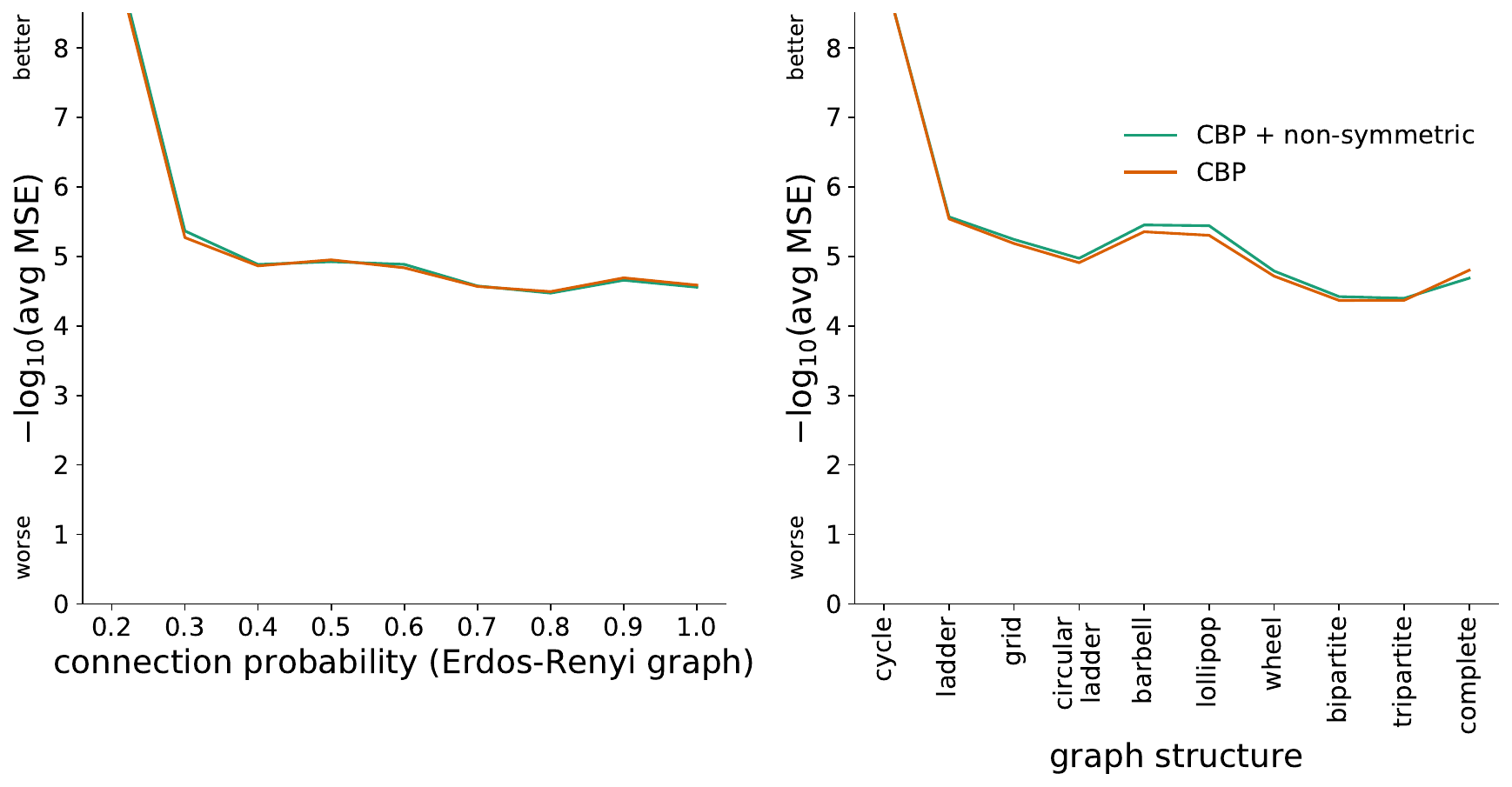}
  \vspace*{-3mm}
  \caption{\textbf{Breaking the symmetry constraint in the parameter matrices $\bm{\alpha}$ and $\bm{\beta}$ does not provide significant improvement to the model}, although it brings $2 n_{\text{edges}}$ new parameters to fit (it even performs worse sometimes, because of slight overfitting). However, it frees some constraints on the parameters.
  For clarity, several points are not shown, as all models perform extremely well for trees and graphs close to trees (here $p=0.2$ and ``cycle").}
  \label{fig:errors-BP-CBP-FBP-BPnotsymmetric-CBPnotsymmetric-FBPnotsymmetric}
\end{figure}

\subsection{Comparison between all algorithms}\label{subsec:comparison-algos}

To complete the figures, which only illustrate specific points from the text by showing algorithms of the same class, we provide in Table \ref{table:comparison-algos} a ranked list containing algorithms of all types. The rank is defined by the performance of each algorithm for Erdos-Renyi graphs with connection probability $p=0.6$. For each connection probability, 30 randomly weighted graphs of 9 nodes were generated, each with 100 examples (vectors $M_{\text{ext}}$) in the test set.
More precisely, the performance is measured identically to previously, i.e., by computing the average over graphs of $-\log_{10}$(avg MSE) between the true marginals and the estimated marginals where ``avg" of ``avg MSE" is taken over examples.

Overall, among fitted models, the best algorithm was Reweighted BP, followed by Circular BP. Both significantly outperformed BP. Loop-Corrected BP is the best algorithm overall, although much more complex than Reweighted BP or Circular BP.


\begin{table}[h!] 
\caption{\textbf{Comparison between various classes of algorithms}. Graphs considered here are Erdos-Renyi graphs ($p=0.6$). The score measure is defined in Equation \eqref{eq:def-score-measure}.}
\label{table:comparison-algos}
\vskip 0.15in
\begin{center}
\begin{small}
 \begin{tabular}{c c c c}  
 Name & Score & \# parameters & Reference \\ 
 \midrule 
 Loop-corrected BP & 6.38 & - & \citet{Mooij2007b} \\ 
 Circular BP & 4.83 & $2n_{\text{edges}} + 2n_{\text{nodes}}$ & this paper \\ 
 Classical tanh model & 2.21 & $2 n_{\text{edges}} + n_{\text{nodes}}$ & \citet{Wilson1972} \\ 
 Fractional BP & 2.07 & $n_{\text{edges}}$ & \citet{Wiegerinck2002} \\ 
 Double-loop BP & 1.92 & - & \citet{Heskes2003} \\ 
 Circular BP with $(\bm{\kappa}, \bm{\beta}, \bm{\gamma}) = (\bm{1}, \bm{1}, \bm{1})$ & 1.88 & $n_{\text{edges}}$ & \citet{Jardri2013} \\ 
  BP & 1.49 & - & \citet{Pearl1988} \\ 
 Tree-Reweighted BP & 1.22 & - & \citet{Wainwright2003}\\ 
 Mean Field / Variational MP & 0.97 & - & \citet{Winn2005} \\ 
\bottomrule
\end{tabular}
\end{small}
\end{center}
\vskip -0.1in
\end{table}

Figure \ref{fig:CBP-learning-results} 
from the main text provides visual comparison for Erdos-Renyi graphs (including the case $p = 0.6$ considered in the above table) for most models listed in Table \ref{table:comparison-algos}.






\section{Unsupervised learning of Circular BP - Additional details}\label{sec:unsupervised-learning-appendix}

In this section, we provide additional considerations related to the unsupervised learning rule proposed in the main text.
Section \ref{subsec:justif-unsupervised-rule} proposes a mathematical intuition behind the learning rule on $\alpha$ and $\kappa$ proposed in the main text.
Section \ref{subsec:learning-beta-gamma} proposes a way of learning $\bm{\beta}$ and $\bm{\gamma}$ in an unsupervised fashion (contrary to the main paper which investigates unsupervised learning of the other two parameters of CBP: $\bm{\alpha}$ and $\bm{\kappa}$).

\subsection{Justification for the unsupervised learning rule on $\alpha$ and $\kappa$}\label{subsec:justif-unsupervised-rule}

\paragraph{Learning rule on $\bm{\alpha}$}
The unsupervised learning rule for $\bm{\alpha}$ minimizes the total amount of information sent in the network $\sum\limits_{(i,j)} M_{i \to j}^2$, or equivalently, minimizes the quantity of so-called prediction errors $\sum\limits_{(i,j)} (B_i - B_i^j)^2$:
\begin{align}
    L &= \sum\limits_{(i,j)} M_{i \to j}^2 \\
    &\approx \sum_{(i,j)} \Big(\arctanh\Big[\tanh(J_{ij}) \tanh(B_i - \alpha_{ij} M_{j \to i})\Big]\Big)^2
\end{align}

We compute the derivative of the loss function $L$ w.r.t. $\bm{\alpha}$ by considering that $\frac{\partial (B_i - \alpha_{ij} M_{j \to i})}{\partial \alpha_{kl}} \approx - M_{j \to i} \delta_{jk} \delta_{il}$, i.e., by neglecting higher-order dependencies. Consequently,
\begin{equation}
    \frac{\partial L}{\partial \alpha_{ij}} \approx \frac{\partial M_{i \to j}^2}{\partial \alpha_{ij}} + \frac{\partial M_{j \to i}^2}{\partial \alpha_{ij}}
\end{equation}
where
\begin{align}
    \frac{\partial M_{i \to j}^2}{\partial \alpha_{ij}} \approx - 2& M_{i \to j} (\arctanh)'\Big(\tanh(J_{ij}) \tanh(B_i - \alpha_{ij} M_{j \to i})\Big) \times \notag\\ &\tanh(J_{ij}) \tanh'(B_i - \alpha_{ij} M_{j \to i}) M_{j \to i}
\end{align}
which has the same sign as $- 2 M_{i \to j} \tanh(J_{ij}) M_{j \to i}$ (symmetrical in $(i,j)$).
We thus propose the following learning rule:
\begin{equation}\label{eq:learning-rule-alpha-1}
    \Delta \alpha_{ij} \propto -\frac{\partial L}{\partial \alpha_{ij}} \propto M_{i \to j} \tanh(J_{ij}) M_{j \to i}
\end{equation}
This learning rule tends to make the correlations between opposite messages disappear, i.e., 
$<M_{i \to j} M_{j \to i}>_{\text{examples}} \approx 0$. In the simulations, we used an alternative learning rule (see Equation \eqref{eq:unsupervised-learning-rule-beta}):
\begin{equation}\label{eq:learning-rule-alpha-2}
    \Delta \alpha_{ij} \propto (B_i - \alpha_{ij} M_{j \to i}) M_{j \to i} + (B_j - \alpha_{ij} M_{i \to j}) M_{i \to j}
\end{equation}
(the right-hand terms of Equations \eqref{eq:learning-rule-alpha-1} and \eqref{eq:learning-rule-alpha-2} have identical signs). 


\paragraph{Learning rule on $\bm{\kappa}$}
The learning rule on $\bm{\kappa}$ aims at avoiding reverberation of external information by cycles. $\kappa_i$ acts as a scaling factor on the log-likelihood ratio (see Equation \eqref{eq:eCBP-belief-log}). The proposed learning rule aims at decorrelating the information truly received by node $i$ from the outside world ($M_{\text{ext} \to i}$), and all the information received by node $i$ except this true external information ($B_i - M_{\text{ext} \to i}$): 
\begin{equation}
    \Delta \kappa_i \propto - M_{\text{ext} \to i} (B_i - M_{\text{ext} \to i}) 
\end{equation} 
Note that a learning rule consisting in from minimizing, as for $\bm{\alpha}$
, the quantity of messages sent throughout the graph, would not work here as 
$\bm{\kappa} \to \bm{0}$ 
would be enough to have all messages go to zero. Finally, the learning rule 
$\Delta \kappa_i \propto - M_{\text{ext} \to i} (B_i - \kappa_i M_{\text{ext} \to i})$ 
would not work either as the sign of the correlation between $M_{\text{ext} \to i}$ and 
$B_i - \kappa_i M_{\text{ext} \to i} \propto M_{j \to i}$ 
does not depend on $\kappa_i$ but instead on the graph topology and weights $\bm{J}$.

\subsection{Unsupervised learning of $\beta$ and $\gamma$}\label{subsec:learning-beta-gamma} 

In section \ref{sec:unsupervised-learning-eCBP} of the main paper, we considered the case where $(\bm{\beta}, \bm{\gamma}) = (\bm{1}, \bm{1})$ in Circular BP and proposed a way of learning parameters $\bm{\alpha}$ and $\bm{\kappa}$ in an unsupervised fashion. Here we ask the question of the learning of the additional parameters $\bm{\beta}$ and $\bm{\gamma}$. 

As a reminder, parameter $\bm{\beta}$ simply appears in the formulation of CBP as a multiplicative factor to the true weights $\bm{J}$ (defined by $\psi_{ij}(x_i,x_j) = \exp(J_{ij} x_i x_j)$). Similarly, parameter $\bm{\gamma}$ simply appears as a multiplicative factor to the true external inputs $\bm{M_{\text{ext}}}$ (defined by $\psi_{i}(x_i) = \exp(M_{\text{ext} \to i} x_i)$). Therefore, fitting $\bm{\beta}$ and $\bm{\gamma}$ is equivalent to fitting $\bm{J}$ and $\bm{M_{\text{ext}}}$, respectively the graph weights and inputs (where $\bm{J}$ and $\bm{M_{\text{ext}}}$ are not related anymore to the factors of the probability distribution).

Parameter $\bm{\beta}$ can be learnt similarly to \citet{Mongillo2008, Jardri2013} which proposes a expectation maximization (EM) algorithm \citep{Dempster1977} and alternatively a Hebbian-like learning rule using stochastic gradient descent, to learn the factors $\{\psi_{ij}\}$.
This approach is based on the fact that messages from BP and its variants allow not only to compute unitary marginal probabilities $\{p_i(x_i)\}$ but also pairwise marginal probabilities $\{p_{ij}(x_i, x_j)\}$. For instance, for CBP, one can assume that:
\begin{align}\label{eq-unsupervised:eFBP-pairwise-belief}
    b_{ij}(x_i,x_j) \propto \psi_{ij}(x_i,x_j)^{\alpha_{ij} \beta_{ij}} \psi_i(x_i)^{\kappa_i \gamma_i} \psi_j(x_j)^{\kappa_j \gamma_j}
\prod\limits_{k \in \mathcal{N}(i) \setminus j} m_{k \to i}(x_i)^{\kappa_i} \notag \\ 
\times \prod\limits_{k \in \mathcal{N}(j) \setminus i} m_{k \to j}(x_j)^{\kappa_j} \: m_{j \to i}(x_i)^{\kappa_i - \alpha_{ij}} m_{i \to j}(x_j)^{\kappa_i - \alpha_{ji}}
\end{align}
\begin{equation}\label{eq:eFBP-pairwise-belief-2}
    \implies b_{ij}(x_i, x_j) \propto \psi_{ij}(x_i, x_j)^{\alpha_{ij} \beta_{ij}} \times \frac{b_i(x_i)}{m_{j \to i}(x_i)^{\alpha_{ij}}} \times \frac{b_j(x_j)}{m_{i \to j}(x_j)^{\alpha_{ij}}}
\end{equation}
which has the same form as the expression of the message update equation for CBP: $m_{i \to j}(x_j)$ depends on $b_i(x_i) / m_{j \to i}(x_i)^{\alpha_{ij}}$ in Equation \eqref{eq:eCBP-message}. In particular, for an Ising model, Equation \eqref{eq:eFBP-pairwise-belief-2} gives:
\begin{equation}\label{eq:eFBP-pairwise-belief-log}
    \log\Big(\frac{b_{ij}(1, 1)}{b_{ij}(1, 0)}\Big) = 2 \alpha_{ij} \beta_{ij} J_{ij} + (B_j - \alpha_{ij} M_{i \to j}) 
\end{equation}
which shows the importance of the quantity $B_j - \alpha_{ij} M_{i \to j}$, together with the message update equation of CBP $M_{j \to i} = f(B_j - \alpha_{ij} M_{i \to j},\, \beta_{ij} J_{ij})$.
Note, however, that Equations \eqref{eq-unsupervised:eFBP-pairwise-belief},
\eqref{eq:eFBP-pairwise-belief-2} and
\eqref{eq:eFBP-pairwise-belief-log} are only approximations for CBP. Indeed, Equation \eqref{eq-unsupervised:eFBP-pairwise-belief} is the expression of the pairwise beliefs for Reweighted BP and not Circular BP (see Equation \eqref{eq-dem:eFBP-pairwise-belief}). Such an equation does not exist for Circular BP, as CBP does not come from any approximation of the Gibbs free energy (which is the starting point to yield such an expression). However, as seen previously, CBP can be defined as an approximation to Reweighted BP with a slightly simpler message update equation. By making the hypothesis that even though CBP has a different update equation for the messages, it eventually computes the (approximate) solution to Reweighted BP, then Equation \eqref{eq-unsupervised:eFBP-pairwise-belief} is (approximately) valid.

\citet{Jardri2013} proposes, in the particular situation where $(\bm{\kappa}, \bm{\beta}, \bm{\gamma}) = (\bm{1}, \bm{1}, \bm{1})$ (i.e., for the original Circular BP), an online update of the factor parameters (here, simply 
$J_{ij}$ or equivalently, $\beta_{ij}$) based on the pairwise beliefs $\{b_{ij}(x_i,x_j)\}$, with a Hebbian-like learning rule.
Learning the external message $M_{\text{ext} \to i}$ (or equivalently, its weight $\gamma_i$) could be done similarly by defining an abstract ``external unit" connected to $x_i$ through weight $\gamma_i$. 

Overall, it might be possible (this still needs to be demonstrated numerically) to learn all parameters of Circular BP in an unsupervised way: ($\bm{\alpha}, \bm{\kappa}, \bm{\beta}, \bm{\gamma}$). The learning on $\bm{\alpha}$ and $\bm{\kappa}$ is anti-Hebbian (or inhibitory Hebbian) and the learning on $\bm{\beta}$ and $\bm{\gamma}$ is Hebbian.

\section{Convergence properties of Circular BP - Full proofs}\label{sec:conv-results-appendix}

In this section, we analyse the convergence properties of CBP by providing the full proofs of the theorems presented in the main text.

As seen in Figures \ref{fig:CBP-learning-results} 
and \ref{fig:marginals-CBP-models-structures}, the convergence of the algorithm is crucial to carry out approximate inference: it can be observed a sharp transition between having beliefs very close to the correct marginals and very far from them. When one of the algorithms does not converge, beliefs have very little to do with the correct marginals (this observation was already made for BP in \citet{Murphy1999}). It is therefore important to control the convergence properties, and if possible, to ensure the convergence of our proposed algorithm.

Importantly, for all the variants of BP cited previously (Fractional BP, Circular BP from \citet{Jardri2013} with $(\bm{\beta}, \bm{\gamma}, \bm{\kappa}) = (\bm{1}, \bm{1}, \bm{1})$, Variational message-passing, etc.), there is no theoretical result about the existence of parameters that guarantee convergence. Absence of convergence of an algorithm means that it gets trapped in a limit cycle (\textit{frustration} phenomenon) or alternatively wanders around chaotically without being able to produce approximate marginals. The typical behavior is a strong oscillation of the beliefs from one extreme ($b_i(x_i=1) \approx 100\%$) to the other ($b_i(x_i=1) \approx 0\%$). 
This is an serious obstacle to the use of such algorithms to perform approximate inference. 

On the contrary, we prove here that in the case of Ising models, there exist parameters for which Circular BP is stable. More precisely, for any set of parameters $(\bm{\gamma}, \bm{\beta})$, one can choose parameters $(\bm{\alpha}, \bm{\kappa})$ such that the algorithm converges (see Theorem \ref{theorem:convergence-is-possible-appendix}).

Having a stable algorithm is important, for two reasons. First, it helps performing reasonably good approximate inference with the algorithm for \emph{any} probability distribution. Indeed, whatever the probability distribution is (i.e., the weights $\bm{J}$ and external inputs $\bm{M_{\text{ext}}}$ are), it is possible to select parameters for which the algorithm will produce approximate marginals. 
Second, being in a stable regime allows to perform computations and ensures that estimations of probabilities at some point in time do not depend on the past state of the system (meaning that the system should have only one fixed point for any probability distribution, and the system should converge to this fixed point). In other words, the system should adapt quickly to new inputs and consider its past as little as possible.

\subsection{Convergence results for Circular BP}
Here we state sufficient conditions for the convergence of Circular BP in an Ising model. 

We start by defining matrix $\bm{A}$ whose coefficients are:
\begin{equation}\label{eq:def-matrix-A-appendix}
    A_{i \to j, k \to l}  = \abs{\kappa_i} \tanh\abs{\beta_{ij} J_{ij}} \delta_{il} \mathbbm{1}_{\mathcal{N}(i)}(k) \abs{1 - \frac{\alpha_{ij}}{\kappa_i}}_{j = k}
\end{equation}
where $\mathbbm{1}_{\mathcal{N}(i)}(k) = 1$ if $k \in \mathcal{N}(i)$, otherwise $= 0$, $\delta_{il} = 1$ if $i = l$, otherwise $= 0$, and $\abs{1 - \frac{\alpha_{ij}}{\kappa_i}}_{j = k} = 1 - \frac{\alpha_{ij}}{\kappa_i}$ if $j = k$, otherwise $1$. 

\begin{theorem}\label{theorem:conv-norm-appendix}
    If for any induced operator norm $\lVert \cdot \rVert$ (sometimes called natural matrix norm), $\lVert A \rVert < 1$, then CBP converges to a unique fixed point and the rate of convergence is at least linear.
\end{theorem}

\begin{proof} 
The proof of Theorem \ref{theorem:conv-norm-appendix} follows closely the one of Lemma 2 of \citet{Mooij2007} for Belief Propagation, that is, the special case $(\bm{\alpha}, \bm{\kappa}, \bm{\beta}, \bm{\gamma}) = (\bm{1}, \bm{1}, \bm{1}, \bm{1})$.
The message update equation for Circular BP, given by Equations \eqref{eq:eCBP-message-log} and
\eqref{eq:eCBP-belief-log}, is:
\begin{equation}
    M_{i \to j}^{\text{new}} = \tanh^{-1}\Bigg[\tanh{(\beta_{ij} J_{ij})} \tanh{\Bigg( \kappa_i \gamma_i M_{\text{ext} \to i} + \kappa_i \sum\limits_{k \in \mathcal{N}(i) \setminus j} M_{k \to i} + (\kappa_i - \alpha_{ij}) M_{j \to i} \Bigg)} \Bigg]
\end{equation}
Let F be the update function for the messages: $\mathbf{M}^{\text{new}} = F(\mathbf{M}^{\text{old}})$ 
(in what follows, we drop the superscript \textit{old} for clarity). 
The derivative of F is:
\begin{align}
    F'(\mathbf{M})_{i \to j, k \to l} & = \frac{\partial M_{i \to j}^{\text{new}}}{M_{k \to l}}\\
    & = \tilde{A}_{i \to j, k \to l} B_{i \to j}(\mathbf{M})\label{eq:derivative-F-eCBP}
\end{align}
where 
\begin{equation}
    \tilde{A}_{i \to j, k \to l}  = \abs{\kappa_i} \tanh\abs{\beta_{ij} J_{ij}} \delta_{il} \mathbbm{1}_{\mathcal{N}(i)}(k) \Big(1 - \frac{\alpha_{ij}}{\kappa_i}\Big)_{j = k}
\end{equation}
\begin{equation}
    B_{i \to j}(\mathbf{M}) = \sgn(\kappa_i \beta_{ij} J_{ij}) \cfrac{1 - \tanh^2\Bigg(\kappa_i \gamma_i M_{\text{ext} \to i} + \kappa_i \mathlarger{\sum}\limits_{k \in \mathcal{N}(i) \setminus j} M_{k \to i} + (\kappa_i - \alpha_{ij}) M_{j \to i} \Bigg)}{1 - \tanh^2\big( M_{i \to j}^\text{new}\big)}
\end{equation}
Note that $\sup\limits_{\mathbf{M}} B_{i \to j}(\mathbf{M}) = 1$, as $\sup\limits_x \cfrac{1 - \tanh^2(x)}{1 - \tanh^2(K) \tanh^2(x)} = 1$.\\
It comes that $\sup\limits_{\mathbf{M}} F'(\mathbf{M})_{i \to j, k \to l} \leq A_{i \to j, k \to l}$
where we defined $A$ as $A = \abs{\tilde{A}}$. 
$A$ does not depend on the external messages $\mathbf{M}_{\text{ext}}$, nor does it depend on parameter $\bm{\gamma}$.\\
If for any norm $\lVert \cdot \rVert$ on matrices, $\lVert A \rVert < 1$, $\sup\limits_{\mathbf{M}} F'(\mathbf{M})_{i \to j, k \to l} < 1$, then function $F$ is a $\lVert \cdot \rVert$ - contraction (see Lemma 2 of \citet{Mooij2007}) and the sequence $M, F(M), F \circ F(M), \dots$, that is, the Circular BP algorithm, converges to a unique fixed point with at least a linear rate. 
\end{proof}
Hence choosing in Theorem \ref{theorem:conv-norm-appendix} the spectral norm (induced by the $l_2$-norm), it comes straightforwardly:
\begin{corollary}
    If the largest singular value of A, $\sigma_{\max}(A) < 1$, then Circular BP converges.
\end{corollary}

\begin{theorem}\label{theorem:conv-spectral-radius-appendix}
    If $ \forall (i,j)$, $\alpha_{ij}/\kappa_i \leq 1$ and the spectral radius of A, $\rho(A) < 1$, then Circular BP converges to the unique fixed point.
\end{theorem}
\begin{proof}
The proof of Theorem \ref{theorem:conv-spectral-radius-appendix} follows closely the one of Corollary 3 of \citet{Mooij2007}.
According to Equation \eqref{eq:derivative-F-eCBP},
\begin{equation}
    F'(\mathbf{M})_{i \to j, k \to l} = \tilde{A}_{i \to j, k \to l} B_{i \to j}(\mathbf{M})
\end{equation}
Note that $A \equiv \abs{\tilde{A}} = \tilde{A}$ as $\tilde{A}$ is non-negative (because of the hypothesis $\alpha_{ij}/\kappa_i \leq 1$).
We directly conclude by using Theorem 2 of \citet{Mooij2007}: $F'(\mathbf{M})$ is the product of $A$, a constant non-negative matrix, with $B(\mathbf{M})$ whose coefficients are bounded by 1 in absolute value, thus the sequence $M, F(M), F \circ F(M), \dots$, meaning, the Circular BP algorithm, converges to a fixed point and this fixed point does not depend on the initial condition $M$. 
\end{proof}
Theorem \ref{theorem:conv-spectral-radius-appendix} has the strong constraint that $\alpha_{ij}/\kappa_i \leq 1$ for all $(i,j)$. However, when this constraint is verified, the spectral radius criterion is sharper than the norm criterion of Theorem \ref{theorem:conv-norm-appendix} because $\rho(A) \leq \lVert A \rVert$ for all induced operator norms.\\ 

Finally, a consequence of Theorem \ref{theorem:conv-spectral-radius-appendix} is the following fundamental result, which distinguishes Circular BP from related approaches like Power EP, Fractional BP and $\alpha$-BP:
\begin{theorem}\label{theorem:convergence-is-possible-appendix}
    For a given weighted graph (defined by its weights $J_{ij}$), it is always possible to find parameters $\bm{\alpha}$ and $\bm{\kappa}$ such that Circular BP converges for any external input $\bm{M_{\text{ext}}}$ and any choice of parameters $(\bm{\gamma}, \bm{\beta})$.
\end{theorem}
\begin{proof}
    Let us take $\alpha_{ij} = \kappa_i \equiv p \in \mathbb{R}_{+}$. In this case, $A_{i \to j, k \to l}  = p \tanh\abs{\beta_{ij} J_{ij}} \delta_{il} \mathbbm{1}_{\mathcal{N}(i) \setminus j}(k)$. When $p \to 0$, all coefficients of A go to zero. 
    The spectral radius is a continuous application, and the null matrix has a spectral radius of zero, thus the spectral radius of A goes to zero when $p \to 0$.
    We conclude by using Theorem \ref{theorem:conv-spectral-radius-appendix} as $\alpha_{ij} / \kappa_i = 1 \leq 1$: there exists $p^\star$ such that for all $p < p^\star$, CBP converges. 
\end{proof}
Notably, proof of Theorem \ref{theorem:convergence-is-possible-appendix} shows that choosing $\bm{\alpha}$ and $\bm{\kappa}$ uniformly, equal and large enough guarantees the convergence of CBP.

We use this result to initialize parameters in the supervised fitting procedure (gradient-descent based). $\bm{\alpha}$ and $\bm{\kappa}$ are first set at the BP value of $p = 1$, and we decrease $p$ until the spectral radius of matrix $\bm{A}$ goes below 1, which ensures that Circular BP converges according to the theory. In practice, we decrease $p$ by incrementing $1/p$ by steps of 1. 
Note that we are using a sufficient condition, thus Circular BP could converge for higher values of $p$ than the one chosen. 

\subsection{Extension of the convergence results to other related algorithms}\label{subsec:extension-conv-results}

As shown above, Circular BP (with parameters $(\bm{\alpha}, \bm{\kappa}, \bm{\beta}, \bm{\gamma})$) converges to the unique fixed point, whatever the probability distribution and parameters $(\bm{\beta}, \bm{\gamma})$ are, given the right choice of parameters $(\bm{\alpha}, \bm{\kappa})$.

In particular, for $(\bm{\beta}, \bm{\gamma}) = (\bm{1}, \bm{1})$, CBP with parameters $(\bm{\alpha}, \bm{\kappa})$ has the same convergence properties as the general CBP. In fact, it is the combination of these two parameters that makes the convergence possible. Without parameter $\bm{\kappa}$ (i.e., with $\bm{\kappa} = \bm{1}$), it is not possible to guarantee that there exists $\bm{\alpha}$ such that the algorithm converges. We go even further by stating the following conjecture: there exist weighted graphs (typically, with strong enough weights) for which, for any choice of parameter $\bm{\alpha}$, Circular BP from \citet{Jardri2013} (and similarly, Fractional BP, both having $\bm{\kappa} = \bm{1}$) does not converge. 
As a reminder, parameter $\kappa_i$ scales the belief $B_i$. It is rather intuitive that it is involved in the stability of the system.

Very similarly, CBP converges to the unique fixed point, whatever the probability distribution and parameters $(\bm{\alpha}, \bm{\kappa}, \bm{\gamma})$ are, given the right choice of parameter $\bm{\beta}$. This can be seen easily: in the demonstration of Theorem \ref{theorem:convergence-is-possible-appendix}, coefficients of $A$ go to zero as well if $\bm{\beta} \to \bm{0}$, thus for $\bm{\beta}$ uniform and sufficiently small, the algorithm converges to the unique fixed point. 


As stated above, the convergence result for Circular BP distinguishes the algorithm from related approaches including Circular BP from \citet{Jardri2013}, Power EP, Fractional BP and $\alpha$-BP. 
Indeed, as seen above, $\bm{\alpha}$ alone is not enough to ensure convergence of these algorithms: an additional parameter $\bm{\kappa}$ would be needed but is not present in these algorithms. However, controlling parameter $\bm{\kappa}$ alone (i.e., taking fixed $\bm{\alpha}$) is not sufficient to guarantee the convergence of CBP but is enough for CBP with fixed $\bm{\alpha} = \bm{0}$. Note that not being able to guarantee the convergence of an algorithm does not mean that the algorithm does not converge. Simulations seem to indicate that for strong enough weights and strong enough $\alpha$ values, $\kappa \to 0$ is not enough to make CBP converge. However, such an approach is successful on all the graphs used for which the weights are moderately strong ($J_{ij} \sim \mathcal{N}(0,1)$).

\section{Theoretical background for Circular BP}\label{sec:eCBP-theory} 

As explained in the main text, Circular BP highly relates to Reweighted BP algorithms, which are message-passing algorithms trying to minimize generalizations of the Bethe free energy (used to define the Belief Propagation algorithm).
This section provides some theoretical background underlying these Reweighted BP algorithms (e.g., Fractional BP, Power EP, and $\alpha$-BP), based on 
which the Circular BP algorithm can be defined. 
Importantly, we will see that the Circular BP algorithm is not associated to a particular reweighted Bethe free energy, but is very similar to algorithms trying to do so. We insist that the goal here is not to propose theoretical foundations of Circular BP \textit{per se}, but instead to relate Circular BP to algorithms with solid theoretical foundations.

This section considers the special case of \textbf{Markov Random Fields with at most pairwise interactions} 
(and when stated, the even more special case of such distributions over binary variables).
The general case follows closely the demonstration provided here, and is detailed in section \ref{sec:CBP-in-general-factor-graphs}.

\subsection{Reweighted BP algorithms}

\subsubsection{Gibbs free energy approximations for accurate inferences}\label{subsec:inference-problem}

\paragraph{Problem}
First, we start by writing the general variational problem: given the true probability distribution $p(\mathbf{x})$, we would like to find an approximating probability distribution $b(\mathbf{x})$ (\textit{variational distribution}), whose marginals are easier to compute than those of $p(\mathbf{x})$.
To do that in practice, we minimize the ``difference" between $p(\mathbf{x})$ and $b(\mathbf{x})$, together with some hypotheses on $b(\mathbf{x})$ making the computation of its marginals feasible.

The difference chosen is the Kullback–Leibler divergence:
\begin{equation}
    D_{KL}(b \Vert p) = \sum\limits_{\mathbf{x}} b(\mathbf{x}) \log\Bigg(\frac{b(\mathbf{x})}{p(\mathbf{x})}\Bigg)
\end{equation}
Minimizing $D_{KL}(b \Vert p)$ will result in a probability distribution $b(\mathbf{x})$ ``close" to the original distribution $p(\mathbf{x})$.
The true probability distribution $p(\mathbf{x})$ is written $p(\mathbf{x}) = \frac{e^{-E(\mathbf{x})}}{Z}$ where Z is a normalization constant ensuring that $\sum\limits_{\mathbf{x}} p(\mathbf{x}) = 1$.
\begin{align}
    \implies D_{KL}(b \Vert p) &= \sum\limits_{\mathbf{x}} b(\mathbf{x}) \log(b(\mathbf{x})) + \log(Z) + \sum\limits_{\mathbf{x}} b(\mathbf{x}) E(\mathbf{x})\\
     &= -S_b - F + U_b
\end{align}
where $S_b = - \sum\limits_{\mathbf{x}} b(\mathbf{x}) \log(b(\mathbf{x}))$ is the \emph{variational entropy}, $F = - \log(Z)$ is the \emph{Helmholtz free energy}, and $U_b = \sum\limits_{\mathbf{x}} b(\mathbf{x}) E(\mathbf{x})$ is the \emph{variational average energy}.
The \textbf{\emph{Gibbs free energy}} (also called the \emph{variational free energy}) is defined as $G_b = U_b - S_b$, which we want to minimize in order to minimize the KL divergence. 

Given the probability distribution $p(\mathbf{x})$ and the variational distribution $b(\mathbf{x})$, the variational average energy $U_b$ can be decomposed very easily: $p(\mathbf{x}) \propto \prod_{(i,j)} \psi_{ij}(x_i,x_j) \prod_i \psi_i(x_i)$ or equivalently, $E(x) = - \sum_{(i,j)} \log(\psi_{ij}(x_i,x_j)) - \sum_i \log(\psi_i(x_i))$, and therefore,
\begin{align}
    U_b &= - \sum\limits_\mathbf{x} b(\mathbf{x}) \sum\limits_{(i,j)} \psi_{ij}(x_i,x_j) - \sum\limits_x b(\mathbf{x}) \sum\limits_{i} \psi_{i}(x_i) \notag\\
    &= - \sum\limits_{(i,j)} \sum\limits_{x_i,x_j} b_{ij}(x_i,x_j) \psi_{ij}(x_i,x_j) - \sum\limits_{i} \sum\limits_{x_i} b_i(x_i) \psi_{i}(x_i)
\end{align}
However, it is most of the time very tricky to compute the entropy term $S_b= - \sum\limits_{\mathbf{x}} b(\mathbf{x}) \log(b(\mathbf{x}))$: its formula does not simplify in the general case, and therefore involves to sum an exponential number of terms, which is infeasible in practice in graphs with a large number of nodes. 


\paragraph{The Bethe approximation} 

The Bethe approximation consists of estimating 
the entropy $S_b$ of $b(\mathbf{x})$ as if the probabilistic graph 
representing $b(\mathbf{x})$ were a tree (i.e., were cycle-free). Indeed, in the case of trees, the probability distribution $b(\mathbf{x})$ can be written as a product of terms only involving its pairwise marginals $b_{ij}(x_i, x_j) \equiv \sum_{\mathbf{x} \setminus (x_i, x_j)} b(\mathbf{x})$ and its unitary marginals $b_i(x_i) \equiv \sum_{\mathbf{x} \setminus x_i} b(\mathbf{x})$):
\begin{equation}\label{eq:BP-approx-belief}
    b(\mathbf{x}) \approx 
    \prod\limits_{(i,j)}\frac{b_{ij}(x_i,x_j)}{b_i(x_i) b_j(x_j)} \prod\limits_i b_i(x_i)
\end{equation}
This is equivalent to approximating the entropy $S_b$ of $b(\mathbf{x})$ as follows:
\begin{align}
    - S_b & \approx \sum\limits_x b(x) \sum\limits_{(i,j)} \log\Big(\frac{b_{ij}(x_i,x_j)}{b_i(x_i) b_j(x_j)}\Big) + \sum\limits_x b(x) \sum\limits_{i} \log(b_{i}(x_i)) \notag\\
    & \approx \sum\limits_{(i,j)} \sum\limits_{(x_i,x_j)} b_{ij}(x_i,x_j) \log\Big(\frac{b_{ij}(x_i, x_j)}{b_i(x_i) b_j(x_j)}\Big) + \sum\limits_{i} \sum\limits_{x_i} b_{i}(x_i) \log(b_{i}(x_i))
\end{align}
Note that $S_b = - \sum_{(i,j)} D_{KL}(b_{ij}, b_i b_j) + \sum_i S_{b_i}$ (which can also be written $S_b \approx \sum_{(i,j)} S_{b_{ij}} + \sum_i (1 - |\mathcal{N}(i)|) S_{b_{i}}$). 
This leads to the following approximation of the Gibbs free energy ($G_b$) known as the \textbf{\emph{Bethe approximation}}: $G_b \approx G_b^{\text{Bethe}}$ where $G_b^{\text{Bethe}}$ is called the \textbf{\emph{Bethe free energy}} and is defined by:
\begin{align}\label{eq:Bethe-Free-Energy}
    G_b^{\text{Bethe}} =  & \sum\limits_{(i,j)} \sum\limits_{(x_i,x_j)} b_{ij}(x_i,x_j) \log\Big(\frac{b_{ij}(x_i,x_j)}{b_i(x_i) b_j(x_j)}\Big) - \sum\limits_{(i,j)} \sum\limits_{(x_i,x_j)} b_{ij}(x_i,x_j) \log\Big(\psi_{ij}(x_i,x_j)\Big)\notag \\
    &+ \sum\limits_i \sum\limits_{x_i} b_i(x_i) \log\Big(b_i(x_i)\Big) - \sum\limits_i \sum\limits_{x_i} b_i(x_i) \log\Big(\psi_i(x_i)\Big)
\end{align}

\paragraph{Generalizing the Bethe approximation}
To allow for a better approximation of the Gibbs free energy, one can use a parametric approximation of the average energy $U_b$ and the variational entropy $S_b$, eventually leading to a generalization of the Bethe free energy. We will refer to this approximation as the \textbf{reweighted Bethe approximation}.

First, to compute the entropy term $S_b$, the variational distribution $b(x)$ is hypothesized to decompose into:
\begin{equation}\label{eq-annex:eFBP-approx-b-wrt-marginals}
    b(\mathbf{x}) \approx \prod\limits_{(i,j)}\Big(\frac{b_{ij}(x_i,x_j)}{b_i(x_i) b_j(x_j)}\Big)^{\hat{\beta}_{ij}} \prod\limits_i (b_i(x_i))^{\hat{\gamma}_i}
\end{equation}
with additional parameters $(\bm{\hat{\beta}}$ and $\bm{\hat{\gamma}})$ 
with respect to Equation \eqref{eq:BP-approx-belief}.
This 
leads to the following parametric approximation of the variational entropy $S_b$ (entropy of $b(\mathbf{x})$):
\begin{equation}\label{eq:approx-entropy-eFBP}
    -S_b \approx \sum\limits_{(i,j)} \hat{\beta}_{ij} \sum\limits_{(x_i,x_j)} b_{ij}(x_i,x_j) \log\Big(\frac{b_{ij}(x_i, x_j)}{b_i(x_i) b_j(x_j)}\Big) + \sum\limits_{i} \hat{\gamma}_{i} \sum\limits_{x_i} b_{i}(x_i) \log(b_{i}(x_i))
\end{equation}
Note that this is equivalent to approximating the entropy $S_b$ as $S_b \approx -\sum\limits_{(i,j)} \hat{\beta}_{ij} D_{KL}(b_{ij}, b_i b_j) + \sum\limits_i \hat{\gamma}_i S_i$
, which can also be written as a weighted sum of local entropies $S_b \approx \sum\limits_{(i, j)} \hat{\beta}_{ij} S_{b_{ij}} + \sum\limits_i \Big(\hat{\gamma}_i - \sum\limits_{j \in \mathcal{N}(i)} \hat{\beta}_{ij}\Big) S_{b_i}$. 

Second, to compute the average energy term $U_b$, the true probability distribution $p(\mathbf{x})$ is hypothesized to decompose into: 
\begin{equation}\label{eq-annex:eFBP-approx-b-wrt-potentials}
    p(\mathbf{x}) \approx \prod\limits_{(i,j)}\big(\psi_{ij}(x_i,x_j)\big)^{\beta_{ij}} \prod\limits_i \big(\psi_{i}(x_i)\big)^{\gamma_i}
\end{equation}
with additional parameters $(\bm{\beta}$ and $\bm{\gamma})$ with respect to Equation \eqref{eq:factorization-p_x-pairwise}. 
Likewise, Equation \eqref{eq-annex:eFBP-approx-b-wrt-potentials} leads to the following parametric approximation of the average energy $U_b$:\footnote{Note that the correct expression of the average energy corresponds to $(\bm{\beta}, \bm{\gamma}) = (\bm{1}, \bm{1})$; see Equation \eqref{eq-annex:eFBP-approx-b-wrt-potentials}. However, the approximation introduced in the computation of the average energy $U_b$ could compensate for some of the wrongness in the approximation of the entropy $S_b$, eventually leading to a better approximation of the Gibbs free energy $G_b = U_b - S_b$.}
\begin{equation}\label{eq:approx-average-energy-eFBP}
    U_b \approx - \sum\limits_{(i,j)} \beta_{ij} \sum\limits_{(x_i,x_j)} b_{ij}(x_i,x_j) \log\Big(\psi_{ij}(x_i,x_j)\Big) - \sum\limits_i \gamma_i \sum\limits_{x_i}  b_i(x_i) \log\Big(\psi_i(x_i)\Big)
\end{equation}

Put together, these approximations of the variational entropy $S_b$ and average energy $U_b$ lead to the following (parametric) approximation of the Gibbs free energy $G_b = U_b - S_b$ which we will refer to as the \textbf{reweighted Bethe free energy}. $G_b \approx G_b^{\text{approx}}$, with
\begin{align}\label{eq:modif-Bethe-Free-Energy-eFBP-pairwise}
    G_b^{\text{approx}} =  & \sum\limits_{(i,j)} \frac{1}{\alpha_{ij}} \sum\limits_{(x_i,x_j)} b_{ij}(x_i,x_j) \log\Big(\frac{b_{ij}(x_i,x_j)}{b_i(x_i) b_j(x_j)}\Big) - \sum\limits_{(i,j)} \beta_{ij} \sum\limits_{(x_i,x_j)} b_{ij}(x_i,x_j) \log\Big(\psi_{ij}(x_i,x_j)\Big) \notag \\&+ \sum\limits_i \frac{1}{\kappa_i} \sum\limits_{x_i} b_i(x_i) \log\Big(b_i(x_i)\Big) - \sum\limits_i \gamma_i \sum\limits_{x_i}  b_i(x_i) \log\Big(\psi_i(x_i)\Big)
\end{align}
where $\bm{\alpha} := 1/\hat{\bm{\beta}}$ and $\bm{\kappa} := 1/\hat{\bm{\gamma}}$. Parameters $(\bm{1/\alpha};\, \bm{\beta};\, \bm{1/\kappa};\, \bm{\gamma})$ are called \emph{counting numbers} (see \citet{Yedidia2005}), with \textit{entropic counting numbers} $(\bm{1/\alpha};\, \bm{1/\kappa})$ and \textit{average energy counting numbers} $(\bm{\beta};\, \bm{\gamma})$.
For $(\bm{\alpha}, \bm{\kappa}, \bm{\beta}, \bm{\gamma}) = (\bm{1}, \bm{1}, \bm{1}, \bm{1})$, 
one recovers the Bethe free energy used in BP \citep{Yedidia2003}.

\subsubsection{Message-passing algorithms to minimize Gibbs free energy approximations}\label{sec:from-gibbs-FE-approx-to-messages}

\paragraph{Theoretical idea behind Belief Propagation}

The Belief Propagation algorithm consists of trying to minimize the Bethe free energy $G_b^{\text{Bethe}}$, coming from the Bethe approximation of the Gibbs free energy, with a fixed point equation. The minimization of the Bethe free energy results into a probability distribution $b(\mathbf{x})$. The underlying idea is the following: because the variational distribution $b(\mathbf{x})$ and the initial distribution $p(\mathbf{x})$ are `close' according to the KL-divergence, they are thought to have `close' marginals.

In more detail, as shown below in the paragraph ``Demonstration: from Gibbs approximation to Reweighted Bp", we minimize the Bethe free energy $G_b^{\text{Bethe}}$ under the constraints $\sum_{x_i} b_i(x_i) = 1$ and $\sum_{x_i} b_{ij}(x_i,x_j) = b_j(x_j)$ (necessary conditions for $b(\mathbf{x})$ to be a probability distribution). This constrained minimization leads to an equation describing the stationary points of the Bethe free energy, involving quantities called ``messages" ($m_{i \to j}$) which are related to the Lagrange multipliers of the constrained optimization problem. This (fixed-point) equation is used directly to define the Belief Propagation algorithm; see Equations \eqref{eq:BP-message} and \eqref{eq:BP-belief}, respectively message update equation and expression of the belief (approximate marginal probability).

BP is closely linked to the Bethe free energy $G_b^{\text{Bethe}}$: fixed points of Loopy BP are stationary points of the Bethe free energy \citep{Yedidia2001}, and stable fixed points of Loopy BP correspond to minima of the Bethe free energy \citep{Heskes2002}.

This leads to several theoretical considerations. First, the fixed points of BP correspond to stationary points of the Bethe free energy $G_b^{\text{Bethe}}$ \citep{Yedidia2003}: the message update equation in BP is not a gradient descent procedure on the Bethe free energy, but instead a fixed-point equation. In fact, it was later shown that stable fixed points of BP are 
minima of the Bethe free energy \citep{Heskes2003}. In the case where the system has one fixed point and this fixed point is stable, the BP procedure is guaranteed to converge to the unique fixed point, that is, to the global minimum of the Bethe free energy. However, in the general case, BP might converge only to a local optimum of the Bethe Free energy, or not converge at all. 

Note that in the case where the initial distribution can be represented by a tree, Equation \eqref{eq:BP-approx-belief} is not an approximation. This explains why BP is exact when applied to acyclic probabilistic graphs.


\paragraph{Idea behind Reweighted Belief Propagation}
As seen above, the Belief Propagation algorithm comes from approximating the entropy term $S_b$ in the Gibbs free energy, by making the hypothesis that the variational distribution $b(\mathbf{x})$ can be written as if the associated probabilistic graph were a tree, i.e., were cycle-free. Because of that, Belief Propagation performs exactly when applied to acyclic graphs, but performs poorly in many other cases (where the approximation becomes bad).

The idea of the Reweighted Belief Propagation algorithms is to provide a better approximation to the Gibbs free energy, by reweighting terms in the expression of the variational entropy, and possibly, as considered here, by also reweighting terms in the expression of the variational average energy (which is initially exact but its approximation can help compensate for some of the remaining approximation left by approximating only the variational entropy). 

The message update equations of Reweighted BP are simply fixed-point equations of $G_b^{\text{approx}}$, meaning that the beliefs computed by Reweighted BP are stationary points of the reweighted Bethe energy 
$G_b^{\text{approx}}$. 
This demonstration is similar to the one for BP \citep{Yedidia2001, Yedidia2003}, with the additional parameters ($\bm{\alpha}, \bm{\kappa}, \bm{\beta}, \bm{\gamma})$. Just as BP is closely linked to the Bethe free energy $G_b^{\text{Bethe}}$, Reweighted BP is closely linked to its generalization $G_b^{\text{approx}}$.
The idea (or rather, hope) of Reweighted BP is that a closer approximation to the Gibbs free energy will lead to closer marginal probabilities.

\paragraph{Demonstration: from Gibbs approximation to Reweighted BP}
In what follows, we derive the message-passing update equations of the Reweighted BP algorithm (which includes as special cases Fractional BP, Power EP and $\alpha$-BP, among others) based on the reweighted Bethe approximation of the Gibbs free energy $G_b^{\text{approx}}$; see Equation \eqref{eq:modif-Bethe-Free-Energy-eFBP-pairwise}. This of course includes the derivation of the BP algorithm (for which the counting numbers are equal to 1). We remind the reader that we consider the case where the probability distribution is a Markov Random Field with at most pairwise interactions (note that variables $x_1, x_2, \dots, x_n$ are not necessarily binary), and that Appendix \ref{sec:CBP-in-general-factor-graphs} contains a similar demonstration for general factor graphs, in which interactions can be of higher order.



The goal is to minimize the Gibbs free energy approximation $G_b^{\text{approx}}$ of Equation \eqref{eq:modif-Bethe-Free-Energy-eFBP-pairwise} under some constraints (normalization constraint $\sum_{x_i} b_i(x_i) = 1$ 
and marginalization constraint $\sum_{x_i} b_{ij}(x_i,x_j) = b_j(x_j)$). Therefore, we form a Lagrangian to take these constraints into account by adding Lagrange multipliers $(\mu, \lambda)$ to $G_b^{\text{approx}}$, where $\{\mu_i\}$ 
correspond to the normalization constraints and $\{\lambda_{i \to j}(x_j)\}$ correspond to the marginalization constraints. The Lagrangian writes: 
\begin{align}
    \mathcal{L} = G_b^{\text{approx}} &+ \sum_{i} \mu_i \Big(\sum_{x_i} b_i(x_i) - 1\Big) \notag \\
    &+ \sum_{(i,j)} \sum\limits_{x_j} \lambda_{i \to j}(x_j) \Big(\sum\limits_{x_i} b_{ij}(x_i,x_j) - b_j(x_j)\Big) \notag \\ 
    &+ \sum_{(i,j)} \sum\limits_{x_i} \lambda_{j \to i}(x_i) \Big(\sum\limits_{x_j} b_{ij}(x_i,x_j) - b_i(x_i)\Big)
\end{align}
We compute the partial derivatives of the Lagrangian:
\begin{empheq}[left=\empheqlbrace]{align}
  &\frac{\partial \mathcal{L}}{\partial b_i(x_i)} = - \sum\limits_{j \in \mathcal{N}(i)} \frac{1}{\alpha_{ij}} + \frac{1}{\kappa_i} + \frac{1}{\kappa_i} \log(b_i(x_i)) - \gamma_i \log(\psi_i(x_i)) + \mu_i - \sum\limits_{j \in \mathcal{N}(i)} \lambda_{j \to i}(x_i) \notag 
  \\
  &\frac{\partial \mathcal{L}}{\partial b_{ij}(x_i,x_j)} = \frac{1}{\alpha_{ij}} + \frac{1}{\alpha_{ij}} \log\Big(\frac{b_{ij}(x_i,x_j)}{b_i(x_i) b_j(x_j)}\Big) - 
\beta_{ij} \log(\psi_{ij}(x_i,x_j)) + \lambda_{j \to i}(x_i) + \lambda_{i \to j}(x_j) \notag 
\end{empheq}
and cancel these partial derivatives to obtain the following expression for the unitary beliefs:
\begin{equation}
    b_i(x_i) \propto \psi_i(x_i)^{\kappa_i \gamma_i} \prod\limits_{k \in \mathcal{N}(i)} \exp\Big(\kappa_i \lambda_{k \to i}(x_i)\Big)
\end{equation}
and for the pairwise beliefs:
\begin{align}
    b_{ij}(x_i,x_j) \propto & \psi_{ij}(x_i,x_j)^{\alpha_{ij} \beta_{ij}} \psi_i(x_i)^{\kappa_i \gamma_i} \psi_j(x_j)^{\kappa_j \gamma_j}
\prod\limits_{k \in \mathcal{N}(i) \setminus j} \exp\Big(\kappa_i \lambda_{k \to i}(x_i)\Big)  \prod\limits_{k \in \mathcal{N}(j) \setminus i} \exp\Big(\kappa_j \lambda_{k \to j}(x_j)\Big) \notag \\
& \times \exp\Big(\lambda_{j \to i}(x_i)\Big(\kappa_i - \alpha_{ij}\Big)\Big)
\exp\Big(\lambda_{i \to j}(x_j)\Big(\kappa_j - \alpha_{ij}\Big)\Big)
\end{align}
Now defining the messages as a function of the Lagrange multipliers $m_{j \to i}(x_i) \equiv \exp(\lambda_{j \to i}(x_i))$, the approximate marginals $b_i(x_i)$ and approximate pairwise marginals $b_{ij}(x_i, x_j)$ can be written simply as:
\begin{equation}\label{eq-dem:eFBP-belief}
    b_i(x_i) \propto \Bigg(\psi_i(x_i)^{\gamma_i} \prod\limits_{k \in \mathcal{N}(i)} m_{k \to i}(x_i)\Bigg)^{\kappa_i}
\end{equation}
\begin{align}\label{eq-dem:eFBP-pairwise-belief}
    b_{ij}(x_i,x_j) \propto \psi_{ij}(x_i,x_j)^{\alpha_{ij} \beta_{ij}} \psi_i(x_i)^{\kappa_i \gamma_i} \psi_j(x_j)^{\kappa_j \gamma_j}
\prod\limits_{k \in \mathcal{N}(i) \setminus j} m_{k \to i}(x_i)^{\kappa_i} \notag \\ 
\times \prod\limits_{k \in \mathcal{N}(j) \setminus i} m_{k \to j}(x_j)^{\kappa_j} \: m_{j \to i}(x_i)^{\kappa_i - \alpha_{ij}}
 m_{i \to j}(x_j)^{\kappa_i - \alpha_{ij}}
\end{align}
Eventually, thanks to the constraint $\sum\limits_{x_i} b_{ij}(x_i,x_j) = b_j(x_j)$, we obtain the relation between the messages $\bm{m}$:
\begin{equation}\label{eq:damped-eFBP-from-Lagrangian}
    m_{i \to j}(x_j)^{\kappa_j} \propto \Bigg( \sum_{x_i} \psi_{ij}(x_i,x_j)^{\alpha_{ij} \beta_{ij}} \Big(\psi_i(x_i)^{\gamma_i} \prod\limits_{k \in \mathcal{N}(i) \setminus j} m_{k \to i}(x_i) \: m_{j \to i}(x_i)^{1 - \alpha_{ij} / \kappa_i}\Big)^{\kappa_i}\Bigg) \Big(m_{i \to j}(x_j)\Big)^{\kappa_j - \alpha_{ij}}
\end{equation}
\begin{equation}\label{eq:eFBP-from-Lagrangian}
    \Leftrightarrow m_{i \to j}(x_j) \propto \Bigg( \sum_{x_i} \psi_{ij}(x_i,x_j)^{\alpha_{ij} \beta_{ij}} \Big(\psi_i(x_i)^{\gamma_{i}} \prod\limits_{k \in \mathcal{N}(i) \setminus j} m_{k \to i}(x_i) \: m_{j \to i}(x_i)^{1 - \alpha_{ij} / \kappa_i}\Big)^{\kappa_i}\Bigg)^{1/\alpha_{ij}}
\end{equation}
The Reweighted BP algorithm consists of running iteratively the fixed-point equation \eqref{eq:eFBP-from-Lagrangian}:
\begin{equation}\label{eq-dem:eFBP-message}
    m_{i \to j}^{\text{new}}(x_j) \propto \Bigg( \sum_{x_i} \psi_{ij}(x_i,x_j)^{\alpha_{ij} \beta_{ij}} \Big(\psi_i(x_i)^{\gamma_{i}} \prod\limits_{k \in \mathcal{N}(i) \setminus j} m_{k \to i}^{\text{old}}(x_i) \: m_{j \to i}^{\text{old}}(x_i)^{1 - \alpha_{ij} / \kappa_i}\Big)^{\kappa_i}\Bigg)^{1/\alpha_{ij}}
\end{equation}
Note that one could also use directly Equation \eqref{eq:damped-eFBP-from-Lagrangian} instead of \eqref{eq:eFBP-from-Lagrangian} to define the Reweighted BP algorithm:
\begin{equation}\label{eq:damped-eFBP-message}
    m_{i \to j}^{\text{new}}(x_j) \propto \Bigg( \sum_{x_i} \psi_{ij}(x_i,x_j)^{\alpha_{ij} \beta_{ij}} \Big(\psi_i(x_i)^{\gamma_i} \prod\limits_{k \in \mathcal{N}(i) \setminus j} m_{k \to i}^{\text{old}}(x_i) \: m_{j \to i}^{\text{old}}(x_i)^{1 - \alpha_{ij} / \kappa_i}\Big)^{\kappa_i}\Bigg)^{1/\kappa_j} \Big(m_{i \to j}^{\text{old}}(x_j)\Big)^{1 - \alpha_{ij}/\kappa_j}
\end{equation}
In fact, Equations \eqref{eq:damped-eFBP-message} and \eqref{eq-dem:eFBP-message} correspond respectively to the damped (with a particular damping value) versus undamped update equation; see section \ref{sec:damping}. There is no absolute better choice: fixed points do not depend on the amount of damping, and damping might provide better convergence properties but might slow down the system in cases where the algorithm converges without damping (see section \ref{sec:damping}). In the simulations, we take no damping, except stated otherwise, to distinguish between the effect of damping and the effect of the algorihtm (e.g., Circular versus Fractional BP).

\subsection{Reweighted BP algorithms}

\paragraph{Definition} 
As shown above, minimizing the reweighted Bethe free energy $G_b^{\text{approx}}$ under the constraints $\sum_{x_i} b_i(x_i) = 1$ and $\sum_{x_i} b_{ij}(x_i,x_j) = b_j(x_j)$ leads to a general \textbf{Reweighted BP} algorithm defined by iterating the following (fixed-point) update equation: 
\begin{equation}\label{eq:eFBP-message}
    m_{i \to j}^{\text{new}}(x_j) \propto \Bigg( \sum_{x_i} \psi_{ij}(x_i,x_j)^{\alpha_{ij} \beta_{ij}} \Big(\psi_i(x_i)^{\gamma_i} \prod\limits_{k \in \mathcal{N}(i) \setminus j} m_{k \to i}(x_i) \: m_{j \to i}(x_i)^{1 - \alpha_{ij}/\kappa_i}\Big)^{\kappa_i}\Bigg)^{1/\alpha_{ij}}
\end{equation}
Approximate marginals or beliefs are computed using:
\begin{equation}\label{eq:eFBP-belief}
    b_i(x_i) \propto \Bigg(\psi_i(x_i)^{\gamma_i} \prod\limits_{k \in \mathcal{N}(i)} m_{k \to i}(x_i)\Bigg)^{\kappa_i}
\end{equation}

\paragraph{Special cases of Reweighted BP}
Special cases of the general reweighted Bethe approximation lead to different algorithms. 
We mention here several \emph{Reweighted BP} algorithms: Fractional BP, Power-EP, $\alpha$-BP, Tree-Reweighted BP, and Variational message-passing (mean-field method); see below for more details (mathematical definitions and relationships between algorithms). In these algorithms, parameters $\bm{\beta}$ and $\bm{\gamma}$  take value $\bm{1}$, that is, can be related to Gibbs free energy approximations which only approximate the variational entropy part, not the average energy.
\textbf{Fractional BP} \citep{Wiegerinck2002} comes as a result of approximating the entropy term of the Gibbs free energy by hypothesizing that the variational distribution $b(\mathbf{x})$ writes according to Equation \eqref{eq:eFBP-approx-b-wrt-marginals} in the particular case $\bm{\kappa} = \bm{1}$, that is, with the single parameter $\bm{\alpha}$.
\textbf{Reweighted BP} \citep{Loh2014} is a message-passing algorithm based on approximating the variational entropy as a weighted sum of local entropies. This is equivalent to the Gibbs approximation above with parameters $\bm{\alpha}, \bm{\kappa}$ but as in previous algorithms with $(\bm{\beta}, \bm{\gamma}) = (\bm{1}, \bm{1})$.
\textbf{Tree-reweighted BP} \citep{Wainwright2002, Wainwright2003,  Wainwright2005} aims at computing an upper bound of the log-partition function ($\log(Z)$), by using the same algorithm; the only difference is that $\alpha_{ij}$ symbolizes the inverse appearance probability of edge $(i,j)$ in the set of spanning trees, which imposes $\alpha_{ij} \geq 1$ \citep{Minka2005}.
\textbf{Convex BP} \citep{Hazan2008, Hazan2010}, which admits as special case Tree-reweighted BP
, is a general message-passing algorithm based on convex free energies.
\textbf{Power Expectation Propagation} \citep{Minka2002, Minka2004} and $\bm{\alpha}$\textbf{-BP} \citep{Liu2019, Liu2020} relate to the problem of minimizing the $\alpha$-divergence between $p(\mathbf{x})$ and $b(\mathbf{x})$ instead of the KL-divergence of BP. $\alpha$-BP and Power EP differ in that $\alpha$-BP considers the projection into factorized distributions, while Power EP considered projections into exponential families.
Despite the fact that these two last algorithms do not come directly from the reweighting of the Bethe free energy, their message update equation is the same as previously mentioned algorithms; see also \citep{Minka2005}.
In fact, Power EP, $\alpha$ BP, and Fractional BP are identical up to the amount of damping, where damping consists of taking partial message update steps \citep{Murphy1999}; see below. They all correspond to $(\bm{\kappa}, \bm{\beta}, \bm{\gamma}) = (\bm{1}, \bm{1}, \bm{1})$.
\textbf{Variational message-passing} \citep{Winn2005}, which is the message-passing version of the mean-field method \citep{Peterson1987}, is the same as Fractional BP for $\bm{\alpha} \to \bm{0}$ \citep{Wiegerinck2002, Minka2005}.
Lastly, BP corresponds to $(\bm{\alpha}, \bm{\kappa}, \bm{\beta}, \bm{\gamma}) = (\bm{1}, \bm{1}, \bm{1}, \bm{1})$.

\paragraph{Special case of Reweighted BP: Binary case, pairwise factors}\label{sec:dem-eFBP-binary-case}
With the additional hypothesis that variables $x_i$ are binary, Equations \eqref{eq:eFBP-message} and \eqref{eq:eFBP-belief} can be written simply in the log domain:
\begin{empheq}[left=\empheqlbrace]{align}
  &M_{i \to j}^{\text{new}} = g\Big(B_i - \alpha_{ij} M_{j \to i},\, (\psi_{ij})^{\beta_{ij}} ,\, \alpha_{ij}\Big) \label{eq-dem:eFBP-message-log-pairwise}\\
  &B_i = \kappa_i \Bigg(\sum_{j \in \mathcal{N}(i)} M_{j \to i} + \gamma_i M_{\text{ext} \to i}\Bigg) \label{eq-dem:eFBP-belief-log-pairwise}
\end{empheq}
where quantities $\bm{M}$ and $\bm{B}$ are given by $B_i \equiv \frac{1}{2} \log\big(\frac{b_i(x_i=+1)}{b_i(x_i=-1)}\big)$, $M_{i \to j} \equiv \frac{1}{2} \log\big(\frac{m_{i \to j}(x_j=+1)}{m_{i \to j}(x_j=-1)}\big)$, and $M_{\text{ext} \to i} \equiv \frac{1}{2} \log\big(\frac{\psi_i(x_i = +1)}{\psi_i(x_i = -1)}\big)$.
Function $g$ is a sigmoidal function of its first argument, with parameters $\psi_{ij}$ (interaction between variables $x_i$ and $x_j$), $\alpha_{ij}$ and $\beta_{ij}$ (counting numbers associated to the edge $(i,j)$): 
\begin{equation}\label{eq-g_ij-general-case}
g(x,\, \psi_{ij}^{\beta_{ij}},\, \alpha_{ij}) = \frac{1}{2 \alpha} \log \Bigg(\frac{\big(\psi_{ij}(x_i=+1,x_j=+1)\big)^{\alpha_{ij} \beta_{ij}} e^{2x} + \big(\psi_{ij}(x_i=-1,x_j=+1)\big)^{\alpha_{ij} \beta_{ij}}}{\big(\psi_{ij}(x_i=+1,x_j=-1)\big)^{\alpha_{ij} \beta_{ij}} e^{2x} + \big(\psi_{ij}(x_i=-1,x_j=-1)\big)^{\alpha_{ij} \beta_{ij}}} \Bigg)
\end{equation}

In the even more particular case of the Ising model, for which $\psi_{ij}(x_i, x_j) \propto \exp(J_{ij} x_i x_j)$, then
\begin{equation}\label{eq:g_ij-pairwise-general}
    g(x,\, \psi^{\beta},\, \alpha) = \frac{1}{\alpha} \arctanh\Bigg(\tanh\Big(\alpha \beta J\Big) \tanh\big(x\big)\Bigg)
\end{equation}
\begin{proof}
Use Equation \eqref{eq-g_ij-general-case}, $\log(x) = 2 \arctanh\big(\frac{x-1}{x+1}\big)$, and $\tanh(x) = 2 \sigma(2x) - 1 = \frac{e^{2x}-1}{e^{2x}+1}$.
\end{proof}

\subsection{Mathematical definitions of several Reweighted BP algorithms}

\subsection{Definitions}\label{subsec:def-algos}

\paragraph{BP}
The message update equation associated to the Belief Propagation algorithm is:
\begin{equation}
    m_{i \to j}^{\text{new}}(x_j) \propto \sum_{x_i} \psi_{ij}(x_i,x_j) \psi_i(x_i) \prod\limits_{k \in \mathcal{N}(i) \setminus j} m_{k \to i}(x_i)
\end{equation}

\paragraph{Power-EP}
Power Expectation Propagation \citep{Minka2004} is an extension of Expectation Propagation \citep{Minka2001a} which consists of minimizing the $\alpha$-divergence between the true probability distribution $p(\mathbf{x})$ and the variational distribution $b(\mathbf{x})$, instead of the KL-divergence as in Expectation Propagation and Belief Propagation. The associated message update equation is:
\begin{equation}\label{eq:FBP-message}
    m_{i \to j}^{\text{new}}(x_j) \propto \Big( \sum_{x_i} \psi_{ij}(x_i,x_j)^{\alpha_{ij}} \psi_i(x_i) \prod\limits_{k \in \mathcal{N}(i) \setminus j} m_{k \to i}(x_i) \: m_{j \to i}(x_i)^{1 - \alpha_{ij}}\Big)^{1/\alpha_{ij}}
\end{equation}

\paragraph{Fractional BP and alpha-BP}

The Fractional BP algorithm \citep{Wiegerinck2002} is defined as a result of approximating the entropy of the variational distribution $b(\mathbf{x})$.
As explained above, $\alpha$-BP is derived differently, but has the same update equation as Fractional BP.

The message update equations of Fractional BP and $\alpha$-BP are respectively defined in Equation (17) of \citet{Wiegerinck2002}) and in Equation (17) of \citet{Liu2020}, and take the following form:
\begin{equation}\label{eq:damped-FBP-particular-damping}
    m_{i \to j}^{\text{new}}(x_j) \propto \Big( \sum_{x_i} \psi_{ij}(x_i,x_j)^{\alpha_{ij}} \psi_i(x_i) \prod\limits_{k \in \mathcal{N}(i) \setminus j} m_{k \to i}(x_i) \: m_{j \to i}(x_i)^{1 - \alpha_{ij}}\Big) m_{i \to j}(x_j)^{1 - \alpha_{ij}}
\end{equation}

Note that Fractional BP is derived as this general Reweighted BP algorithm presented above (but with $(\bm{\kappa}, \bm{\gamma}, \bm{\beta}) = (\bm{1}, \bm{1}, \bm{1})$), and is defined with damping, i.e., uses Equation \eqref{eq:damped-eFBP-message} rather than Equation \eqref{eq-dem:eFBP-message}.

\paragraph{Reweighted BP}
The message update equation associated to the general Reweighted BP algorithm is recalled here:
\begin{equation}\label{eq:eFBP-message-def}
    m_{i \to j}^{\text{new}}(x_j) \propto \Bigg( \sum_{x_i} \psi_{ij}(x_i,x_j)^{\alpha_{ij} \beta_{ij}} \Big(\psi_i(x_i)^{\gamma_i} \prod\limits_{k \in \mathcal{N}(i) \setminus j} m_{k \to i}(x_i) \: m_{j \to i}(x_i)^{1 - \alpha_{ij}/\kappa_i}\Big)^{\kappa_i}\Bigg)^{1/\alpha_{ij}}
\end{equation}

\paragraph{Variational message-passing}
Variational message-passing \citep{Winn2005}, or mean-field inference, is the message-passing version of the mean-field method \citep{Peterson1987}. Its message update equation is:
\begin{equation}
     m_{i \to j}^{\text{new}}(x_j) \propto \exp\bigg(\sum_{x_i} \log\big(\psi_{ij}(x_i,x_j)\big) \psi_i(x_i) \prod\limits_{k \in \mathcal{N}(i)} m_{k \to i}(x_i)\bigg)
\end{equation}
Mean-field is known to be overconfident and perform poorly compared to BP \citep{Weiss2001, Mooij2004}.

\subsection{Relations between Reweighted BP algorithms}

\paragraph{Equivalence between damped Fractional BP, Power EP, and $\alpha$-BP}
The Fractional Belief Propagation algorithm closely relates to several approximate inference algorithms that themselves extended BP, such as Power EP \citep{Minka2002, Minka2004}, $\alpha$-BP \citep{Liu2019, Liu2020}, Tree-reweighted BP 
\citep{Wainwright2002, Wainwright2003,  Wainwright2005}, and Variational message-passing \citep{Winn2005}.
Here we show the equivalence between Fractional BP, Power EP, and $\alpha$-BP , if we consider a damped version of the algorithms.
These models are conceptually very similar to each other \citep{Minka2005}. 
In fact, the three algorithms are identical up to the amount of damping, where damping consists of taking partial message update steps \citep{Murphy1999}. Indeed, damped Fractional BP, defined similarly to damped BP (see section \ref{sec:damping}) from the undamped Fractional BP message update equation in Equation \eqref{eq:FBP-message}, is written:
\begin{equation}\label{eq:damped-FBP-message}
    m_{i \to j}^{\text{new}}(x_j) \propto \Bigg( \sum_{x_i} \psi_{ij}(x_i,x_j)^{\alpha_{ij}} \psi_i(x_i) \prod\limits_{k \in \mathcal{N}(i) \setminus j} m_{k \to i}(x_i) \: m_{j \to i}(x_i)^{1 - \alpha_{ij}}\Bigg)^{(1-\epsilon_{i \to j})/\alpha_{ij}} \times m_{i \to j}(x_j)^{\epsilon_{i \to j}}
\end{equation}
and corresponds to Equation \eqref{eq:damped-FBP-particular-damping} (defining Fractional BP in \citet{Wiegerinck2002}) for the particular damping value $\epsilon_{i \to j} = 1 - \alpha_{ij}$.
When it comes to Power EP \citep{Minka2004}, the damping does not appear from the derivations of the algorithm but Power EP is still presented with the possibility of having (any) damping $\bm{\epsilon}$ (see Equations (22) and (23) of \citet{Minka2004}). It is stated in the paper that with the particular value of damping $\epsilon_{i \to j} = 1 - \alpha_{ij}$, the algorithm is ``convenient computationally and tends to have good convergence properties". Interestingly, with this particular value of damping, it is shown that the Power EP (and equivalently, $\alpha$-BP and Fractional BP) proposed implements a minimization of the $\alpha$-divergence (see also the work on $\alpha$-BP \citep{Liu2020} as well as \citet{Minka2005}). 

Again, the damping values only alters the convergence properties of the algorithm but not the fixed points. A consequence is that if the system has a single fixed point and the undamped algorithm converges (respectively damped), then the damped (respectively undamped) algorithm converges and the convergence value is identical between the damped and undamped algorithms.
In the rest of this work, we will consider these algorithms as one: all algorithms implemented in the simulations do not have damping in order not to confound the effect of having different algorithms (for instance, Circular BP versus Fractional BP) with the effects coming from using different damping values.

\paragraph{Tree-reweighted BP as special case of Fractional BP}
Tree-reweighted BP 
\citep{Wainwright2002, Wainwright2003,  Wainwright2005} is a particular case of the algorithms mentioned above. The message update equation of Tree-reweighted BP is identical to Equation \eqref{eq:FBP-message}. The only difference is that in Tree-reweighted BP, $\alpha_{ij}$ symbolizes the inverse appearance probability of edge $(i,j)$ in the set of spanning trees and therefore it imposes the constraint $\alpha_{ij} \geq 1$ \citep{Minka2005}. Note that Fractional BP does not impose any constraints on the value of $\alpha_{ij}$, which even could be negative.

\paragraph{BP and Reweighted BP}
The special case of Belief Propagation is recovered for $(\bm{\alpha}, \bm{\kappa}, \bm{\beta}, \bm{\gamma}) = (\mathbf{1}, \mathbf{1}, \mathbf{1}, \mathbf{1})$.

\paragraph{Fractional BP and Reweighted BP}
The Reweighted BP algorithm generalizes Fractional BP \citep{Wiegerinck2002}, Power EP \citep{Minka2004} and $\alpha$-BP \citep{Liu2019} which all correspond to $(\bm{\kappa}, \bm{\beta}, \bm{\gamma}) = (\bm{1}, \bm{1}, \bm{1})$, and more particularly use the damped message update equation \eqref{eq:damped-eFBP-from-Lagrangian} rather than its undamped version \eqref{eq:eFBP-from-Lagrangian} (see damping in section \ref{sec:damping}). 

Note that the update equation for Power EP (Equation \eqref{eq:FBP-message}) is not identical to the one given in the Fractional BP paper \citep{Wiegerinck2002}, but is the same up the amount of damping in the algorithm. 

\paragraph{Variational message-passing as special case of Fractional BP}
Variational message-passing \citep{Winn2005} is the same as Fractional BP for $\bm{\alpha} = \bm{0}$ \citep{Wiegerinck2002, Minka2005}. More precisely, Fractional BP is not defined for $\bm{\alpha} = \bm{0}$. However, taking $\bm{\alpha} \to \bm{0}$ and $(\bm{\kappa}, \bm{\beta}, \bm{\gamma}) = (\bm{1}, \bm{1}, \bm{1})$ in the Gibbs free energy approximation (Equation \eqref{eq:modif-Bethe-Free-Energy-eFBP-pairwise}) leads to the Variational mean-field free energy $G_b^{MF}$, where
\begin{equation}
    G_b^{MF} = - \sum\limits_{(i,j)}  \sum\limits_{(x_i,x_j)} b_{i}(x_i) b_{j}(x_j) \log\big(\psi_{ij}(x_i,x_j)\big) + \sum\limits_i \sum\limits_{x_i} b_i(x_i) \log\big(b_i(x_i)\big) - \sum\limits_i \sum\limits_{x_i}  b_i(x_i) \log\big(\psi_i(x_i)\big)
\end{equation}
Indeed, if $\bm{\alpha} \to \bm{0}$, then terms in $\frac{b_{ij}(x_i,x_j)}{b_i(x_i) b_j(x_j)}$ dominate in the Gibbs free energy approximation, which forces the variational marginals to factorize: $b_{ij}(x_i,x_j) = b_i(x_i) b_j(x_j)$. 
Eventually, similarly to section \ref{sec:from-gibbs-FE-approx-to-messages}, one obtains the Variational message-passing update equation given in section \ref{subsec:def-algos}. 
Note that variational message-passing is different from Circular BP with $\bm{\alpha} = \bm{0}$. 



\subsection{Circular BP as approximation of Reweighted BP}\label{subsec:eCBP}

\paragraph{Definition of Circular BP}
Circular BP consists of running the following message update equation:
\begin{equation}\label{eq-annex:eCBP-message}
    m_{i \to j}^{\text{new}}(x_j) \propto \sum_{x_i} \psi_{ij}(x_i,x_j)^{\beta_{ij}} \Big(\psi_i(x_i)^{\gamma_i} \prod\limits_{k \in \mathcal{N}(i) \setminus j} m_{k \to i}(x_i) \: m_{j \to i}(x_i)^{1 - \alpha_{ij} / \kappa_i}\Big)^{\kappa_i}
\end{equation}
The approximate marginals or beliefs are computed using:
\begin{equation}\label{eq-annex:eCBP-belief}
    b_i(x_i) \propto \Bigg(\psi_i(x_i)^{\gamma_i} \prod\limits_{k \in \mathcal{N}(i)} m_{k \to i}(x_i)\Bigg)^{\kappa_i}
\end{equation}
This algorithm admits as particular case BP (which corresponds to $(\bm{\alpha}, \bm{\beta}, \bm{\kappa}, \bm{\gamma}) = (\bm{1}, \bm{1}, \bm{1}, \bm{1})$)
as well as the algorithm from \citet{Jardri2013}, also called \emph{Circular BP} (which corresponds to $(\bm{\beta}, \bm{\kappa}, \bm{\gamma}) = (\bm{1}, \bm{1}, \bm{1})$). 

For probability distributions over binary variables, this translates into the following equation in the log-domain:
\begin{empheq}[left=\empheqlbrace]{align}
  &M_{i \to j}^{\text{new}} = f\Big(B_i - \alpha_{ij} M_{j \to i},\, \psi_{ij}^{\beta_{ij}}\Big) \label{eq-annex:eCBP-message-log}
  \\
  &B_i = \kappa_i \Big(\sum\limits_{j \in \mathcal{N}(i)} M_{j \to i} + \gamma_i M_{\text{ext} \to i}\Big) \label{eq-annex:eCBP-belief-log}
\end{empheq}
where $B_i \equiv \frac{1}{2} \log\big(\frac{b_i(x_i=+1)}{b_i(x_i=-1)}\big)$, $M_{i \to j} \equiv \frac{1}{2} \log\big(\frac{m_{i \to j}(x_j=+1)}{m_{i \to j}(x_j=-1)}\big)$, and $M_{\text{ext} \to i} \equiv \frac{1}{2} \log\big(\frac{\psi_i(x_i = +1)}{\psi_i(x_i = -1)}\big)$.\\ $f$ is a sigmoidal function of its first argument given by:
\begin{equation}\label{eq:def-f_ij-general-case}
    f(x,\, \psi_{ij}^{\beta_{ij}}) = \frac{1}{2} \log \Bigg(\frac{\big(\psi_{ij}(x_i=+1,x_j=+1)\big)^{\beta_{ij}} e^{2x} + \big(\psi_{ij}(x_i=-1,x_j=+1)\big)^{\beta_{ij}}}{\big(\psi_{ij}(x_i=+1,x_j=-1)\big)^{\beta_{ij}} e^{2x} + \big(\psi_{ij}(x_i=-1,x_j=-1)\big)^{\beta_{ij}}} \Bigg)
\end{equation}
in the general case
, and 
\begin{equation}\label{eq:def-f_ij_eCBP}
    f(x,\, \psi_{ij}^{\beta_{ij}}) = \arctanh\Big(\tanh(\beta_{ij} J_{ij}) \tanh(x)\Big) 
\end{equation}
in the specific case of Ising models for which $\psi_{ij}(x_i, x_j) \propto \exp(J_{ij} x_i x_j)$. 
\begin{proof}
Use Equation \eqref{eq:def-f_ij-general-case}, $\log(x) = 2 \arctanh\big(\frac{x-1}{x+1}\big)$, and $\tanh(x) = 2 \sigma(2x) - 1 = \frac{e^{2x}-1}{e^{2x}+1}$.
\end{proof}

\paragraph{Link to \citet{Jardri2013}}
\citet{Jardri2013} defines an algorithm which corresponds to a special case of Circular BP: $\bm{\kappa}, \bm{\beta}, \bm{\gamma} = (\bm{1}, \bm{1}, \bm{1})$. However, this algorithm was not designed for approximate inference, but instead to disturb the exactness of BP in graphs without cycles.

\paragraph{Circular BP as approximation of Reweighted BP}

Circular BP is not associated to any choice of ($\bm{\alpha}, \bm{\kappa}, \bm{\beta}, \bm{\gamma}$) in Reweighted BP, thus does not come from any modification of the Bethe free energy (and does not relate either to any other approximation of the Gibbs free energy). Instead, as explained below, CBP can be seen as an approximation of Reweighted BP.
Circular BP can be seen as an approximation to the general Reweighted BP algorithm defined above. This is not obvious from the general message update equations (see Equation \eqref{eq:eFBP-message} for Reweighted BP and Equation \eqref{eq-annex:eCBP-message} for Circular BP), but can be justified mathematically in the case where the distribution $p(\mathbf{x})$ considered is an Ising model, that is, a probability distribution over binary variables, involving at most pairwise interactions, and with particular pairwise interactions ($\psi_{ij}(x_i, x_j) \propto \exp(J_{ij} x_i x_j)$).
In the case of Ising models, Reweighted BP and Circular BP are indeed written identically, except Reweighted BP uses function $g$ and Circular BP uses $f$ (see Equations \eqref{eq-dem:eFBP-message-log-pairwise}-\eqref{eq-dem:eFBP-belief-log-pairwise} versus \eqref{eq-annex:eCBP-message-log}-\eqref{eq-annex:eCBP-belief-log}).
However, functions $f$ and $g$ are approximately equal: 
\begin{equation}
    g(x, \beta J, \alpha) \equiv \frac{1}{\alpha} \arctanh\Bigg(\tanh\Big(\alpha \beta J\Big) \tanh\big(x\big)\Bigg)
    \approx \arctanh\Bigg(\tanh\Big(\beta J\Big) \tanh\big(x\big)\Bigg)
    \equiv f(x,\, \beta J)
\end{equation}
In summary, the Circular BP algorithm strongly relates to the Reweighted BP algorithm; neither of these two algorithms generalizes the other one, but their message update equations are approximately the same.



\begin{figure}
  \centering
  \includegraphics[width=0.8\linewidth]{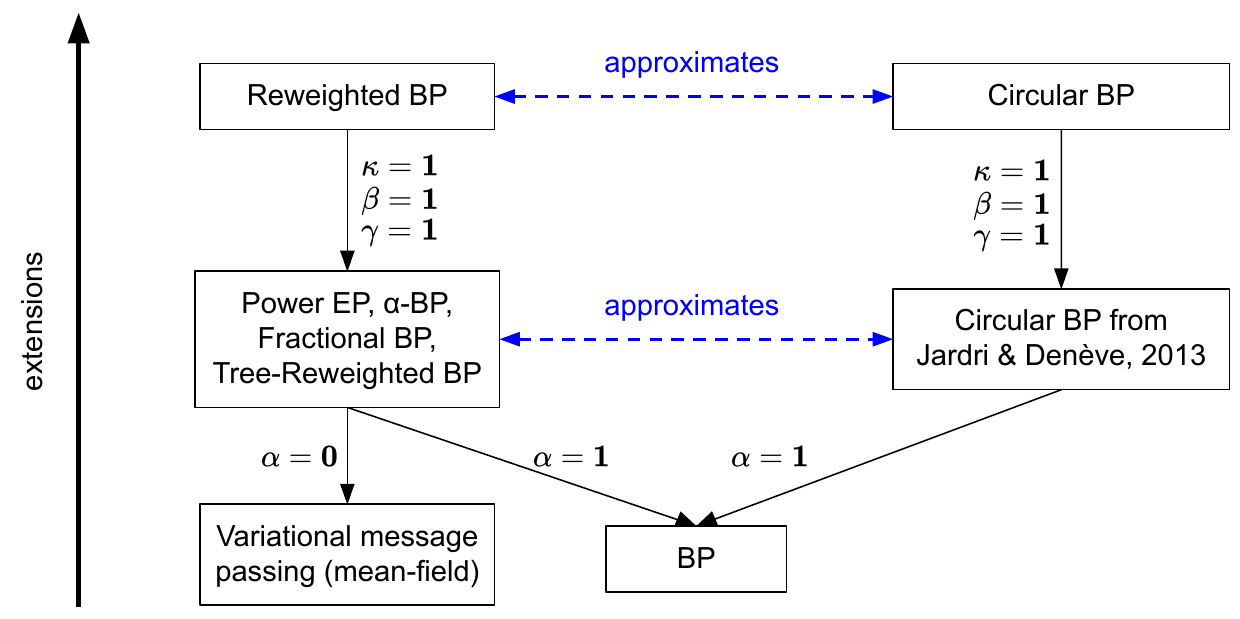} 
  \caption{\textbf{Relations between algorithms}. Algorithms on the first line (Reweighted and Circular BP) have parameters $(\bm{\alpha}, \bm{\kappa}, \bm{\beta}, \bm{\gamma})$. Algorithms on the second line have parameter $\bm{\alpha}$. Algorithms on the third line do not have parameters. 
  }
  \label{fig:relations-algorithms}
\end{figure}

\paragraph{Circular BP and Reweighted BP}
As seen in section \ref{subsec:eCBP}, the update equation of Circular BP algorithm approximates the one of Reweighted BP: indeed, in the Ising model case, update equations are identical up to the function used ($g$ for Reweighted BP and $f$ for Circular BP) and functions are approximately equal ($f(x, J) \approx g(x, J, \alpha)$).

\paragraph{Circular BP from \citet{Jardri2013} and Fractional BP}
Likewise, Circular BP from \citet{Jardri2013} (which corresponds to the case $(\bm{\kappa}, \bm{\beta}, \bm{\gamma}) = (\bm{1}, \bm{1}, \bm{1})$ of the Circular BP algorithm defined in this paper) and Fractional BP (which corresponds to the case $(\bm{\kappa}, \bm{\beta}, \bm{\gamma}) = (\bm{1}, \bm{1}, \bm{1})$ of the Reweighted BP algorithm of Equation \eqref{eq:eFBP-message-def} up to the damping value).

Circular BP does not correspond to any choice of $(\bm{\alpha}, \bm{\kappa}, \bm{\beta}$, or $\bm{\gamma})$. However, because functions $g$ and $f$ (both in the case $\bm{\beta} = \bm{1}$) are similar: 
\begin{equation}
    g(x, J, \alpha) \equiv \frac{1}{\alpha} \arctanh\Bigg(\tanh\Big(\alpha \beta J\Big) \tanh\big(x\big)\Bigg)
    \approx \tanh^{-1}\Bigg(\tanh\Big(\beta J\Big) \tanh\big(x\big)\Bigg) 
    \equiv f(x, J)
\end{equation}
, then the message update equation of Circular BP approximates the update equation corresponding to $(\bm{\kappa}, \bm{\beta}, \bm{\gamma}) = (\bm{1}, \bm{1}, \bm{1})$.

\paragraph{Comparing Fractional BP and Circular BP}

In this paragraph, we show the similarity between Fractional BP (Reweighted BP in the particular case $(\bm{\kappa}, \bm{\beta}, \bm{\gamma}) = (\bm{1}, \bm{1}, \bm{1})$) and Circular BP (also in the particular case $(\bm{\kappa}, \bm{\beta}, \bm{\gamma}) = (\bm{1}, \bm{1}, \bm{1})$). Note that this particular case of Circular BP corresponds to the algorithm from \citet{Jardri2013}.


The message-update equations for the two algorithms (respectively Equation \eqref{eq:FBP-message} and Equation \eqref{eq:eCBP-message}) are very similar. It becomes even more obvious when considering an Ising model:
both write as:
\begin{empheq}[left=\empheqlbrace]{align}
  &M_{i \to j}^{\text{new}} = g(B_i - \alpha_{ij} M_{j \to i},\, J_{ij},\, \alpha_{ij})\label{eq-repeat:CBP-message-message-log}\\
  &B_i = \sum_{j \in \mathcal{N}(i)} M_{j \to i} + M_{\text{ext} \to i} \label{eq-repeat:CBP-beliefs-log-odds}
\end{empheq}
where $g(x, J, \alpha) \equiv \frac{1}{\alpha} f(x, \alpha J)$ for Fractional BP, and $g(x, J, \alpha=1) = f(x,\, J)$ for Circular BP.
Interestingly, 
functions $g$ and $f$ are close to each other:
\begin{equation}
    g(x, J, \alpha) \equiv \frac{1}{\alpha} \arctanh\Bigg(\tanh\Big(\alpha J\Big) \tanh\big(x\big)\Bigg)
    \approx \tanh^{-1}\Bigg(\tanh\Big(J\Big) \tanh\big(x\big)\Bigg) 
    \equiv f(x, J)
\end{equation}
which can be easily justified mathematically for $\alpha$ or $J$ small enough, and can be seen in practice for various $J$ and $\alpha$ in Figure \ref{fig:FBP-vs-CBP}A.
This means that the message update equation of Fractional BP approximates the update equation of Circular BP.
A consequence of the similarity between their update equations is that Circular BP and Fractional BP yield similar results; see Figure \ref{fig:FBP-vs-CBP}C and D.

\begin{figure}[h] 
  \centering
  \includegraphics[width=0.8\linewidth]{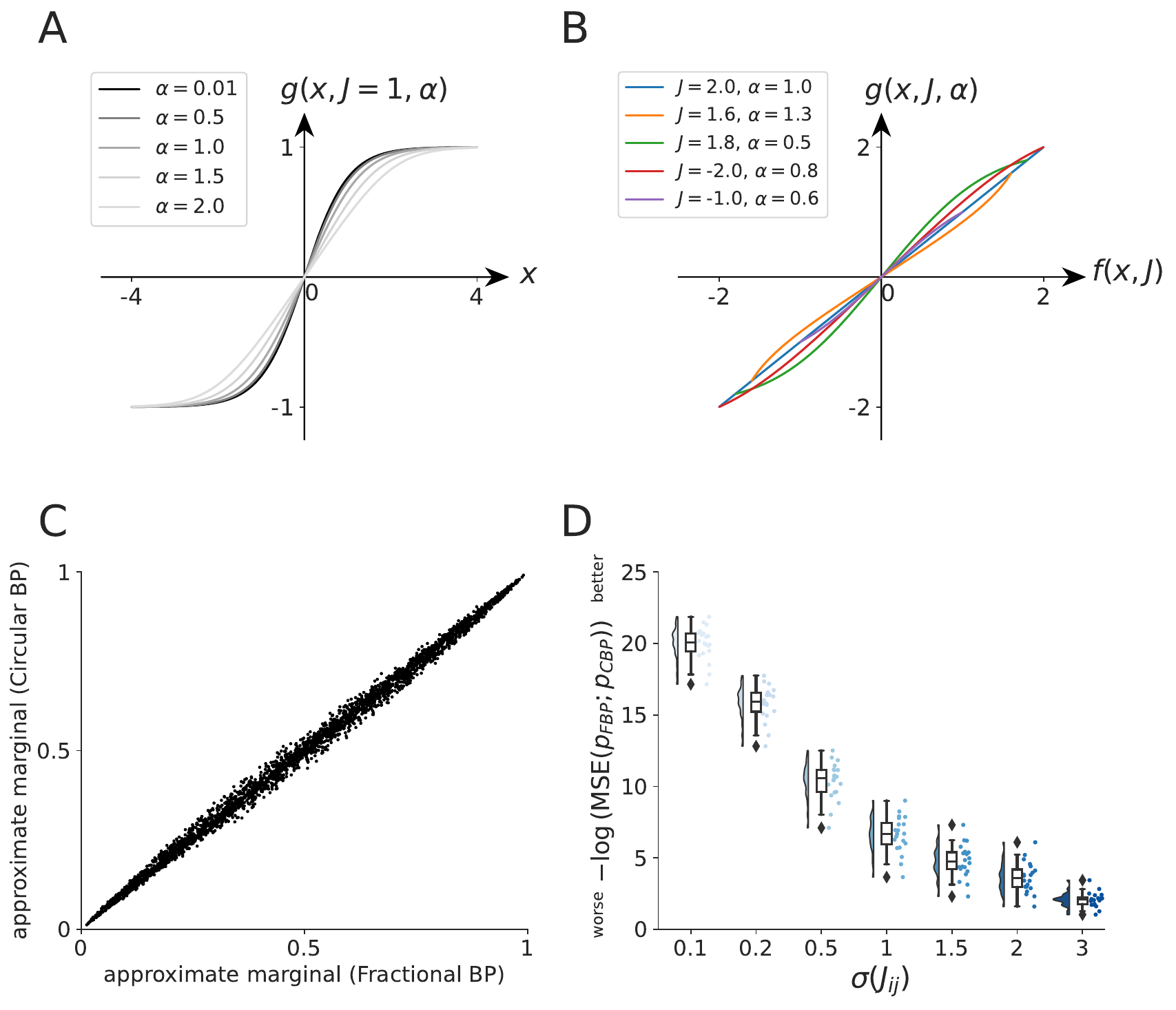} 
  \caption{\textbf{Comparison between Fractional BP and Circular BP}. 
  \textbf{(A)} Visual comparison between the update functions for Fractional BP (function $g$, which depends on the strength of the relationship $J$ and on $\alpha$) and Circular BP (function $f$ which depends only on $J$ and corresponds to $g$ for $\alpha = 1$) for various $\alpha$ and given $J = 1$. All functions have the identical sigmoidal shape. \textbf{(B)} The update functions for Circular BP ($f)$ and for Fractional BP ($g$) are approximately identical, for various $J$ and $\alpha$. \textbf{(C)} Example of marginals produced by the Fractional BP algorithm versus produced by the Circular BP algorithm for a given random graph with randomly generated weights and random $\bm{\alpha}$. \textbf{(D)} Influence of the interaction strength $\sigma(J)$ on the approximation: lower weights lead to a better fit between Circular BP and Fractional BP. One point represents different randomly generated generated $\bm{\alpha}$ for the same graph.}
  \label{fig:FBP-vs-CBP} 
\end{figure}

\include{Reweighted_BP_algorithms.tex}
\include{relations_algorithms.tex}
\section{Damped Circular BP}\label{sec:damping}

Section \ref{sec:damping} gives a definition of damped CBP. 

\paragraph{Notion of damping}

Let a discrete system: 
\begin{equation*}
    x_{t+1} = f(x_t)
\end{equation*}
A common technique to improve convergence of the system is to take partial (or damped) update steps: 
\begin{equation*}\label{eq:damped-system}
    x_{t+1} = (1-\epsilon) f(x_t) + \epsilon x_t
\end{equation*}
with $\epsilon \in [0;1[$.
This procedure, called \emph{damping}, does not modify the fixed points of the system: a fixed point of the original system is a fixed point of the damped system, and reciprocally.

A parallel can be drawn between the damped discrete system and the following continuous system:
\begin{equation}
    \tau \dot{x}(t) = -x(t) + f(x(t)) 
\end{equation}
Indeed, this continuous system can be discretized (Euler approximation) into:
\begin{equation}
    x_{t+\delta t} = (1-\epsilon) f(x_t) + \epsilon x_t
\end{equation}
with $\epsilon = 1- \delta t/\tau$ (i.e., $\tau = \delta t/(1 - \epsilon)$). This corresponds to the discrete damped system.

\paragraph{Damped Circular BP}

We start by rewriting the message update equation for Circular BP (Equation \eqref{eq:eCBP-message}) without damping:
\begin{equation}\label{eq-repeat:eCBP-message}
    m_{i \to j}^{\text{new}}(x_j) \propto \sum_{x_i} \psi_{ij}(x_i,x_j)^{\beta_{ij}} \Big(\psi_i(x_i)^{\gamma_i} \prod\limits_{k \in \mathcal{N}(i) \setminus j} m_{k \to i}(x_i) \: m_{j \to i}(x_i)^{1 - \alpha_{ij}/\kappa_i}\Big)^{\kappa_i}
\end{equation}
and its equivalent in the log-domain for probability distributions with binary variables (Equation \eqref{eq:eCBP-message-log}):
\begin{equation}
M_{i \to j}^{\text{new}} = f(B_i - \alpha_{ij} M_{j \to i},\, \beta_{ij} J_{ij})
\end{equation}
where $M_{i \to j} \equiv \frac{1}{2} \log\big(\frac{m_{i \to j}(x_j=+1)}{m_{i \to j}(x_j=-1)}\big)$, $M_{\text{ext} \to i} \equiv \frac{1}{2} \log\big(\frac{\psi_i(x_i = +1)}{\psi_i(x_i = -1)}\big)$, $B_i \equiv \frac{1}{2} \log\big(\frac{b_i(x_i=+1)}{b_i(x_i=-1)}\big) = \kappa_i \Big(\sum\limits_{j \in \mathcal{N}(i)} M_{j \to i} + \gamma_i M_{\text{ext} \to i}\Big)$, and function $f$ is given in Equation \eqref{eq:def-f_ij-general-case}.

Damping is a technique commonly used while running Belief Propagation (special case $(\bm{\alpha}, \bm{\kappa}, \bm{\gamma}, \bm{\beta}) = (\bm{1}, \bm{1}, \bm{1}, \bm{1})$), and consists of taking partial message update steps in the log-space: $M_{i \to j}^{\text{new}} = f(B_i - M_{j \to i},\, \beta_{ij} J_{ij})$ becomes $M_{i \to j}^{\text{new}} = (1-\epsilon) f(B_i - M_{j \to i},\, \beta_{ij} J_{ij}) + \epsilon M_{i \to j}$. Damping improves the convergence properties of BP. We talk about \emph{damped Belief Propagation} in this case \citep{Murphy1999}. 

Similarly to damped BP, a damped Circular BP algorithm can be defined, by taking partial message update steps in the log-space. The message update equation for the damped algorithm is:
\begin{equation}\label{eq:damped-eCBP-message}
    m_{i \to j}^{\text{new}}(x_j) \propto \Bigg( \sum_{x_i} \psi_{ij}(x_i,x_j)^{\beta_{ij}} \Big(\psi_i(x_i)^{\gamma_i} \prod\limits_{k \in \mathcal{N}(i) \setminus j} m_{k \to i}(x_i) \: m_{j \to i}(x_i)^{1 - \alpha_{ij} /\kappa_i} \Big)^{\kappa_i}\Bigg)^{1-\epsilon_{i \to j}} \times m_{i \to j}(x_j)^{\epsilon_{i \to j}}
\end{equation}
where $\epsilon_{i \to j}$ is the damping factor associated to the oriented edge $i \to j$ ($\bm{\epsilon} = \bm{0}$ means no damping, i.e., standard CBP) and is often taken uniformly over the edges.
It is equivalent in the log-domain to the following equation: 
\begin{equation}\label{eq:damped-eCBP-message-log}
    M_{i \to j}^{\text{new}} = (1-\epsilon_{i \to j}) f(B_i - \alpha_{ij} M_{j \to i},\, \beta_{ij} J_{ij}) + \epsilon_{i \to j} M_{i \to j}
\end{equation}
where
\begin{equation}\label{eq:damped-eCBP-belief-log}
    B_i = \kappa_i \Big(\sum\limits_{j \in \mathcal{N}(i)} M_{j \to i} + \gamma_i M_{\text{ext} \to i}\Big)
\end{equation}
\section{Circular BP in general factor graphs}\label{sec:CBP-in-general-factor-graphs}

In the main paper, we only considered Ising models, which are a particular class of pairwise factors graphs (that is, probability distributions $p(x)$ which could be written as the product of unitary and pairwise potentials: $p(\mathbf{x}) \propto \prod_{(i,j)} \psi_{ij}(x_i,x_j) \prod_i \psi_i(x_i)$), acting on binary variables where potentials take a particular form ($\psi_{ij}(x_i,x_j) = \exp(J_{ij} x_i x_j)$).
In this section, we present the general case:
we formulate the Circular BP algorithm in general factor graphs (instead of pairwise graphs) $p(\mathbf{x}) \propto \prod\limits_c \psi_c (\mathbf{x}_c)$ where variables can be continuous.

Circular BP algorithm can be defined on general factor graphs similarly to the pairwise factor case:
\begin{equation}\label{eq:eCBP-message-general} 
    m_{\psi_c \to x_j}^{\text{new}}(x_j) \propto \sum_{\mathbf{x}_c \setminus x_j} \psi_c(\mathbf{x}_c)^{\beta_c} \Bigg( \prod\limits_{x_i \in \mathcal{N}(\psi_c) \setminus x_j} \psi_i(x_i)^{\gamma_i} \prod\limits_{x_i \in \mathcal{N}(\psi_c) \setminus x_j} \Big[ \prod\limits_{\psi_{d} \in \mathcal{N}(x_i) \setminus \psi_c} m_{\psi_{d} \to x_i}(x_i) \times m_{\psi_c \to x_i}(x_i)^{1 - \alpha_c / \kappa_i}\Big]\Bigg)^{\kappa_i}
\end{equation}
and produces the following approximate marginals:
\begin{equation}
    b_i(x_i) \propto \psi_i(x_i)^{\gamma_i \kappa_i} \prod\limits_{\psi_c \in \mathcal{N}(x_i)} m_{\psi_c \to x_i}(x_i)^{\kappa_i} 
\end{equation}

Note that Equation \eqref{eq:eCBP-message-general} can be rewritten as a system of equations involving both factor-to-variable messages and variable-to-factor messages, instead of factor-to-variables messages alone:
\begin{empheq}[left=\empheqlbrace]{align}
  & m_{\psi_c \to x_i}^{\text{new}}(x_i) \propto \sum\limits_{\mathbf{x}_c \setminus x_i} \psi_c(\mathbf{x}_c)^{\beta_c} \prod\limits_{x_j \in \mathcal{N}(\psi_c) \setminus x_i} m_{x_j \to \psi_c}(x_j) \label{eq:eCBP-message-factor-to-variable-general}
  \\
  & m_{x_j \to \psi_d}^{\text{new}}(x_j) \propto \psi_{j}(x_j)^{\gamma_j \kappa_j} \Bigg(\prod\limits_{\psi_c \in \mathcal{N}(x_j) \setminus \psi_d} m_{\psi_c \to x_j}(x_j)\Bigg)^{\kappa_j} m_{\psi_d \to x_j}(x_j)^{\kappa_j - \alpha_{d}} \label{eq:eCBP-message-variable-to-factor-general}
\end{empheq}

\subsection{Inference in graphical models}

A general Markov random field $p(\mathbf{x})$ can be written:
\begin{equation}
    p(\mathbf{x}) \propto \prod\limits_c \psi_c (\mathbf{x}_c)
\end{equation}
Potentials $\psi_c$ are called "factors" and are associated to a clique $\mathbf{x}_c$ ($\mathbf{x}_c$ is a group of variables $\{x_i\}$ or simply one variable $x_i$). As Figure \ref{fig:general-factor-graph} shows, the probability distribution can be represented graphically as a factor graph composed of variable nodes $x_i$ and factor nodes $\psi_c$, with links between $\psi_c$ and all the variable nodes present in (vector) $\mathbf{x}_c$.

\begin{figure}[h]
  \centering
  \includegraphics[width = 0.9\linewidth]{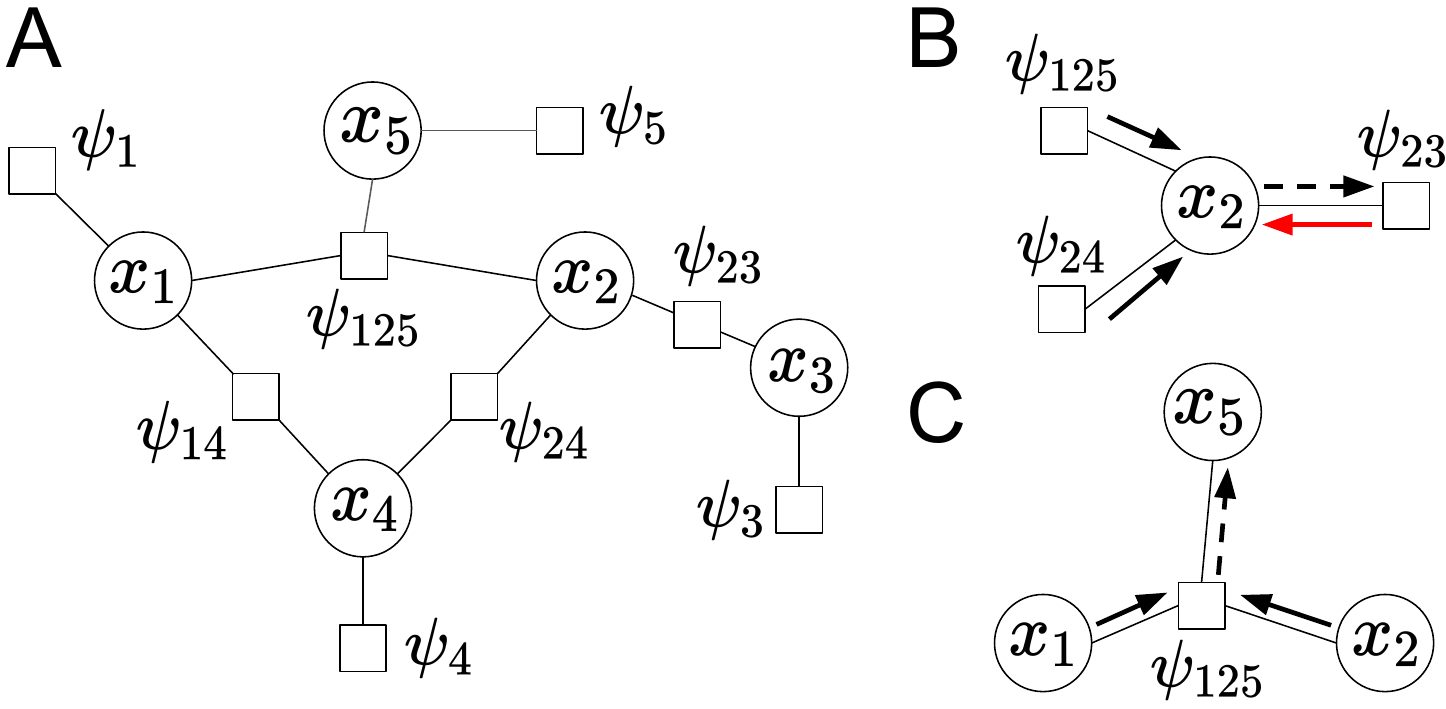}
  \vspace*{1mm} 
  \caption{Belief Propagation and Circular Belief Propagation on a general factor graph. \textbf{(A)} Example of a factor graph with higher-order interactions (i.e., not only unitary and pairwise factors), here representing the probability distribution $p(\mathbf{x}) = \psi_{125}(x_1, x_2, x_5) \psi_{23}(x_2, x_3) \psi_{14}(x_1, x_4) \psi_{24}(x_2, x_4) \psi_1(x_1) \psi_3(x_3) \psi_4(x_4) \psi_5(x_5)$. \textbf{(B)} Belief Propagation updates a given variable-to-factor message (from $x$ to $\psi$, dotted black line) according to the messages received by node $x$ from other factor nodes than $\psi$ (full black lines). For Circular BP, the opposite message (coming from $\psi$ to $x$, full red line) is also taken into account. \textbf{(C)} Both BP and Circular BP update factor-to-variable messages (from $\psi$ to $x$, dotted black line) according to the messages collected by $\psi$ from other variable nodes than $x$ (full black lines), and the interaction function $\psi$.}
  \label{fig:general-factor-graph}
\end{figure}

\subsection{Belief Propagation}

Belief Propagation can be defined on a factor graph \citep{Kschischang2001}. At every iteration, the algorithm updates messages from factor nodes to variable nodes and from variable nodes to factor nodes:
\begin{empheq}[left=\empheqlbrace]{align}
  &m_{\psi_c \to x_i}^{\text{new}}(x_i) \propto \sum\limits_{\mathbf{x}_c \setminus x_i} \psi_c(\mathbf{x}_c) \prod\limits_{x_j \in \mathcal{N}(\psi_c) \setminus x_i} m_{x_j \to \psi_c}(x_j) \label{eq:BP-message-factor-to-variable-general}
  \\
  & m_{x_j \to \psi_d}^{\text{new}}(x_j) \propto \psi_{j}(x_j) \prod\limits_{\psi_c \in \mathcal{N}(x_j) \setminus \psi_d} m_{\psi_c \to x_j}(x_j) \label{eq:BP-message-variable-to-factor-general}
\end{empheq}
where $\mathcal{N}(x)$ are the neighbors of node $x$ in the factor graph. Neighbors of variable nodes are factor nodes, and reciprocally.

Once messages have converged (or at some given maximum iteration), approximate marginal probabilities or \emph{beliefs} are computed as:
\begin{equation}\label{eq:BP-belief-general}
    b_i(x_i) \propto \psi_i(x_i) \prod_{\psi_c \in \mathcal{N}(x_i)} m_{\psi_c \to x_i}(x_i)
\end{equation}

A message $m_{\psi_c \to x_i}(x_i)$ from factor node to variable node correspond to probabilistic information about variable $x_i$ collected by the factor node $\psi_c$. The message is based on the information available elsewhere in the graph (observed variables, prior distribution over variables) received by $\psi_c$, and takes into account the probabilistic interactions between $x_i$ and its neighbors (i.e., the interaction factor $\psi_c$).

A message $m_{x_j \to \psi_d}(x_j)$ from variable node to factor node is simply the sum (in the log-domain) of the local information at $x_j$ (e.g., noisy observation or prior) with the messages received by $x_j$ from all factors neighboring $x_j$ except $\psi_d$.

Note that if factors are all pairwise (case considered in the main text: $\psi_c = \psi_{ij}$) then BP equations \eqref{eq:BP-message-factor-to-variable-general} and \eqref{eq:BP-message-variable-to-factor-general} can be written with messages going from factor to variable node only. We recover Equations \eqref{eq:BP-message} and \eqref{eq:BP-belief} from the main text by defining $m_{i \to j}(x_j) \equiv m_{\psi_c \to x_j}(x_j)$.

\subsection{Reweighted BP}

We consider here a modification of Belief Propagation based on a parametric approximation of the entropy of the approximating distribution $b(\mathbf{x})$, with the parameters $(\bm{\kappa}, \bm{\alpha}, \bm{\beta}, \bm{\gamma})$. $\kappa_i$ and $\gamma_i$ are assigned to the variable node $x_i$, while $\alpha_{c}$ and $\beta_c$ are assigned to the factor $\psi_c$. 
We obtain the following modified update equations (see following paragraph for the demonstration):
\begin{empheq}[left=\empheqlbrace]{align}
  & m_{\psi_c \to x_i}^{\text{new}}(x_i) \propto \bigg(\sum\limits_{\mathbf{x}_c \setminus x_i} \psi_c(\mathbf{x}_c)^{\alpha_c \beta_c} \prod\limits_{x_j \in \mathcal{N}(\psi_c) \setminus x_i} m_{x_j \to \psi_c}(x_j) \bigg)^{1 / \alpha_{c}} \label{eq:eFBP-message-factor-to-variable-general}
  \\
  & m_{x_j \to \psi_d}^{\text{new}}(x_j) \propto \psi_{j}(x_j)^{\gamma_j \kappa_j} \Bigg(\prod\limits_{\psi_c \in \mathcal{N}(x_j) \setminus \psi_d} m_{\psi_c \to x_j}(x_j)\Bigg)^{\kappa_j} m_{\psi_d \to x_j}(x_j)^{\kappa_j - \alpha_{d}} \label{eq:eFBP-message-variable-to-factor-general}
\end{empheq}
and beliefs (approximate marginal probabilities) are computed using:
\begin{equation}\label{eq:eFBP-belief-general}
    b_i(x_i) \propto \psi_i(x_i)^{\gamma_i \kappa_i} \bigg(\prod_{\psi_c \in \mathcal{N}(x_i)} m_{\psi_c \to x_i}(x_i)\bigg)^{\kappa_i}
\end{equation}
The special case where $(\bm{\kappa}, \bm{\alpha}, \bm{\beta}, \bm{\gamma}) = (\bm{1}, \bm{1}, \bm{1}, \bm{1})$ corresponds to BP.

Similarly to above, if factors are all pairwise (case considered in the main text) then the Reweighted BP equations \eqref{eq:eFBP-message-factor-to-variable-general} and
\eqref{eq:eFBP-message-variable-to-factor-general} can be written with messages going from factor to variable node only. We recover Equations \eqref{eq:eFBP-message} and \eqref{eq:eFBP-belief} from the main text by defining $m_{i \to j}(x_j) \equiv m_{\psi_c \to x_j}(x_j)$. 

\subsection{Theoretical background for Reweighted BP}

Here we provide the theoretical foundations underlying the modification of BP given in Equations \eqref{eq:eFBP-message-factor-to-variable-general}, \eqref{eq:eFBP-message-variable-to-factor-general}, and
\eqref{eq:eFBP-belief-general}.

Note that in the demonstration, we cover all the special cases of Reweighted BP, including BP (for which $(\bm{\alpha},\bm{\kappa}, \bm{\beta}, \bm{\gamma}) = (\bm{1}, \bm{1}, \bm{1}, \bm{1})$), as well as Fractional BP, Power EP and $\alpha$-BP (which all correspond to $(\bm{\kappa}, \bm{\beta}, \bm{\gamma}) = (\bm{1}, \bm{1}, \bm{1})$).

As stated in section \ref{sec:eCBP-theory} of the Appendix, the approach taken by BP to compute marginals $p(\mathbf{x})$, whose formula is known but whose marginals are hard to compute, is to approximate $p(\mathbf{x})$ with distribution $b(\mathbf{x})$ (called the \emph{variational} distribution) whose marginals are easier to compute.
The Gibbs free energy (that we would like to minimize) is given by:
\begin{equation}
    G_b = U_b - S_b
\end{equation}
where the variational average energy $U_b$ can be computed easily:
\begin{align}
    U_b &= \sum\limits_\mathbf{x} b(\mathbf{x}) E(\mathbf{x}) \notag \\
    &= - \sum\limits_\mathbf{x} b(\mathbf{x}) \sum\limits_{\text{cliques } c}   \log\Big(\psi_c(\mathbf{x}_c)\Big) - \sum\limits_x b(\mathbf{x}) \sum\limits_{i} \log\Big(\psi_{i}(x_i)\Big) \notag \\
    &= - \sum\limits_{\text{cliques } c}   \sum\limits_{\mathbf{x}_c} b_{c}(\mathbf{x}_c) \log\Big(\psi_c(\mathbf{x}_c)\Big) - \sum\limits_{i} \sum\limits_{x_i} b_i(x_i) \log\Big(\psi_{i}(x_i)\Big)
\end{align}
, contrary to the variational entropy $S_b$.

As it is not possible to easily compute $S_b$, Belief Propagation estimates it as if the factor graph representing $b(x)$ was a tree (i.e., was acyclic). This means that:
\begin{equation}\label{eq:approx-tree-general}
    b(\mathbf{x}) \approx \prod\limits_{\text{cliques } c} \Bigg(\cfrac{b_c(\mathbf{x}_c)}{\prod\limits_{i \in \mathcal{N}(c)} b_i(x_i)}\Bigg) \prod\limits_{\text{nodes } i} b_i(x_i)
\end{equation}
where $b_c(\mathbf{x}_c) \equiv \sum\limits_{x \setminus \mathbf{x}_c} b(\mathbf{x})$ (where for instance $\mathbf{x}_c = (x_1, x_2)$) and $b_i(x_i) \equiv \sum\limits_{x \setminus x_i} b(\mathbf{x})$.

The equation above can also be written:
\begin{equation*}
    b(\mathbf{x}) \approx \prod\limits_{\text{cliques } c} b_c(\mathbf{x}_c) \prod\limits_{\text{nodes } i} b_i(x_i)^{1 - |\mathcal{N}(i)|}
\end{equation*}
where $|\mathcal{N}(i)|$ is the number of neighbors of node $i$ in the graph representation of the distribution.

The approximation of $b(x)$ given in Equation \eqref{eq:approx-tree-general} is equivalent to approximating the entropy $S_b$ of $b(\mathbf{x})$ as follows:
\begin{align}
    -S_b & = \sum\limits_{\mathbf{x}} b(\mathbf{x}) \log(b(\mathbf{x})) \notag\\
    & \approx \sum\limits_{\mathbf{x}} b(\mathbf{x}) \sum\limits_{\text{cliques } c} \log(b_c(\mathbf{x}_c)) + \sum\limits_{\mathbf{x}} b(\mathbf{x}) \sum\limits_{\text{nodes } i} (1 - |\mathcal{N}(i)|) \log(b_i(x_i)) \notag\\
    & \approx \sum\limits_{c} \sum\limits_{\mathbf{x}_c} b_c(\mathbf{x}_c) \log(b_c(\mathbf{x}_c)) + \sum\limits_i (1 - |\mathcal{N}(i)|) \sum\limits_{x_i} b_i(x_i) \log(b_i(x_i))
\end{align}
In contrast, the Reweighted BP algorithm consists of approximating the variational distribution $b(\mathbf{x})$ as:
\begin{equation}
    b(\mathbf{x}) \approx \prod\limits_{\text{cliques } c} \Bigg(\cfrac{b_c(\mathbf{x}_c)}{\prod\limits_{i \in \mathcal{N}(c)} b_i(x_i)}\Bigg)^{1/\alpha_{c}} \prod\limits_{\text{nodes } i} b_i(x_i)^{1/\kappa_i}
\end{equation}
which can also be written:
\begin{equation}\label{eq:approx-belief-eFBP-general}
    b(\mathbf{x}) \approx \prod\limits_{\text{cliques } c} b_c(\mathbf{x}_c)^{1/\alpha_{c}} \prod\limits_{\text{nodes } i} b_i(x_i)^{1/\kappa_i - \lvert \mathcal{N}(i) \rvert / \alpha_i}
    \quad \text{where} \quad \frac{1}{\alpha_i} \equiv \frac{1}{\lvert \mathcal{N}(i) \rvert} \sum\limits_{c;i \in \mathcal{N}(c)} \frac{1}{\alpha_{c}}
\end{equation}
This leads to the following parametric approximation of the variational entropy:
\begin{align}
    - S_b & = \sum\limits_{\mathbf{x}} b(\mathbf{x}) \log(b(\mathbf{x})) \notag \\
    & \approx \sum\limits_{\mathbf{x}} b(\mathbf{x}) \sum\limits_{\text{cliques } c} \frac{1}{\alpha_c} \log(b_c(\mathbf{x}_c)) + \sum\limits_{\mathbf{x}} b(\mathbf{x}) \sum\limits_{\text{nodes } i} \Big(\frac{1}{\kappa_i} - \frac{\abs{\mathcal{N}(i)}}{\alpha_i}\Big) \log(b_i(x_i)) \notag \\
    & \approx \sum\limits_{\text{cliques } c} \frac{1}{\alpha_c} \sum\limits_{\mathbf{x}_c} b_c(\mathbf{x}_c) \log(b_c(\mathbf{x}_c)) + \sum\limits_i \Big(\frac{1}{\kappa_i} - \frac{\abs{\mathcal{N}(i)}}{\alpha_i}\Big) \sum\limits_{x_i} b_i(x_i) \log(b_i(x_i)) \label{eq:approx-entropy-eFBP-general}
\end{align}

Additionally, the initial probability distribution $p(\mathbf{x})$ is approximated as follows:
\begin{equation}
    p(\mathbf{x}) \approx \prod\limits_{\text{cliques } c} \psi_c(\mathbf{x}_c)^{\beta_c} \prod\limits_{\text{nodes } i} \psi_i(x_i)^{\gamma_i}
\end{equation}
which is equivalent to the following approximation of the average energy $U_b$:
\begin{align}
    U_b &= \sum\limits_\mathbf{x} b(\mathbf{x}) E(\mathbf{x}) \notag \\
    &\approx - \sum\limits_\mathbf{x} b(\mathbf{x}) \sum\limits_{\text{cliques } c}  \beta_c \log\Big(\psi_c(\mathbf{x}_c)\Big) - \sum\limits_x b(\mathbf{x}) \sum\limits_{i} \gamma_i \log\Big(\psi_{i}(x_i)\Big) \notag \\
    &\approx - \sum\limits_{\text{cliques } c}  \beta_c \sum\limits_{\mathbf{x}_c} b_{c}(\mathbf{x}_c) \log\Big(\psi_c(\mathbf{x}_c)\Big) - \sum\limits_{i} \gamma_i \sum\limits_{x_i} b_i(x_i) \log\Big(\psi_{i}(x_i)\Big)
\end{align}

Altogether, we obtain the following approximation of the Gibbs free energy $G_b^{\text{approx}} \approx G_b = U_b - S_b$:
\begin{multline}\label{eq:modif-Bethe-Free-Energy-general}
    G_b^{\text{approx}} = \sum\limits_{\text{cliques } c} \frac{1}{\alpha_c} \sum\limits_{\mathbf{x}_c} b_c(\mathbf{x}_c) \log\Bigg(\cfrac{b_c(\mathbf{x}_c)}{\prod\limits_{i \in \mathcal{N}(c)} b_i(x_i)}\Bigg) - \sum\limits_{\text{cliques } c} \beta_c \sum\limits_{\mathbf{x}_c} b_c(\mathbf{x}_c) \log\Big(\psi_c(\mathbf{x}_c)\Big) \\+ \sum\limits_i \frac{1}{\kappa_i} \sum\limits_{x_i} b_i(x_i) \log\Big(b_i(x_i)\Big) - \sum\limits_i \gamma_i \sum\limits_{x_i} b_i(x_i) \log\Big(\psi_i(x_i)\Big)
\end{multline}

\subsection{From Gibbs free energy approximation to messages}
In what follows, we derive the message-passing update equations of Reweighted BP \eqref{eq:eFBP-message-factor-to-variable-general}, \eqref{eq:eFBP-message-variable-to-factor-general}, and \eqref{eq:eFBP-belief-general} (and its special cases BP, Fractional BP, Power EP and $\alpha$-BP, among others) based on the approximation of the Gibbs free energy $G_b^{\text{approx}}$ given in Equation \eqref{eq:modif-Bethe-Free-Energy-general}. The message update equations are simply fixed-point equations of $G_b^{\text{approx}}$, meaning that the beliefs computed by Reweighted BP are stationary points of the approximate Gibbs free energy $G_b^{\text{approx}}$. This demonstration is similar to the one for BP \citep{Yedidia2001, Yedidia2003}
, with additional parameters $\bm{\alpha}$, $\bm{\kappa}$, $\bm{\beta}$, and $\bm{\gamma}$.

We form the Lagrangian by adding Lagrange multipliers to $G_b^{\text{approx}}$. Lagrange multiplier $\mu_i$ (resp. $\mu_{c}$) corresponds to the normalization constraint $\sum_{x_i} b_i(x_i) = 1$ (resp. $\sum_{\mathbf{x}_c} b_c(\mathbf{x}_c) = 1$), while $\lambda_{cj}(x_j)$ corresponds to the marginalization constraint $\sum_{\mathbf{x}_{c \setminus j}} b_c(\mathbf{x}_c) = b_j(x_j)$. We obtain: 
\begin{align}
    \mathcal{L} = G_b^{\text{approx}} + \sum_{i} \mu_i \Big(\sum_{x_i} b_i(x_i) - 1\Big) + \sum_{\text{cliques } c} \mu_c \Big(\sum_{\mathbf{x}_c} b_c(\mathbf{x}_c) - 1\Big) \notag \\
    + \sum_{\text{cliques } c} \sum_{j \in \mathcal{N}(c)} \sum\limits_{x_j} \lambda_{cj}(x_j) \Big(\sum\limits_{\mathbf{x}_{c \setminus j}} b_c(\mathbf{x}_c) - b_j(x_j)\Big)
\end{align}
The partial derivatives of the Lagrangian are:
\begin{empheq}[left=\empheqlbrace]{align}
  & \frac{\partial \mathcal{L}}{\partial b_i(x_i)} = - \sum\limits_{c \in \mathcal{N}(i)} \frac{1}{\alpha_c} + \frac{1}{\kappa_i} + \frac{1}{\kappa_i} \log(b_i(x_i)) - \gamma_i \log(\psi_i(x_i)) + \mu_i - \sum\limits_{c \in \mathcal{N}(i)} \lambda_{ci}(x_i)
  \\
  & \frac{\partial \mathcal{L}}{\partial b_c(\mathbf{x}_c)} = \frac{1}{\alpha_c} + \frac{1}{\alpha_c} \log\Bigg(\frac{b_c(\mathbf{x}_c)}{\prod\limits_{i \in \mathcal{N}(c)} b_i(x_i)}\Bigg) - \beta_c
\log(\psi_c(\mathbf{x}_c)) + \sum\limits_{i \in \mathcal{N}(c)} \lambda_{ci}(x_i) + \mu_c
\end{empheq}
It comes, by cancelling the partial derivatives of the Lagrangian:
\begin{equation}
    b_i(x_i) \propto \psi_i(x_i)^{\gamma_i \kappa_i} \prod\limits_{c \in \mathcal{N}(i)} \exp\Big(\kappa_i \lambda_{ci}(x_i)\Big)
\end{equation}
and
\begin{equation}
    b_c(\mathbf{x}_c) \propto \Bigg(\prod\limits_{i \in \mathcal{N}(c)} b_i(x_i)\Bigg) \psi_c(\mathbf{x}_c)^{\alpha_c \beta_c} \prod\limits_{i \in \mathcal{N}(c)} \exp\Big(-\alpha_c \lambda_{ci}(x_i)\Big)
\end{equation}
\begin{equation}
    \implies b_c(\mathbf{x}_c) \propto  \psi_c(\mathbf{x}_c)^{\alpha_c \beta_c} \Bigg(\prod_{i \in \mathcal{N}(c)} \psi_i(x_i)^{\gamma_i \kappa_i}
\bigg(\prod\limits_{d \in \mathcal{N}(i) \setminus c} \exp\big(\kappa_i \lambda_{di}(x_i)\big)\bigg) 
\exp\Big(\lambda_{ci}(x_i)\Big(\kappa_i - \alpha_c\Big)\Big) \Bigg)
\end{equation}
We obtain, for $m_{c \to i}(x_i) \equiv \exp\Big(\lambda_{ci}(x_i)\Big)$, the following expression of the (approximate) marginal and pairwise beliefs:
\begin{empheq}[left=\empheqlbrace]{align}
  & b_i(x_i) \propto \psi_i(x_i)^{\gamma_i \kappa_i} \prod\limits_{c \in \mathcal{N}(i)} m_{c \to i}(x_i)^{\kappa_i} \label{eq:eFBP-belief-general-dem}
  \\
  &  b_c(\mathbf{x}_c) \propto \psi_c(\mathbf{x}_c)^{\alpha_c \beta_c} \Bigg(\prod_{i \in \mathcal{N}(c)} \psi_i(x_i)^{\gamma_i \kappa_i}\Bigg)
\Bigg(\prod\limits_{i \in \mathcal{N}(c)} \prod\limits_{d \in \mathcal{N}(i) \setminus c} m_{d \to i}(x_i)^{\kappa_i} \Bigg)
\Bigg(\prod\limits_{i \in \mathcal{N}(c)} m_{c \to i}(x_i)^{\kappa_i - \alpha_c} \Bigg) \label{eq:eFBP-pairwise-belief-general-dem}
\end{empheq}
Eventually, thanks to the constraint $\sum\limits_{\mathbf{x}_c \setminus x_j} b_c(\mathbf{x}_c) = b_j(x_j)$, we obtain the fixed point equations for the messages:
\begin{equation}\label{eq:damped-eFBP-from-Lagrangian-general}
    m_{c \to j}(x_j)^{\kappa_j} \propto \Bigg( \sum_{\mathbf{x}_c \setminus x_j} \psi_c(\mathbf{x}_c)^{\alpha_c \beta_c} \prod\limits_{i \in \mathcal{N}(c) \setminus j} \Bigg[\psi_i(x_i)^{\gamma_i} \bigg(\prod\limits_{d \in \mathcal{N}(i) \setminus c} m_{d \to i}(x_i)\bigg) m_{c \to i}(x_i)^{1 - \alpha_c/\kappa_i}\Bigg]^{\kappa_i}\Bigg) m_{c \to j}(x_j)^{\kappa_j - \alpha_c}
\end{equation}
\begin{equation}\label{eq:eFBP-from-Lagrangian-general} 
    \Leftrightarrow m_{c \to j}(x_j) \propto \Bigg( \sum_{\mathbf{x}_c \setminus x_j} \psi_c(\mathbf{x}_c)^{\alpha_c \beta_c} \prod\limits_{i \in \mathcal{N}(c) \setminus j} \Bigg[ \psi_i(x_i)^{\gamma_i} \bigg(\prod\limits_{d \in \mathcal{N}(i) \setminus c} m_{d \to i}(x_i)\bigg) m_{c \to i}(x_i)^{1 - \alpha_c/\kappa_i}\Bigg]^{\kappa_i}\Bigg)^{1 / \alpha_c}
\end{equation} 
The Reweighted BP algorithm consists of running iteratively the fixed-point equation \eqref{eq:eFBP-from-Lagrangian-general}. This single equation, which involves only messages from factor node to variable node, can be rewritten into the two following equations by introducing messages from variable node to factor node: 
\begin{empheq}[left=\empheqlbrace]{align}
  & m_{\psi_c \to x_i}^{\text{new}}(x_i) \propto \bigg(\sum\limits_{\mathbf{x}_c \setminus x_i} \psi_c(\mathbf{x}_c)^{\alpha_c \beta_c} \prod\limits_{x_j \in \mathcal{N}(\psi_c) \setminus x_i} m_{x_j \to \psi_c}(x_j) \bigg)^{1 / \alpha_{c}} \label{eq-repeat:eFBP-message-factor-to-variable-general}
  \\
  & m_{x_j \to \psi_d}^{\text{new}}(x_j) \propto \psi_{j}(x_j)^{\gamma_j \kappa_j} \Bigg(\prod\limits_{\psi_c \in \mathcal{N}(x_j) \setminus \psi_d} m_{\psi_c \to x_j}(x_j)\Bigg)^{\kappa_j} m_{\psi_d \to x_j}(x_j)^{\kappa_j - \alpha_{d}} \label{eq-repeat:eFBP-message-variable-to-factor-general}
\end{empheq}
The expression of the unitary $b_i(x_i)$ and clique beliefs $b_c(\mathbf{x}_c)$ comes from Equations \eqref{eq:eFBP-belief-general-dem} and \eqref{eq:eFBP-pairwise-belief-general-dem}:
\begin{empheq}[left=\empheqlbrace]{align}
  & b_i(x_i) \propto \psi_i(x_i)^{\gamma_i \kappa_i} \prod\limits_{\psi_c \in \mathcal{N}(x_i)} m_{\psi_c \to x_i}(x_i)^{\kappa_i} 
  \\
  & b_c(\mathbf{x}_c) \propto \psi_c(\mathbf{x}_c)^{\alpha_c \beta_c} \prod_{x_i \in \mathcal{N}(\psi_c)} m_{x_i \to \psi_c}(x_i) 
\end{empheq}
Note that one could also use directly Equation \eqref{eq:damped-eFBP-from-Lagrangian-general} instead of \eqref{eq:eFBP-from-Lagrangian-general} to define the message update equation of the Reweighted BP algorithm. In fact, Equations \eqref{eq:damped-eFBP-from-Lagrangian-general} and  \eqref{eq:eFBP-from-Lagrangian-general} correspond to the damped versus undamped update equation; see section \ref{sec:damping} about damping. There is no absolute better choice: fixed points obtained are identical in both cases, and damping provides better convergence properties but slows down the system (see section \ref{sec:damping}).

\subsection{Comparison with related models: BP, Fractional BP, Power EP, alpha-BP, and Circular BP}

The special case of Belief Propagation is recovered for $(\bm{\alpha}, \bm{\kappa}, \bm{\beta}, \bm{\gamma}) = (\mathbf{1}, \mathbf{1}, \mathbf{1}, \mathbf{1})$.

Fractional BP \citep{Wiegerinck2002}, Power EP \citep{Minka2004} and $\alpha$-BP \citep{Liu2019} use the damped message update equation \eqref{eq:damped-eFBP-from-Lagrangian} (with $(\bm{\kappa}, \bm{\beta}, \bm{\gamma}) = (\bm{1}, \bm{1}, \bm{1})$) rather than its undamped version \eqref{eq:eFBP-from-Lagrangian} (see section \ref{sec:damping}).


Circular BP of \citet{Jardri2013} also considers $(\bm{\kappa}, \bm{\beta}, \bm{\gamma}) = (\bm{1}, \mathbf{1}, \mathbf{1})$, and modifies the message update equation from variable to factor (Equation \eqref{eq:BP-message-variable-to-factor-general} for BP) into:
\begin{equation}\label{eq:CBP-message-variable-to-factor-general}
    m_{x_j \to \psi_d}(x_j) \propto \psi_{j}(x_j) \Bigg(\prod\limits_{\psi_c \in \mathcal{N}(x_j) \setminus \psi_d} m_{\psi_c \to x_j}(x_j)\Bigg) m_{\psi_d \to x_j}(x_j)^{1 - \alpha_{x_j \to \psi_d}}
\end{equation}
Note that Circular BP was only defined on pairwise factor graphs, meaning that all cliques $c$ have at most 2 elements: the product of Equation \eqref{eq-repeat:eFBP-message-variable-to-factor-general} does not appear in this case.

Notably, Reweighted BP, BP, Fractional BP and Circular BP 
are similar but differ in the number of degrees of liberty. BP has no degree of liberty. Fractional BP (as well as Power EP and $\alpha$-BP) has $n_{\text{factors}}$ degrees of liberty. Circular BP has $n_{\text{factors}} + n_{\text{variables}}$ degrees of liberty. Circular BP has as many degrees of liberty as the number of edges in the factor graph.

\paragraph{Circular BP on general factor graphs}
Similarly to the pairwise factors case, Circular BP can be seen as an approximation of (an extension of) Fractional BP, which includes reweighting of both the variational entropy and the variational average energy components. 

As in the pairwise factor case presented in the main text, the Circular BP algorithm on general factor graphs can be seen as an approximation of reweighted BP algorithms, based on the modification of Equation \eqref{eq:eFBP-from-Lagrangian-general} into:
\begin{equation}
    m_{c \to j}(x_j) \propto  \sum_{\mathbf{x}_c \setminus x_j} \psi_c(\mathbf{x}_c)^{\beta_c} \prod\limits_{i \in \mathcal{N}(c) \setminus j} \Bigg[ \psi_i(x_i)^{\gamma_i} \bigg(\prod\limits_{d \in \mathcal{N}(i) \setminus c} m_{d \to i}(x_i)\bigg) m_{c \to i}(x_i)^{1 - \alpha_c/\kappa_i}\Bigg]^{\kappa_i}
\end{equation} 


\paragraph{Learning to carry out approximate inference in the general case}
The experiments reported in section \ref{sec:learning-CBP} demonstrated successes of supervised learning on rather small and binary graphical models with pairwise interactions. However, many situations where humans perform inference nearly optimally involve much more generative models. A follow-up of this work using the theory developed in the current section 
could focus on more general graphical models, with higher-order interactions and/or with continuous variables, and investigate whether parameters can be learnt (in a supervised but also unsupervised manner) in these cases.



In particular, Circular BP applied to a probability distribution composed of at most pairwise interactions but any type of variables (discrete or continuous) can be written very simply, similarly to the binary case  (see also Equations 10-12 of the supplementary material of \citet{Jardri2013}):
\begin{empheq}[left=\empheqlbrace]{align}
  &M_{i \to j}^{\text{new}}(X_j) = f\big(B_i(X_i) - \alpha_{ij} M_{j \to i}(X_i), J_{ij}\big)\\
  &B_i(X_i) = \kappa_i \Big(\sum_{k \in \mathcal{N}(i)} M_{k \to i}(X_i) + \gamma_i M_{\text{ext} \to i}(X_i)\Big)
\end{empheq}
with log-messages $M_{i \to j}(X_j) = \log(m_{i \to j}(X_j))$, $M_{\text{ext} \to i}(X_i) = \log(m_{\text{ext} \to i}(X_i))$, and log-beliefs $B_i(X_i) = \log(p(X_i)) + \log(Z)$ where $Z$ is the normalization constant of BP. Finally, $f$ is a function of function (it applies to $g(X_i)$) which, among other things, performs an integration over $X_i$ and returns a function of $X_j$):
\begin{equation}
    f(g(x_i), J_{ij}) = \log\Big(\int_{x_i} \psi_{ij}(x_i,x_j)^{\beta_{ij}} \exp(g(x_i)) dx_i\Big)
\end{equation}
Instead of encoding for $\log(p(X=+1)/p(X=-1))$ as in the binary case, the quantity computed is $\log(p(X))$, the log of the whole probability distribution of variable $X$. Likewise, the external input to the nodes is no longer a scalar ($M_{\text{ext} \to i}(X_i = +1) - M_{\text{ext} \to i}(X_i = -1)$) but a function (that is, $M_{\text{ext} \to i}(X_i)$).
\end{document}